\documentclass[11pt, oneside]{article}
\usepackage{etoolbox}
\newcommand{\arxiv}[1]{\iftoggle{colt}{}{#1}}
\newcommand{\colt}[1]{\iftoggle{colt}{#1}{}}
\newtoggle{colt}
\global\togglefalse{colt}

\arxiv{\usepackage{geometry}
\geometry{letterpaper}}
\arxiv{
\usepackage{graphicx}	
\usepackage{amssymb}
\usepackage{amsfonts,latexsym,amsthm,amssymb,amsmath,amscd,euscript}
}
\colt{\usepackage{amsthm}}
\usepackage{bbm}
\usepackage{framed}
\arxiv{\usepackage{fullpage}}
\usepackage{mathrsfs}
\usepackage{hyperref}
\usepackage{enumitem}

\colt{
\usepackage{titlesec}
\titlespacing*{\section}{0pt}{0.3\baselineskip}{0.3\baselineskip}
\titlespacing*{\subsection}{0pt}{0.3\baselineskip}{0.3\baselineskip}
\titlespacing*{\subsubsection}{0pt}{0.3\baselineskip}{0.3\baselineskip}
}

\hypersetup{
  colorlinks=true,
  linkcolor=blue,
  filecolor=blue,
  citecolor = black,      
  urlcolor=cyan,
}
\arxiv{\usepackage{accents}}
\usepackage{setspace}
\hypersetup{colorlinks=true,citecolor=blue,urlcolor=black,linkbordercolor={1 0 0}}
\usepackage{tikz-cd}
\newcount\Comments  %
\ifnum\Comments=1
\usepackage[inline]{showlabels}

\fi
\usepackage{turnstile}
\usepackage{mathpartir}
\usepackage{tikz}
\usepackage{xspace}

\usepackage{algorithm}
\usepackage{verbatim}
\usepackage[noend]{algpseudocode}

\newcommand{\multiline}[1]{\parbox[t]{\dimexpr\linewidth-\algorithmicindent}{#1}}
\newcommand{\nc}{\newcommand}

\usepackage[nameinlink,capitalize]{cleveref}
\crefformat{equation}{#2(#1)#3}
\Crefformat{equation}{#2(#1)#3}

\Crefformat{figure}{#2Figure #1#3}
\Crefname{assumption}{Assumption}{Assumptions}
\Crefformat{assumption}{#2Assumption #1#3}
   \Crefname{question}{Question}{Questions}
   \Crefformat{question}{#2Question #1#3}
   \Crefname{claim}{Claim}{Claims}
   \Crefformat{claim}{#2Claim #1#3}
   \Crefname{problem}{Problem}{Problems}
  \Crefformat{problem}{#2Problem #1#3}
\Crefname{subsubsection}{Section}{Sections}
\crefformat{subsubsection}{#2Section #1#3}
\Crefformat{subsubsection}{#2Section #1#3}

\nc{\sups}[1]{^{\scriptscriptstyle{#1}}}
\nc{\subs}[1]{_{\scriptscriptstyle{#1}}}

\input{widebar}
\newcommand{\wb}{\widebar}

\newcommand{\epbell}{\varepsilon_{\mathsf{bkup}}}
\newcommand{\epfinal}{\varepsilon_{\mathsf{final}}}
\newcommand{\epapx}{\varepsilon_{\mathsf{apx}}}
\newcommand{\sigtr}{\sigma_{\mathsf{tr}}}
\newcommand{\Ftp}{F^{\mathsf{tl}}}
\newcommand{\Ftpm}{\til{F}^{\mathsf{tl}}}

\newcommand{\Fnormal}{F^{\mathsf{normal}}}

\newcommand{\Pilin}{\Pi^{\mathsf{lin}}}
\newcommand{\Pilinp}[1][\sigma]{\Pi^{\mathsf{Plin}, {#1}}}
\newcommand{\Pilinpp}{\Pi^{\mathsf{Plin}}}

\newcommand{\Sm}[1][\sigma]{\mathsf{S}_{#1}}
\nc{\Critic}{\texttt{Critic}\xspace}
\nc{\PSDPUCB}{\texttt{PSDP-UCB}\xspace}
\nc{\LSVIUCB}{\texttt{LSVI-UCB}\xspace}
\nc{\Actor}{\texttt{Actor}\xspace}
\nc{\EstFeature}{\texttt{EstFeature}\xspace}
\nc{\ExpFTPL}{\texttt{ExpFTPL}\xspace}
\nc{\dist}{\mathrm{dist}}
\nc{\Bquad}{B^{\mathsf{quad}}}

\arxiv{
\makeatletter
\newtheorem*{rep@theorem}{\rep@title}
\newcommand{\newreptheorem}[2]{%
\newenvironment{rep#1}[1]{%
 \def\rep@title{#2 \ref{##1}}%
 \begin{rep@theorem}}%
 {\end{rep@theorem}}}
\makeatother
}

\makeatletter
\newcommand\xlabel[2][]{\phantomsection\def\@currentlabelname{#1}\label{#2}}
\makeatother

{
\theoremstyle{plain}
\newtheorem{theorem}{Theorem}
\newtheorem{lemma}[theorem]{Lemma}
\newtheorem{corollary}[theorem]{Corollary}

\newtheorem{proposition}[theorem]{Proposition}

\newtheorem{claim}[theorem]{Claim}
\newtheorem{assumption}[theorem]{Assumption}
\theoremstyle{definition}
\newtheorem{definition}{Definition}

\newtheorem{question}[definition]{Question}
\newtheorem{problem}[definition]{Problem}

\numberwithin{theorem}{section}
\numberwithin{definition}{section}
}

\nc{\DMO}{\DeclareMathOperator}

\DeclareMathOperator*{\argmin}{arg\,min} %
\DeclareMathOperator*{\argmax}{arg\,max}

\DMO{\prox}{prox}
\DMO{\UCB}{UCB}
\DMO{\LCB}{LCB}
\nc{\phidiff}{\phi\sups{\Delta}}
\nc{\pexp}{q_{\mathrm{exp}}}
\nc{\nn}{\nonumber}
\nc{\rk}{\mathrm{rk}}
\nc{\brk}[3]{{\rm br}_{#1}^{#2}({#3})}
\nc{\co}{{\rm co}}
\nc{\br}[2]{{\rm br}^{#1}({#2})}
\nc{\depth}[1]{{\rm d}({#1})}
\nc{\tA}{\textsc{A}}
\nc{\child}[2]{{\rm ch}_{#1}({#2})}
\nc{\parent}[1]{{\rm pa}({#1})}
\nc{\dg}{\dagger}
\nc{\bB}{\mathbf{B}}
\nc{\Span}{{\rm Span}}
\nc{\unif}{\mathsf{unif}}
\nc{\indsig}[2]{\mathcal{I}_{#1}({#2})}
\nc{\total}{{\rm fin}}
\nc{\early}{{\rm pre}}
\nc{\zsink}{z_{\rm sink}}
\nc{\lowv}{{\rm low}}
\nc{\ol}{\overline}
\nc{\ul}{\underline}
\nc{\madec}[3]{\texttt{ma-dec}_{#1}({#2}, {#3})}
\nc{\madeco}[1]{\texttt{ma-dec}_{#1}}
\nc{\madecd}[3]{\texttt{ma-dec}^{\texttt{d}}_{#1}({#2}, {#3})}
\nc{\SF}{\mathscr{F}}
\nc{\SH}{\mathscr{H}}
\nc{\SP}{\mathscr{P}}
\nc{\SPc}{\wb{\mathscr{P}}}
\nc{\SB}{\mathscr{B}}
\nc{\SC}{\mathscr{C}}
\nc{\BS}{\mathbb{S}}
\nc{\PiMarkov}{\Pi^{\rm markov}}
\nc{\trunc}[2]{\mathsf{trunc}_{#2}({#1})}
\nc{\sbl}{of strong Bellman type\xspace}
\nc{\inormal}[1][\Phi, u,v]{\til{N}_{{#1}}}

\nc{\gamvec}{\gamma}
\nc{\til}{\widetilde}
\nc{\td}{\tilde}
\nc{\wh}{\widehat}
\arxiv{\nc{\todo}[1]{\ifnum\Comments=1 {\color{red}  [TODO: #1]}\fi}}
\nc{\old}[1]{\ifnum\Comments=1 {\color{brown}  [OLD: #1]}\fi}
\nc{\noah}[1]{\ifnum\Comments=1 {\color{purple} [ng: #1]}\fi}
\nc{\dhruv}[1]{\ifnum\Comments=1 {\color{magenta} [dr: #1]}\fi}
\nc{\BP}{\mathbb{P}}
\nc{\BI}{\mathbb{I}}
\nc{\midpoint}[1][\Phi,\phi_1,\phi_2]{\mu^{\star}_{{#1}}}

\nc{\fools}[3]{\MF_{#3}({#1}, {#2})}
\nc{\fool}[2]{\MF({#1},{#2})}
\nc{\clip}[2]{{\rm clip}\left[ \left. {#1} \right| {#2} \right]}
\nc{\imax}{\omega}
\DMO{\conv}{conv}
\nc{\MH}{\mathcal{H}}
\nc{\MV}{\mathcal{V}}
\nc{\MC}{\mathcal{C}}
\nc{\MI}{\mathcal{I}}
\nc{\st}{\star}
\nc{\lng}{\langle}
\nc{\rng}{\rangle}
\DMO{\OOPT}{opt}
\nc{\dopt}[2]{\ell_{\OOPT}({#1},{#2})}
\nc{\grad}{\nabla}
\nc{\MG}{\mathcal{G}}
\nc{\MP}{\mathcal{P}}
\nc{\PP}{\mathbb{P}}
\nc{\TT}{\mathbb{T}}
\nc{\TTmax}{\TT_{\max}}
\DMO{\REG}{Reg}
\DMO{\WREG}{wReg}
\nc{\reg}[2]{{\Delta}_{{#1}}({#2})}
\nc{\wreg}[2]{{\Delta}^{\rm w}_{{#1}}({#2})}
\nc{\Reg}[2]{{\REG}_{{#1}}({#2})}
\nc{\wReg}[2]{{\WREG}_{{#1}}({#2})}
\DMO{\Ham}{Ham}
\DMO{\Gap}{Gap}
\DMO{\GD}{GD}
\DMO{\GDA}{GDA}
\DMO{\EG}{EG}
\nc{\TE}{\til{\E}}
\nc{\Var}{\mathbb{V}}
\DMO{\Cov}{Cov}
\DMO{\OGDA}{OGDA}
\DMO{\Unif}{Unif}
\DMO{\Tr}{Tr}
\nc{\Qu}{\ul{Q}}
\nc{\Qo}{\ol{Q}}
\nc{\Ro}{\ol{R}}
\nc{\Vu}{\ul{V}}
\nc{\Vo}{\ol{V}}
\nc{\RanQ}{\Delta Q}
\nc{\RanV}{\Delta V}
\nc{\clipQ}{\Delta \breve{Q}}
\nc{\frzQ}{\Delta \mathring{Q}}
\nc{\clipV}{\Delta \breve{V}}
\nc{\clipdelta}{\breve{\delta}}
\nc{\cliptheta}{\breve{\theta}}
\nc{\delmin}{\Delta_{{\rm min}}}
\nc{\delmins}[1]{\Delta_{{\rm min},{#1}}}
\nc{\gapfinal}[1]{\max \left\{ \frac{\frzQ_{{#1}}^{k^\st}(x,a)}{2H}, \frac{\delmin}{4H} \right\}}
\nc{\post}[2]{R({#1}; {#2})}
\nc{\posts}[3]{R_{#3}({#1}; {#2})}

\nc{\algnst}[1]{\begin{align*}#1\end{align*}}
\nc{\algn}[1]{\begin{align}#1\end{align}}
\nc{\matx}[1]{\left(\begin{matrix}#1\end{matrix}\right)}
\renewcommand{\^}[1]{^{(#1)}}

\nc{\nuu}{\nu}

\nc{\bel}[1]{\mathbf{b}({#1})}
\nc{\nbel}[1]{\bar{\mathbf{b}}({#1})}
\nc{\sbel}[2]{\mathbf{b}'_{#1}({#2})}
\nc{\nsbel}[2]{\bar{\mathbf{b}}'_{#1}({#2})}

\nc{\bv}{\mathbf{v}}
\nc{\bone}{\mathbf{1}}
\nc{\bX}{\mathbf{X}}
\nc{\bY}{\mathbf{Y}}
\nc{\bG}{\mathbf{G}}
\nc{\bz}{\mathbf{z}}
\nc{\bw}{\mathbf{w}}
\nc{\bA}{\mathbf{A}}
\nc{\bJ}{\mathbf{J}}
\nc{\bK}{\mathbf{K}}
\nc{\bb}{\mathbf{b}}
\nc{\ba}{\mathbf{a}}
\nc{\bc}{\mathbf{c}}
\nc{\bC}{\mathbf{C}}
\nc{\BR}{\mathbb R}
\nc{\BA}{\mathbb{A}}
\nc{\BC}{\mathbb C}
\nc{\bx}{\mathbf{x}}
\nc{\bS}{\mathbf{S}}
\nc{\bM}{\mathbf{M}}
\nc{\bR}{\mathbf{R}}
\nc{\bN}{\mathbf{N}}
\nc{\NN}{\mathbb{N}}
\nc{\by}{\mathbf{y}}
\nc{\sy}{y}
\nc{\sx}{x}

\nc{\MO}{\mathcal O}
\nc{\MU}{\mathcal{U}}
\nc{\ME}{\mathcal{E}}
\nc{\MN}{\mathcal{N}}
\nc{\MK}{\mathcal{K}}
\nc{\MM}{\mathcal{M}}
\nc{\MS}{\mathcal{S}}
\nc{\MT}{\mathcal{T}}
\nc{\BF}{\mathbb F}
\nc{\BQ}{\mathbb Q}
\nc{\MX}{\mathcal{X}}
\nc{\MA}{\mathcal{A}}
\nc{\MD}{\mathcal{D}}
\nc{\MB}{\mathcal{B}}
\nc{\MZ}{\mathcal{Z}}
\nc{\MJ}{\mathcal{J}}
\nc{\MW}{\mathcal{W}}
\nc{\MR}{\mathcal{R}}
\nc{\MY}{\mathcal{Y}}
\nc{\BZ}{\mathbb Z}
\nc{\BN}{\mathbb N}
\nc{\ep}{\epsilon}
\nc{\epbe}{\varepsilon_{\mathsf{BE}}}
\nc{\epout}{\varepsilon_{\mathsf{outlier}}}
\nc{\bellc}[1][h]{\MT_{#1}^\circ}
\nc{\vep}{\varepsilon}
\nc{\gapfn}[1]{\varepsilon_{#1}}
\nc{\ggapfn}[2]{\varphi_{#1}({#2})}
\nc{\epsahk}{\gapfn{0}}
\nc{\BH}{\mathbb H}
\nc{\BG}{\mathbb{G}}
\nc{\D}{\Delta}
\nc{\MF}{\mathcal{F}}
\nc{\One}[1]{\mathbbm{1}\{{#1}\}}
\nc{\bOne}{\mathbf{1}}
\nc{\Aopt}{\mathcal{A}^{\rm opt}}
\nc{\Amul}{\mathcal{A}^{\rm mul}}

\nc{\SQ}{\mathsf Q}

\nc{\DO}{\accentset{\circ}{\D}}
\nc{\mf}{\mathfrak}
\nc{\mfp}{\mathfrak{p}}
\nc{\mfq}{\mf{q}}
\nc{\mfx}{\mf{s}}
\nc{\Sp}{\mbox{Spec}}
\nc{\Spm}{\mbox{Specm}}
\nc{\hookuparrow}{\mathrel{\rotatebox[origin=c]{90}{$\hookrightarrow$}}}
\nc{\hookdownarrow}{\mathrel{\rotatebox[origin=c]{-90}{$\hookrightarrow$}}}
\nc{\hra}{\hookrightarrow}
\nc{\tra}{\twoheadrightarrow}
\nc{\sgn}{{\rm sgn}}
\nc{\aut}{{\rm Aut}}
\nc{\Hom}{{\rm Hom}}
\nc{\img}{{\rm Im}}
\DMO{\id}{Id}
\DMO{\supp}{supp}
\DMO{\KL}{KL}
\nc{\kld}[2]{D_{\mathsf{KL}}({#1}||{#2})}
\nc{\ren}[2]{D_2({#1}||{#2})}
\nc{\chisq}[2]{\chi^2({#1}||{#2})}
\nc{\tvd}[2]{D_{\mathsf{TV}}({#1}, {#2})}
\nc{\hell}[2]{D_{\mathsf{H}}^2({#1}, {#2})}
\nc{\dbi}[3][\pi]{D_{\mathsf{bi}}^{#1}({#2} \| {#3})}
\DMO{\BSS}{BSS}
\DMO{\BES}{BES}
\DMO{\BGS}{BGS}
\DMO{\poly}{poly}
\nc{\indep}{\perp}
\DMO{\sink}{sink}
\nc{\fp}[1]{\MP_1({#1})}
\nc{\BO}{\mathbb{O}}
\nc{\BT}{\mathbb{T}}

\nc{\RR}{\mathbb{R}}
\nc{\Gradient}{\nabla}
\DMO{\diag}{diag}
\nc{\norm}[1]{\left \lVert #1 \right \rVert}
\nc{\EE}{\mathbb{E}}
\nc{\MQ}{\mathcal{Q}}
\nc{\ML}{\mathcal{L}}
\nc{\cPhi}{\bar \Phi}

\DMO{\PR}{Pr}
\renewcommand{\Pr}{\PR}
\nc{\E}{\mathbb{E}}
\nc{\ra}{\rightarrow}
\renewcommand{\t}{\top}
\nc{\hc}{\{0,1\}^n}
\nc{\pmhc}[1]{\{-1,1\}^{#1}}
\nc{\Dbnd}{D}
\nc{\Bbnd}{B}

\colt{\renewcommand{\cite}{\citep}}
\colt{
  \title[Efficient learning of Linear Bellman Complete MDPs]{Linear Bellman Completeness Suffices for Efficient Online Reinforcement Learning with  Few Actions}
  \coltauthor{
    \Name{Noah Golowich} \Email{nzg@mit.edu} \\
    \Name{Ankur Moitra} \Email{moitra@mit.edu} \\ \addr Massachusetts Institute of Technology}
}
\arxiv{\title{Linear Bellman Completeness Suffices for Efficient Online Reinforcement Learning with  Few Actions}}
\colt{\usepackage{times}}
\arxiv{\author{Noah Golowich\thanks{Email: \texttt{nzg@mit.edu}. Supported by a Fannie \& John Hertz Foundation Fellowship and an NSF Graduate Fellowship.} \\ MIT \and Ankur Moitra\thanks{Email: \texttt{moitra@mit.edu}. Supported in part by a Microsoft Trustworthy AI Grant, an ONR grant and a David and Lucile Packard Fellowship.} \\ MIT}
  \date{June 17, 2024}}

\begin{document}
\maketitle

\begin{abstract}
One of the most natural approaches to reinforcement learning (RL) with function approximation is \emph{value iteration}, which inductively generates approximations to the optimal value function by solving a sequence of regression problems. To ensure the success of value iteration, it is typically assumed that \emph{Bellman completeness} holds, which ensures that these regression problems are well-specified. We study the problem of learning an optimal policy under Bellman completeness in the \emph{online} model of RL with linear function approximation.
In the linear setting, while statistically efficient algorithms are known under Bellman completeness (e.g., \cite{jiang2017contextual,zanette2020learning}), these algorithms all rely on the principle of \emph{global optimism} which requires solving a nonconvex optimization problem. In particular, it has remained open as to whether \emph{computationally efficient} algorithms exist. In this paper we give the first polynomial-time algorithm for RL under linear Bellman completeness when the number of actions is any constant. %
\end{abstract}

\colt{
  }

\section{Introduction}
\emph{Reinforcement learning (RL)} describes the problem of solving {sequential} decision-making problems in a stochastically changing environment, and is typically studied using the formalism of \emph{Markov Decision Processes (MDPs)}. In an MDP, a learning agent must choose a sequence of \emph{actions} over some number of time steps, each of which affects the \emph{state} of the environment and potentially yields some \emph{reward} to the agent. The agent aims to find a \emph{policy}, or a mapping that describes which action to take at each state, that maximizes its expected total reward. In order for RL to yield effective learning strategies in its various application domains, including robotics \cite{gu2017deep}, economics \cite{zheng2022ai}, and healthcare \cite{yu2021reinforcement}, it is necessary %
to come up with efficient strategies for exploring complex state spaces, which may be infinite or exponentially large.

A general approach to RL, which dates back decades \cite{bradtke1996linear,melo2007convergence,sutton2018reinforcement} and yet still forms the basis for many current empirical approaches \cite{hasselt2016deep,schulman2017proximal}, involves the use of {value function approximation}. Recall that the optimal \emph{value function} maps a state-action pair to the agent's expected reward under the optimal policy starting from that state-action pair. %
Then this approach posits that the optimal {value function} belongs to some known function class $\MF$, such as a class of linear functions or neural networks. A key question is: \emph{under what assumptions on the class $\MF$ can we efficiently learn a near-optimal policy, i.e., one with near-maximal expected reward?}

\paragraph{Value iteration and Bellman completeness.} A popular and time-tested approach to finding a near-optimal policy with value function approximation is \emph{value iteration}. To explain this procedure, we consider the \emph{finite-horizon} setting, whereby interactions with the environment occur in \emph{episodes} lasting $H$ time steps. Letting $\MX$ denote the state space and $\MA$ denote the action space, value iteration computes mappings $\hat Q_h : \MX \times \MA \ra \BR$ in a backwards-inductive manner, i.e., for $h = H, H-1, \ldots, 1$. The values $\hat Q_h(x,a)$ should be interpreted as estimates of the optimal value\footnote{I.e., the value of the optimal policy.} at step $h$ given that the state-action pair taken at step $h$ is $(x,a)$. Since the environment's reward functions and transitions are unknown,  $\hat Q_h$ must be estimated empirically. 
Given a dataset $\MD$ consisting of tuples $(x_h, a_h, r_h, x_{h+1})$ of states, actions, and rewards drawn from the environment at step $h$, together with the subsequent state at step $h+1$, 
$\hat Q_h$ is typically chosen to be the function in $\MF$ which minimizes the following square-loss objective whose labels are defined in terms of $\hat Q_{h+1}$:
\begin{align}
\hat Q_h := \argmin_{Q_h \in \MF} \sum_{(x_h, a_h, r_h, x_{h+1}) \in \MD} \left( Q_h(x_h, a_h) - \left( r_h + \max_{a'} \hat Q_{h+1}(x_{h+1}, a') \right) \right)^2\label{eq:lsvi}.
\end{align}
The procedure defined by \cref{eq:lsvi} is often known as \emph{least-squares value iteration} (LSVI) \cite{bradtke1996linear,osband2016generalization}. 
A natural condition under which LSVI might yield good estimates of the optimal value function, and thereby a near-optimal policy, is that the least-squares problem \cref{eq:lsvi} be \emph{well-specified}. This requires that for any $\hat Q_{h+1} \in \MF$, there is some $Q_h' \in \MF$ so that, for all $x_h \in \MX, a_h \in \MA$,
\begin{align}
Q_h'(x_h, a_h) = \E\left[ r_h + \max_{a'} \hat Q_{h+1}(x_{h+1}, a') \ \mid \ (x_h, a_h) \right]\label{eq:bellman-completeness-intro}.
\end{align}
\cref{eq:bellman-completeness-intro} is known to be necessary for LSVI to succeed, in the sense that without it, the value functions computed by LSVI may be wildly divergent \cite{tsitsiklis1996feature}. Moreover, classical results \cite{munos2005error,munos2008finite} showed that if the dataset $\MD$ is sufficiently exploratory, then \cref{eq:bellman-completeness-intro} is sufficient for LSVI to succeed. The requirement that \cref{eq:bellman-completeness-intro} holds for any $\hat Q_{h+1} \in \MF$ is often known as \emph{Bellman completeness}; it is a property of the MDP and the value function class $\MF$. 

In this paper, our focus is on finding computationally efficient algorithms. Since regression problems such as \cref{eq:lsvi} can be computationally intractable for even relatively simple nonlinear classes $\MF$ such as shallow neural networks \cite{dey2020approximation,bakshi2019learning,goel2020tight}, we focus on the case where $\MF$ is simply the class of linear functions in $\phi_h(x,a)$, for some known \emph{feature mappings} $\phi_h: \MX \times \MA \ra \BR^d$. In this setting, the requirement of Bellman completeness in  \cref{eq:bellman-completeness-intro} is known as \emph{linear Bellman completeness} (see \cref{def:lbc} for a formal definition). 
In the remainder of this section we discuss our results on computationally efficient learning of MDPs satisfying linear Bellman completeness. %

\subsection{Exploration and linear Bellman completeness}
\label{sec:online}
The results of  \cite{munos2005error,munos2008finite} referenced above showing sufficiency of Bellman completeness %
do not address the following central problem in RL: \emph{how can we find a dataset $\MD$ which is sufficiently exploratory to compute a near-optimal policy}? More precisely, we would like that for tuples $(x_h, a_h, r_h, x_{h+1}) \in \MD$, the feature vectors $\phi_h(x_h, a_h)$ span a ``sufficient number'' of distinct directions in $\BR^d$. We adopt the standard online setting in RL, allowing the learning algorithm to form $\MD$ by repeatedly sampling trajectories from the MDP using adaptively chosen policies.\footnote{See \cref{sec:online-problem} for a formal definition.} We first note that this problem is known to be {statistically tractable} using a technique known as \emph{global optimism} (see \cref{sec:optimism-overview}). This technique has various instantiations as specific algorithms in the setting of linear Bellman completeness, including  \texttt{ELEANOR} \cite{zanette2020learning}, \texttt{GOLF} \cite{jin2021bellman}, \texttt{OLIVE} \cite{jiang2017contextual}, and \texttt{BilinUCB} \cite{du2021bilinear}.
While these algorithms require only polynomially many samples (i.e., rounds of interaction) to  output a near-optimal policy, all are computationally inefficient, even when specialized to the linear case, since 
they require solving  nonconvex optimization problems. Our goal is to find a computationally efficient algorithm which achieves the same guarantee: %
\begin{question}
  \label{ques:main-problem}
Is there an algorithm which learns an $\ep$-optimal policy in an unknown linear Bellman complete MDP using $\poly(H, d, |\MA|, \ep^{-1})$ samples and time?
\end{question}

A sizeable portion of the work on computationally efficient RL in the last several years has been focused on answering \cref{ques:main-problem} for settings which are strict special cases of linear Bellman completeness. The simplest such setting is the \emph{tabular setting}, which describes the case that $|\MX|, |\MA|$ are finite and the goal is to obtain sample and computational complexities scaling as $\poly(H, |\MX|, |\MA|)$. In this setting, there are several computationally efficient algorithms which can be viewed as variants of value iteration that are \emph{optimistic} in the sense that they add bonuses to the rewards to induce exploration: these include \texttt{UCBVI} \cite{azar2017minimax} and \texttt{Q-learning-UCB} \cite{jin2018q,zhang2020almost}, which are known to obtain near-optimal rates. Tabular MDPs are generalized by the \emph{linear MDP} setting, in which feature vectors $\phi_h(x,a) \in \BR^d$ are given, and the state-action transition probabilities are assumed to be linear in $\phi_h$. Here too there are computationally efficient algorithms, namely \texttt{LSVI-UCB} \cite{jin2020provably}, an optimistic version of LSVI, as well as more recent rate-optimal variants \cite{agarwal2022voql,he2023nearly}.

In the setting of linear Bellman completeness, which is a strict generalization of linear MDPs \cite[Proposition 3]{zanette2020learning}, the algorithm \texttt{FRANCIS} of \cite{zanette2020provably} is computationally efficient and learns a near-optimal policy in the special case that the MDP is \emph{reachable}, meaning that any direction in $\BR^d$ can be reached under some policy. The setting of linear Bellman completeness has also been studied in the special case of \emph{deterministic} dynamics: \cite{wen2016efficient} (Theorem 1) showed that there is a polynomial-time learning algorithm when the transitions and rewards are deterministic, and \cite{du2020agnostic} established a polynomial-time algorithm for the setting of deterministic transitions, stochastic rewards, and positive suboptimality gap. Finally, in recent and concurrent work, \cite{wu2024computationally} established a polynomial-time algorithm for the setting of deterministic transitions and stochastic rewards (without the assumption of positive suboptimality gap). The algorithm of \cite{wu2024computationally} has the additional advantage of working under weaker norm assumptions (namely, the second item of \cref{asm:boundedness} is not required). Finally, we mention that in the related \emph{offline} setting of reinforcement learning, the analogue of \cref{ques:main-problem} was recently resolved in concurrent work \cite{golowich2024role}. 

Despite the above line of work, \cref{ques:main-problem} in its full generality has remained open. Part of the reason for this is that the two principal techniques to perform computationally efficient exploration in RL both break down in the general setting of linear Bellman completeness:
\begin{enumerate}[wide,labelwidth=!,labelindent=0pt]
\item \label{it:local-optimism} The first technique is \emph{local optimism},\footnote{``Local'' is used to distinguish this technique from \emph{global optimism}, which, as discussed above, works in the setting of linear Bellman completeness but is computationally inefficient.} which adds exploration bonuses to the reward at each state which scale inversely with how often the state has been visited. It includes the \texttt{UCBVI} and \texttt{Q-learning-UCB} algorithms for the tabular setting, and the \texttt{LSVI-UCB} algorithm for the linear MDP setting, among others. Local optimism requires that the value function class be complete with respect to the exploration bonuses, which is satisfied for linear MDPs but which fails more generally (see \cref{sec:overview-online}).
\item \label{it:policy-cover} The second technique is to construct a \emph{policy cover}, which is a small set of policies which, roughly speaking, covers all states to the maximum extent possible. This technique includes the \texttt{FRANCIS} algorithm, as well as computationally efficient and oracle efficient algorithms for related tasks in RL such as representation learning \cite{du2019provably,misra2020kinematic,mhammedi2023representation,mhammedi2023efficient,golowich2023exploring} and learning in POMDPs \cite{golowich2022learning}. In order for this approach to work in the absence of reachability, it is necessary to analyze a \emph{truncated} version of the true MDP. But doing so seems impossible in our setting, since truncating the MDP breaks the property of linear Bellman completeness. 
\end{enumerate}
Thus, in addition to generalizing a long line of work on computationally efficient learning of MDPs, \cref{ques:main-problem} captures exactly the point where known exploration paradigms in RL break down.  

\subsubsection{Main result for the online setting}
Our main result is a positive answer to \cref{ques:main-problem} in the case that $|\MA|$ is any constant:
\colt{\vspace{-0.1cm}}
\begin{theorem}[Informal version of \cref{thm:policy-learning}]
  \label{thm:online-intro}
Suppose the ground-truth MDP satisfies linear Bellman completeness. Then for any $\ep > 0$, there is an algorithm (\PSDPUCB; \cref{alg:psdp-ucb}) which with high probability learns an $\ep$-optimal policy using $ (Hd|\MA|\ep^{-1})^{O(|\MA|)} $ samples and time. 
\end{theorem}
\colt{\vspace{-0.1cm}}
We remark that the exponential dependence of the sample and computational complexities on $|\MA|$ is somewhat unusual in RL. We are not aware of any prior work on an RL problem for which $\poly(|\MA|)$ dependence is possible but for which a ``natural'' algorithm for even the $|\MA| = 2$ case does not extend to the general case. 
\cref{thm:online-intro} thus presents an intriguing challenge: to either improve the guarantee to fully answer \cref{ques:main-problem} in the affirmative, or to find a lower bound. %

\arxiv{
  \paragraph{Proof idea of \cref{thm:online-intro}.} The proof of \cref{thm:online-intro} proceeds via the local optimism approach discussed in \cref{it:local-optimism} above. Our main insight is that it is possible to design exploration bonuses for which the class of linear value functions \emph{is} complete with respect to them, i.e., whose Bellman-backup is linear (formally, we say that such exploration bonuses are \emph{Bellman-linear}; see \cref{def:bl}). Doing so requires extensive work to ``build up'' to a function which satisfies several properties required of the exploration bonuses; see \cref{sec:overview-online} for further details. %
We remark that our algorithm \PSDPUCB takes the structure of an optimistic variant of the classical \texttt{PSDP} (Policy Search by Dynamic Programming) algorithm \cite{bagnell2003policy}, which has not previously appeared in the literature. This choice of algorithm structure plays more nicely with the exploration bonuses we introduce.

We find the fact that local optimism provides a solution to \cref{ques:main-problem}, even in the case $|\MA| = O(1)$, to be somewhat surprising. On the one hand, it indicates that Bellman completeness is a stronger property than one might initially anticipate. On the other hand, since Bellman-completeness is satisfied in other settings for which the only known computationally efficient algorithms are policy-cover based (\cref{it:policy-cover} above), our results perhaps provide some hope that local optimism can be made to work in such settings as well. 
}
\paragraph{Organization of the paper.} 
We discuss preliminaries in \cref{sec:prelim}, and overview the proof of \cref{thm:online-intro} in \cref{sec:overview-online}, which is proved formally in \cref{sec:bellman-linear,sec:alg-description}.  %

\section{Preliminaries}
\label{sec:prelim}
A finite-horizon \emph{Markov Decision Process (MDP)} is given by a tuple $M = (H, \MX, \MA, (P_h\sups{M})_{h=1}^H, (r_h\sups{M})_{h=1}^H, d_1\sups{M})$, where $H \in \BN$, $\MX$ is a measure space denoting the {state set}, $\MA$ denotes the {action set}, $P_h\sups{M}(\cdot | x,a) \in \Delta(\MX)$ (for $h \in [H]$) denotes the probability transition kernels, $r_h\sups{M} : \MX \times \MA \ra [0,1]$ (for $h \in [H]$) denotes the reward functions, and $d_1\sups{M} \in \Delta(\MX)$ denotes the initial state distribution. When the MDP $M$ is clear from context, we will drop the superscript $M$ in these notations. We let $A := |\MA|$ denote the number of actions in the MDP. 

A \emph{policy} $\pi$ consists of a tuple $\pi = (\pi_1, \ldots, \pi_H)$, where each $\pi_h : \MX \ra \Delta(\MA)$ is a mapping from states to distributions over actions. Note that we allow policies to be nonstationary and randomized; $\Pi$ denotes the set of all such policies. A policy $\pi \in \Pi$ defines a distribution over \emph{trajectories} $(x_1, a_1, r_1, \ldots, x_H, a_H, r_H) \in (\MX \times \MA \times [0,1]^H$, as follows: first, $x_1 \sim d_1$, and then for each $h \in [H]$, we draw $a_h \sim \pi_h(x_h)$, observe $r_h(x_h, a_h)$, and transition to $x_{h+1} \sim P_h(\cdot | x_h, a_h)$. Let $\SH_H$ denote the set of trajectories. For a function $f : \SH_H \ra \BR^k$, we often write $\E\sups{M, \pi}[f(x_1, a_1, r_1, \ldots, x_H, a_H, r_H)]$ to denote the expectation of $f$ over trajectories drawn from $M$ under policy $\pi$. If $M$ is clear from context, we will simply write $\E^\pi[f(x_1, a_1, r_1, \ldots, x_H, a_H, r_H)]$. Given policies $\pi = (\pi_1, \ldots, \pi_H), \pi' = (\pi_1', \ldots, \pi_H') \in \Pi$ and a step $h \in [H]$, we let $\pi \circ_h \pi'\in \Pi$ denote the policy $(\pi_1, \ldots, \pi_{h-1}, \pi_h', \ldots, \pi_H')$, i.e., which acts according to $\pi$ during the first $h-1$ steps and thereafter acts according to $\pi'$. 

For a function $f : \MX \times \MA \ra \BR$, a (randomized) policy $\pi$ and $h \in [H]$, we write $f(x, \pi_h(x)) := \E_{a \sim \pi_h(x)}[f(x,a)]$. The state-action value function (or \emph{$Q$-function}) and state-value function (or \emph{$V$-function}) of a policy $\pi \in \Pi$ are then defined as follows: for $h \in [H], x \in \MX, a \in \MA$, 
\begin{align}
Q_h^\pi(x,a) := r_h(x,a) + \E^\pi \left[ \sum_{g=h+1}^H r_g(x_g, a_g) \ \mid \ (x_h, a_h) = (x,a) \right], \qquad V_h^\pi(x) := Q_h^\pi(x, \pi_h(x)) \nonumber.
\end{align}
The \emph{optimal policy} is defined as $\pi^\st := \argmax_{\pi \in \Pi} \E[V_1^\pi(x_1)]$ (where expectation is over $x_1 \sim d_1$). It is known that there is always a deterministic optimal policy $\pi^\st$ (i.e., so that $\pi_h^\st(x)$ is a singleton for all $x,h$). We often abbreviate $Q_h^\st(x,a) := Q_h^{\pi^\st}(x,a)$ and $V_h^\st(x) := V_h^{\pi^\st}(x)$. 

\subsection{The online learning problem}
\label{sec:online-problem}
 We assume that the transitions, rewards, and initial state distribution of the ground-truth MDP $M$ are \emph{unknown}  to the algorithm. 
 To learn information about the MDP, the algorithm interacts with $M$ via the \emph{ episodic online learning model}, as follows. The interaction proceeds over a series of $T$ episodes.  %
 In each step $h \in [H]$ of each episode, the algorithm observes the current state $x_h$, specifies an action $a_h \in \MA$ to take, and observes a reward of $r_h = r_h(x_h, a_h)$. Then the environment transitions to a new state $x_{h+1} \sim P_h(\cdot | x_h, a_h)$. We assume that the algorithm can query $\phi_h(x,a) \in \BR^d$ for any $x,a,h$. We remark that our algorithm only needs to query $\phi_h(x,a)$ for states $x$ which are visited at some point in some episode. 

The goal is as follows: for $\ep, \delta \in (0,1)$, to give an algorithm which interacts with the environment for $T = T(\ep, \delta)$ episodes in the manner described above, and then to output a policy $\hat \pi$ so that, with probability $1-\delta$, $\E[V_1^\st(x_1) - V_1^{\hat \pi}(x_1)] \leq \ep$. Moreover, we wish the algorithm to be computationally efficient (both in terms of the computation required when interacting with the environment over the course of  $T$ episodes and when outputting $\hat \pi$ at  termination).

\subsection{Function approximation}
In order for the above learning problem to be tractable, it is necessary to make some assumptions on the ground-truth MDP $M$ being learned. While it is well-known that boundedness of $|\MX|, |\MA|$ implies  efficient learning algorithms whose computational and statistical costs scale with $\poly(|\MX|, |\MA|)$ \cite{azar2017minimax,jin2018q}, realistic RL environments typically have enormous state spaces. Thus we aim for weaker function approximation assumptions: a common and longstanding such assumption \cite{bradtke1996linear,melo2007convergence,sutton2018reinforcement,yang2020reinforcement,jin2020provably} is \emph{linearity} of the value functions, with respect to some known features. In particular, given $d \in \BN$, we assume that functions $\phi_h : \MX \times \MA \ra \BR^d$ are given for all $h \in [H]$, mapping each state-action pair $(x,a)$ to a collection of $d$ \emph{features} which should be interpreted as capturing all relevant information about $(x,a)$.

The weakest assumption on value function linearity is simply that $Q_h^\st$ is linear, i.e., for some $w_h^\st \in \BR^d$, we have $Q_h^\st(x,a) = \lng w_h^\st, \phi_h(x,a) \rng$ for all $(x,a) \in \MX \times \MA$. Unfortunately, unless $\mathsf{NP} = \mathsf{RP}$, it is not possible to  computationally efficiently learn a near-optimal policy under this assumption \cite{kane2022computational,liu2023exponential}. Accordingly, we make a stronger assumption known as \emph{linear Bellman completeness}, which states, roughly speaking, that the Bellman backup of all linear functions is linear. To formally state this assumption (in \cref{def:lbc} below), we  need to introduce the following notation:  for $h \in [H]$, define \colt{$\MB_h := \{ \theta \in \BR^d \ : \ | \lng \phi_h(x,a), \theta \rng | \leq 1  \ \ \forall (x,a) \in \MX \times \MA \}$.}
\arxiv{\begin{align}
\MB_h := \{ \theta \in \BR^d \ : \ | \lng \phi_h(x,a), \theta \rng | \leq 1  \ \ \forall (x,a) \in \MX \times \MA \}\nonumber.
\end{align}}
In words, $\MB_h$ denotes the set of coefficient vectors inducing bounded (linear) functions on $\MX \times \MA$. 
\begin{definition}[Linear Bellman Completeness]
  \label{def:lbc}
 The MDP $M$ is defined to be \emph{linear Bellman complete} with respect to the feature mappings $(\phi_h)_{h \in [H]}$ if, for each $h \in [H]$, there is a mapping  $\MT_h : \MB_{h+1} \ra \MB_h$ so that, for all $\theta \in \MB_{h+1}$ and all $(x,a) \in \MX \times \MA$, \colt{$\lng \phi_h(x,a), \MT_h \theta \rng = \E_{x' \sim P_h(x,a)} \left[ \max_{a' \in \MA} \lng \phi_{h+1}(x', a'), \theta \rng \right]$.}
\arxiv{\begin{align}
\lng \phi_h(x,a), \MT_h \theta \rng = \E_{x' \sim P_h(x,a)} \left[ \max_{a' \in \MA} \lng \phi_{h+1}(x', a'), \theta \rng \right]\nonumber.
\end{align}}
Moreover, we require that for all $h \in [H]$, $(x,a) \in \MX \times \MA$, the reward at $(x,a)$ at step $h$ is given by: \colt{$r_h(x,a) := \lng \phi_h(x,a), \theta_h^{\mathsf r} \rng$}
\arxiv{\begin{align}
r_h(x,a) := \lng \phi_h(x,a), \theta_h^{\mathsf r} \rng\nonumber,
\end{align}}
for some vectors $\theta_h^{\mathsf r} \in \MB_h$. 
\end{definition}
It is immediate from \cref{def:lbc} that if $M$ is linear Bellman complete, then there are vectors $w_h^\st \in H \cdot \MB_h$ so that $Q_h^\st(x,a) = \lng w_h^\st, \phi_h(x,a) \rng$ for all $x,a$. 
In addition to linear Bellman completeness, we make the following standard boundedness assumptions on the coefficient and feature vectors:
\begin{assumption}[Boundedness]
  \label{asm:boundedness}
  We assume the following:
  \begin{enumerate}
    \item For all $h \in [H], x \in \MX, a \in \MA$, we have $\| \phi_h(x,a) \|_2 \leq 1$. 
    \item For some parameter $\Bbnd \in \BR_+$: for all $w_h \in \MB_h$, it holds that $\| w_h \|_2 \leq \Bbnd$.
    \item For all $h \in [H]$, $\| \theta_h^\mathsf{r} \|_2 \leq 1$ (and hence $\sup_{x,a,h} | r_h(x,a)| \leq 1$).  %
    \end{enumerate}
  \end{assumption}

\colt{  We present additional preliminaries in \cref{sec:additional-prelim}. Here we highlight the definition of a linear policy, which is defined as follows: for $w \in \BR^d$ and $h \in [H]$, the associated linear policy at step $h$ is $\pi_{h,w}(x) := \argmax_{a \in \MA} \lng w, \phi_h(x,a) \rng$ \emph{if the argmax is unique}. We discuss in \cref{sec:additional-prelim} how to deal with the (typically measure-zero) situation that the argmax is not unique. %
}

\arxiv{\colt{\section{Additional preliminaries}
  \label{sec:additional-prelim}}
\colt{In this section, we present several additional preliminaries that will be useful in our proofs.} 
\paragraph{Polytope of actions.} 
We assume that $\MA$ is finite and write $A := |\MA|$. Let $\SP^d$ denote the space of $d$-dimensional (bounded)  polyhedra. For $A \in \BN$, let $\SP_A^d \subset \SP^d$ denote the space of $d$-dimensional polyhedra $\Phi$ so that $\dim (\Span \Phi) \leq A$. For each $x \in \MX$, we define $\Phi_h(x) := \{ \phi_h(x,a) :\ a \in \MA\} \subset \BR^d$ and $\cPhi_h(x) := \mathrm{co} (\Phi_h(x)) \in \SP_A^d \subset \SP^d$. %

\paragraph{Gaussian smoothing.} For $\theta \in \BR^d$ and $\sigma > 0$, we write
\begin{align}
\MN_\sigma(\theta)  := \MN(0, \sigma^2 \cdot I_d)(\theta) = \frac{1}{(2\pi)^{d/2} \sigma^d} \cdot \exp \left( -\frac{1}{2\sigma^2} \| \theta \|_2^2 \right)\nonumber
\end{align}
to denote the probability density function of the standard normal distribution with covariance $\sigma^2 I_d$. 
Furthermore, for $f : \BR^d \ra \BR$, we write $\Sm f (\theta)$ to denote the convolution of $f$ with $\MN_\sigma$, namely
\begin{align}
\Sm f(\theta) := \int_{\BR^d} f(z) \MN_\sigma(\theta - z) dz = \int_{\BR^d} f(\theta-z) \MN_\sigma(z) dz = \E_{z \sim \MN(0, \sigma^2 \cdot I_d)}[f(\theta -z)]\nonumber.
\end{align}

\paragraph{Miscellaneous notation.} Given $d \in \BN$, we let $\BS^d$ denote the space of symmetric $d \times d$ matrices, $\BS_+^d \subset \BS^d$ denote the space of positive semidefinite (PSD) matrices, and $\BS_{++}^d \subset \BS_+^d$ denote the space of positive definite matrices. For a PSD matrix $T \in \BS_+^d$, we let $T^{1/2}$ denote the unique PSD matrix whose square is $T$. Given a subset $\MS \subset \BR^d$ and a matrix $T \in \BR^{d \times d}$, we let $T \cdot \MS := \{ Tv :\ v \in \MS \}$. For vectors $v,v' \in \BR^d$, we let $[v,v']$ denote the segment $[v,v'] := \{ \lambda v + (1-\lambda) v' :\ \lambda \in [0,1]\}$. Let $S^{d-1} = \{ u \in \BR^d:\ \| u \|_2 = 1\}$ denote the $d$-dimensional unit sphere. For a square matrix $T$, we let $\| T \|$ denote its spectral norm.

\subsection{Linear policies}
\label{sec:linpol}
Under the assumption of linear Bellman completeness, there is always an optimal policy with the additional structure of being a \emph{linear policy}. Roughly speaking, linear policies take an action defined by the argmax with respect to some fixed coefficient vector: in particular, %
given $w \in \BR^d$ and $h \in [H]$, the associated linear policy at step $h$ is $\pi_{h,w}(x) := \argmax_{a \in \MA} \lng w, \phi_h(x,a) \rng$ \emph{if the argmax is unique}. It requires some care to appropriately break ties for states $x$ at which the $\argmax$ is not unique.\footnote{In particular, an appropriate tie-breaking procedure is  necessary for \cref{lem:lin-lb} to hold.} To do so, given $w \in \BR^d, h \in [H], x \in \MX$, define $\MA_{h,w}(x) := \argmax_{a \in \MA} \lng w, \phi_h(x,a) \rng \subset \MA$, where $\argmax$ is interpreted as the set of all actions maximizing $\lng w, \phi_h(x,a) \rng$. Then we set, for each $a \in \MA_{h,w}(x)$, 
\begin{align}
  \MG_{h,w}(x,a) := \left\{ \theta \in \BR^d : \ \| \theta\|_2 = 1,\ \lng \theta,
 \phi_h(x,a) > \max_{a' \in \MA_{h,w}(x) \backslash \{ a \}} \lng \theta, \phi_h(x,a') \right\}\nonumber.
\end{align}
Let $\nu_d$ denote the spherical measure on $S^{d-1}$. 
It is straightforward to see that, for all $w$ and $x$, $S^{d-1} \backslash \bigcup_{a \in \MA_{h,w}(x)} \MG_{h,w}(x,a)$ has measure 0 with respect to $\nu_d$. We now define, $\pi_{h,w} : x \ra \Delta(\MA)$ to be the following randomized policy: for all $w \in \BR^d$, $h \in [H]$, $x \in \MX$, $a \in \MA$,  
\begin{align}
  \label{eq:linpol-definition}
  \pi_{h,w}(a|x) := \One{a \in \MA_{h,w}(x)} \cdot \nu_d(\MG_{h,w}(x,a)).
\end{align}

It is straightfroward to see that $\pi_{h,w}(\cdot | x) \in \Delta(\MA)$. We say that a policy $\pi$ is a \emph{linear policy} if it is of the form $\pi = (\pi_{1,w_1}, \ldots, \pi_{H, w_H})$, for $w_1, \ldots, w_H \in \BR^d$. We let $\Pilin_h := \{ \pi_{h,w} :\ w \in \BR^d\}$ denote the space of linear policies at step $h$, and $\Pilin := \{ (\pi_1, \ldots, \pi_H) :\ \pi_h \in \Pilin_h \}$ denote the space of linear policies.

}

\section{Technical overview}
\label{sec:overview-online}
In this section, we overview the proof of \cref{thm:online-intro} (stated formally in \cref{thm:policy-learning}), which shows how, in the online setting (\cref{sec:online-problem}), \cref{alg:psdp-ucb} can efficiently learn a near-optimal policy for an unknown MDP which is linear Bellman complete. To simplify our notation, we assume in this section that the parameter $\Bbnd$ in \cref{asm:boundedness} is bounded by $\Bbnd \leq O(1)$. The bulk of the challenge is to perform the task of \emph{exploration}: how can we interact with the environment so as to reach  state-action pairs $(x,a)$ for which $\phi_h(x,a)$ points in a new direction?

\subsection{Prior work: exploration via optimism} \label{sec:optimism-overview}
A popular approach to exploration in RL involves the use of \emph{optimism}, which describes, loosely speaking, the technique of perturbing the algorithm's estimates of the MDP's value function, typically to \emph{increase} the estimated values, so as to induce the algorithm to visit new directions in feature space. Two distinct flavors of optimism have emerged in the literature: the first, \emph{global optimism},\footnote{See, e.g., \cite{zanette2020learning}, which used global optimism to  computationally  inefficiently learn a near-optimal policy under linear Bellman completeness; many other papers, including \cite{jin2021bellman,jiang2017contextual,du2021bilinear}, use global optimism in settings with more general nonlinear function approximation.} constructs a confidence set consisting of all possible vectors $w = (w_1, \ldots, w_H)$ which could be consistent with the optimal value function given the data observed so far. The global optimism procedure then chooses some $\bar w$ in this confidence set which maximizes $\E[\max_a \lng \bar w_1, \phi_1(x_1, a) \rng]$. It then executes the linear policy defined by $\bar w$, namely the policy $(\pi_{1, \bar w_1}, \ldots, \pi_{H, \bar w_H})$, uses the resulting data to update the confidence sets, and repeats. Unfortunately, the optimization problem of finding a maximizing $\bar w$ is nonconvex, and seems unlikely to have an efficient algorithm.

The second type of optimism-based exploration technique in RL is a more \emph{local} approach: at each episode $t$, for some function $B_h\^t : \MX \times \MA \ra \BR_{\geq 0}$, an \emph{exploration bonus} of $B_h\^t(x,a)$ is added to the reward $r_h(x,a)$ for the pair $(x,a)$ at step $h$. Then, the algorithm uses the data gathered prior to episode $t$ to estimate a policy $\pi\^t$ which, roughly speaking, maximizes the expected sum of rewards $r_h(x_h,a_h)$ and bonuses $B_h\^t(x_h,a_h)$ over a trajectory. This technique has proved successful for computationally efficient learning of tabular MDPs (i.e., the setting where $|\MX|, |\MA|$ are finite) \cite{azar2017minimax,jin2018q}, as well as \emph{linear MDPs} \cite{jin2020provably}, which constitute a  generalization of tabular MDPs and a strict subclass of linear Bellman complete MDPs.

Unfortunately, the local optimism approach fails in the more general setting of linear Bellman completeness. At a high level, this failure of local optimism results from the fact that the exploration bonuses $B_h\^t(x,a)$ may not be linear. To illustrate, we consider the \texttt{LSVI-UCB} approach of \cite{jin2020provably}. This approach computes optimistic $Q$-function estimates, $\hat Q_h\^t$, in a backwards-inductive manner. In particular, it defines:
\begin{align}
  \label{eq:lsvi-q}
  \hat Q_h\^t(x,a) := \min\{ \lng \hat w_h\^t, \phi_h(x,a) \rng + B_h\^t(x,a), H \}, %
\end{align}
where $\hat w_h\^t$ is the solution to a least-squares objective function whose labels are given by $\max_{a \in \MA} \hat Q_{h+1}\^t(x^i, a)$, for various states $x^i$.\footnote{The choice of the scalar $H$ in the minimum in \cref{eq:lsvi-q} results from the fact that the $Q$-function of any policy is bounded above in absolute value by $H$, as the reward at each step has absolute value at most 1.} %
The bonus $B_h\^t(x,a)$ is defined as follows:   let $\Sigma_h\^t$ be the covariance matrix of features at step $h$ observed prior to episode $t$. Then $B_h\^t$ is given by a scaling of the \emph{quadratic bonus}, i.e., for some scalar $\beta_h$, $B_h\^t(x,a) := \beta_h \cdot\Bquad_{h}(x,a; (\Sigma_h\^t)^{-1})$, where the quadratic bonus, $\Bquad_{h}(x,a; \Sigma)$, is defined for a general PSD matrix $\Sigma$, by
\begin{align}
  \label{eq:quad-bonus}
  \Bquad_{h}(x,a; \Sigma) := (\phi_h(x,a)^\t \cdot \Sigma \cdot \phi_h(x,a))^{1/2}.
\end{align}
The intuition behind the quadratic bonus is as follows: $\Bquad_{h}(x,a; (\Sigma_h\^t)^{-1})$ will be particularly large if $\phi_h(x,a)$ points in the direction of eigenvectors of $\Sigma_h\^t$ with small eigenvalues, i.e., directions which have not been explored in prior episodes. Thus such \emph{unexplored} states receive larger bonuses, and should be explored more during later episodes. More formally, $\Bquad_h(\cdot)$ takes the same form as the standard error bound from least-squares regression (see \cref{lem:phi-what-wt-diff}), meaning that by adding it to the reward function, one can ``cancel out'' regression errors and thus show that $\hat Q_h\^t$ in \cref{eq:lsvi-q} is optimistic (see \cite{jin2020provably,agarwal2022voql,zhang2022efficient,he2023nearly}). 

To ensure that $\hat w_h\^t$, used in the definition of $\hat Q_h\^t$, is a good estimator of the future rewards and bonuses, we certainly need the regression problem to be \emph{well-specified}, i.e., the expectation of the regression labels is a linear function in the features $\phi_h(x,a)$. %
In particular, this approach crucially relies on the fact that, in a linear MDP, %
$\E_{x' \sim P_h(x,a)} [\max_{a'} \hat Q_{h+1}\^t(x', a')]$, is a linear function of $\phi_h(x,a)$. In fact, an even stronger statement holds for linear MDPs: for \emph{any} function $F : \MX \ra \BR$, its \emph{Bellman backup}, namely $\E_{x' \sim P_h(x,a)}[F(x')]$, is linear in $\phi_h(x,a)$. Unfortunately, this fact fails to hold under the weaker assumption of linear Bellman completeness. We provide an example in \cref{prop:lsvi-counterexample} for which $\E_{x' \sim P_h(x,a)}[\max_{a'} \hat Q_{h+1}\^t(x',a')]$ is \emph{not linear}, for some $\hat Q_{h+1}\^t$ as in \cref{eq:lsvi-q}.

Our main innovation is to show that if the bonus $B_h\^t(x,a)$ is carefully defined to be somewhat different from a quadratic bonus, then a variant of \cref{eq:lsvi-q} \emph{does} have a linear Bellman backup. In the remainder of this section, we discuss in detail how to execute this strategy: in \cref{sec:psdp-ucb-overview}, we first discuss the overall structure of our algorithm, \texttt{PSDP-UCB}, which is a variant of the \texttt{LSVI-UCB} algorithm discussed above but which lends itself to a simpler analysis for the setting of linear Bellman completeness. Then, in \cref{sec:bellman-linearity-overview,sec:gena-overview}, we discuss how to construct exploration bonuses for use in \texttt{PSDP-UCB} which do have a linear Bellman backup.  

\subsection{Overview of \texttt{PSDP-UCB}}
\label{sec:psdp-ucb-overview}
\begin{algorithm}[ht]
  \caption{\texttt{PSDP-UCB}$(T, H, \lambda, \beta, \lambda_1)$: \texttt{PSDP} with upper confidence bounds}
	\label{alg:psdp-ucb}
	\begin{algorithmic}[1]\onehalfspacing
      \Require{Number of episodes $T$, horizon $H$, non-negative parameters $\lambda, \beta, \lambda_1$.}
\For{Round $t = 1,\ldots, T$}
  \For{Step $h = H, \ldots 1$}
    \State For $i \in [n]$, draw \arxiv{samples} $\{ (x_k\^{t,i,h}, a_k\^{t,i,h}, r_k\^{t,i,h})\}_{k=1}^H$ \arxiv{i.i.d.~}from $\unif(\{\hat \pi\^s \circ_h \tilde \pi\^s \circ_{h+1} \hat \pi\^t\}_{s=1}^{t-1}\})$.\label{line:collect-tih-samples}
    \State Set $\Sigma_h\^{t} \gets \lambda I +\sum_{i=1}^{n} \phi_h(x_h\^{t,i,h}, a_h\^{t,i,h}) \phi_h(x_h\^{t,i,h},a_h\^{t,i,h})^\t$.  \label{line:define-sigmaht}    
  \State \label{line:rhat-rewards} For each $g > h$ and $i \in [n]$, define \Comment{\emph{$F_g\^t$ defined in \cref{eq:define-fht-bonus}.}}   %
  \begin{align}
    \hat r_g\^{t,i,h} := r_g\^{t,i,h} + F_g\^t(x_g\^{t,i,h})\label{eq:rhat-rewards}. %
  \end{align}
    \State Set $\hat w_h\^t \gets (\Sigma_h\^{t})^{-1} \cdot \sum_{i=1}^{n} \phi_h(x_h\^{t,i,h}, a_h\^{t,i,h}) \cdot \left(r_h\^{t,i,h} + \sum_{g=h+1}^H \hat r_g\^{t,i,h}\right)$.\label{line:define-w-hat}
    \State Define $\hat \pi_h\^t$ by $\hat \pi_h\^t(x) := \argmax_{a \in \MA} \lng \phi_h(x,a), \hat  w_h\^t \rng$. \emph{(If $t=1$, let $\hat \pi\^t$ be arbitrary.)}\label{line:define-pi-hat}
    \State  Define $(\Sigma', \Lambda') := \trunc{\frac{\beta}{\lambda_1} \cdot (\Sigma_h\^t)^{-1/2}}{\sigtr}$ (per \cref{def:mat-truncation}). \label{line:define-truncation}
    \State  \multiline{Define $\tilde \pi_h\^t : \MX \ra \Delta(\MA)$ by, for all $(x,a) \in \MX \times \MA$, \colt{ $\tilde \pi_h\^t(x)(a) := \PR_{w \sim \MN(0, \Sigma')}\left( a = \argmax_{a' \in \MA} \lng w, \phi_h(x, a') \right) $}. 
    \arxiv{
      \begin{align}
\tilde \pi_h\^t(x)(a) := \PR_{w \sim \MN(0, \Sigma')}\left( a = \argmax_{a' \in \MA} \lng w, \phi_h(x, a') \right)  \nonumber.
      \end{align}
      }}\label{line:define-pi-tilde}

    \EndFor
\EndFor

    \end{algorithmic}
  \end{algorithm}

Our algorithm, \PSDPUCB, is presented in \cref{alg:psdp-ucb}. For some $T \in \BN$, the algorithm proceeds over $T$ \emph{rounds}. In each round $t \in [T]$, \PSDPUCB constructs a policy $\hat \pi\^t = (\hat \pi_1\^t, \ldots, \hat \pi_H\^t)$, which may be interpreted as an estimate of the optimal policy. We will show that, for $T$ sufficiently large, with high probability there is some $t\in[T]$ so that $\E[V_1^\st(x_1) - V_1^{\hat \pi\^t}(x_1)]$ is small.

In each round $t \in [T]$, the algorithm iterates through steps $h = H, H-1, \ldots, 1$: for each such $h$, it will define the mapping $\hat \pi_h\^t : \MX \ra \Delta(\MA)$. Thus, at step $h$, $\hat \pi_{h+1}\^t, \ldots, \hat \pi_H\^t$ have already been defined. At each step $h$, in  \cref{line:collect-tih-samples}, the algorithm samples some number $n$ of trajectories from the MDP according to a uniformly random policy from (roughly) the set $\{ \hat \pi\^s \circ_{h+1} \hat \pi\^t \}_{s \in [t-1]}$.\footnote{\label{fn:pitilde} Technically, this statement is slightly inaccurate: \cref{alg:psdp-ucb} actually samples a uniform policy from the set $\{ \hat \pi\^s \circ_h \tilde \pi_h\^s \circ_{h+1} \hat \pi\^t\}_{s \in [t-1]}$, where $\tilde \pi_h\^s$ is defined for each $s$ to be a policy which performs a sort of ``uniform'' exploration at each step $h$ (\cref{line:define-pi-tilde}). This choice is made to simplify the analysis and is immaterial for our present discussion.} Next, in \cref{line:define-sigmaht}, the algorithm sets $\Sigma_h\^t$ to be the empirical covariance matrix of features at step $h$ under the $n$ trajectories just sampled.

\PSDPUCB then modifies the rewards in each of these $n$ trajectories by adding to the true rewards an exploration bonus, denoted by $F_g\^t(\cdot)$ in \cref{line:rhat-rewards}. We call these modified rewards, denoted by $\hat r_h\^{t,i,h}$ in \cref{alg:psdp-ucb}, \emph{optimistic rewards}. %
The remainder of this section will be focused on defining and explaining the intuition behind the bonus function $F_g\^t$ (the formal definition is given in \cref{eq:define-fht-bonus}). Then, in \cref{line:define-w-hat}, \PSDPUCB sets $\hat w_h\^t$ to be the solution to a least-squares regression problem where the covariance matrix is $\Sigma_h\^t$ and the labels are the cumulative optimistic rewards from step $h$ to $H$. The algorithm then repeats this procedure for step $h-1$, and so on.

We have omitted a description of \cref{line:define-truncation,line:define-pi-tilde}, in which a policy at step $h$, $\tilde \pi_h\^t: \MX \ra \Delta(\MA)$, is defined. This policy plays a relatively minor role (see \cref{fn:pitilde}) and may be ignored for the purpose of the present discussion.

\paragraph{Overview of the analysis of \PSDPUCB.} To analyze \PSDPUCB, we define
\begin{align}
  Q_h\^t(x,a) :=&  r_h(x,a) +  \E^{\hat \pi\^t}\left[ \sum_{g=h+1}^H r_g(x_g, a_g) +  F_g\^t(x_g)\ \mid \ (x_h,a_h) = (x,a) \right]\label{eq:overview-qht-definition},
\end{align}
and $V_h\^t(x) := Q_h\^t(x, \hat \pi_h\^t(x))$. $Q_h\^t$ represents the $Q$-value function corresponding to the policy $\hat \pi\^t$ and the optimistic rewards discussed above. Our main technical lemma is the following statement showing that adding the exploration bonuses $F_h\^t(\cdot)$ to the rewards ensures that $Q_h\^t$ is an upper bound on $Q_h^\st$:
\begin{lemma}[Informal version of \cref{lem:approx-optimism}]
  \label{lem:approx-optimism-informal}
With high probability over the execution of \PSDPUCB (\cref{alg:psdp-ucb}), for all $x,a,h,t$, we have $Q_h\^t(x,a) \geq Q_h^\st(x,a) - o(1)$ and $V_h\^t(x) + F_h\^t(x) \geq V_h^\st(x) - o(1)$. 
\end{lemma}
In light of \cref{lem:approx-optimism-informal}, we can bound the average suboptimality of the policies $\hat \pi\^t$ as follows:
\begin{align}
  \frac 1T \sum_{t=1}^T \left( \E[V_1^\st(x_1) - V_1^{\hat \pi\^t}(x_1)] \right) \leq & \frac 1T \sum_{t=1}^T \left( \E[V_1\^t(x_1) + F_1\^t(x_1) - V_1^{\hat \pi\^t}(x_1)] \right) - o(1) \nonumber\\
  =& \frac 1T \sum_{t=1}^T  \E^{\hat \pi\^t} \left[\sum_{h=1}^H F_h\^t(x_h) \right] - o(1) \label{eq:sum-th-fht},
\end{align}
where the inequality uses \cref{lem:approx-optimism-informal}, and the equality uses the definition of $V_h\^t$. To upper bound \cref{eq:sum-th-fht}, we need to show, roughly speaking, the following: for states $x$ which are likely to be visited by many of the policies $\hat\pi\^t$, the sum of exploration bonuses at $x$ is sublinear in $T$, i.e., $\frac 1T \sum_{t=1}^T F_h\^t(x) = o(T)$. We show a version of this statement ``in feature space'': in particular, we first upper bound $|F_h\^t(x)|$, for all states $x$, as follows.
\begin{lemma}[Informal version of \cref{lem:sigmap-bound,lem:quadratic-sim,lem:orig-truncated}]
  \label{lem:fht-sigma}
  For all $x,h,t$ in the execution of \PSDPUCB, it holds that \colt{$|F_h\^t(x)| \leq \poly(d,H)^{A}\cdot \max_{a \in \MA} \left( \phi_h(x,a)^\t \cdot (\Sigma_h\^t)^{-1} \cdot \phi_h(x,a) \right)^{1/2}$.}
 \arxiv{ \begin{align}
|F_h\^t(x)| \leq \poly(d,H)^{A}\cdot \max_{a \in \MA} \left( \phi_h(x,a)^\t \cdot (\Sigma_h\^t)^{-1} \cdot \phi_h(x,a) \right)^{1/2}\nonumber.
  \end{align}}
\end{lemma}
\cref{lem:fht-sigma} shows that for states $x$ for which all actions $a \in \MA$ point in directions of $\Sigma_h\^t$ corresponding to eigenvectors with large eigenvalues (i.e., ``well-explored'' directions), $|F_h\^t(x)|$ must be small. Using \cref{lem:fht-sigma}, the quantity in \cref{eq:sum-th-fht} may be bounded by $o(1)$ using a variant of the elliptic potential lemma (\cref{lem:epl-gen}). The elliptic potential lemma formalizes the notion that there are finitely many ``different directions'' in $\BR^d$, so if $T$ is sufficiently large, all directions $\BR^d$ frequently explored by policies $\hat \pi\^t$ must correspond to large-eigenvalue directions for $\Sigma_h\^t$.

Thus, it remains to discuss the proofs of  \cref{lem:approx-optimism-informal,lem:fht-sigma}. Notice that the statements of these two lemmas are in tension: \cref{lem:approx-optimism-informal} requires the bonuses $F_h\^t$ be sufficiently large so that $Q_h\^t$ is an upper bound on $Q_h^\st$, while \cref{lem:fht-sigma} requires that $F_h\^t$ be sufficiently small, as controlled by $\Sigma_h\^t$. As discussed in the following subsections, by carefully defining $F_h\^t$, we can ensure that both of these criteria are met.  

\subsection{Implementing local optimism via Bellman-linearity}
\label{sec:bellman-linearity-overview}
Essentially any approach to proving \cref{lem:approx-optimism-informal} requires that $Q_h\^t(x,a)$ be a linear function of $\phi_h(x,a)$.\footnote{In particular, linearity of $Q_h\^t$ establishes that the regression problem solved in \cref{line:define-w-hat} of \cref{alg:psdp-ucb} is well-specified, which allows us to show upper bounds on the prediction error $\left|\lng \hat w_h\^t, \phi_h(x,a) \rng- Q_h\^t(x,a)\right|$ and is in turn used to establish optimism.} %
In turn, as discussed in \cref{sec:optimism-overview}, this requires that the bonuses $F_{h+1}\^t(\cdot)$ have a Bellman backup which is linear in $\phi_{h}(x,a)$. 
It will be useful to formalize the following notion of \emph{Bellman-linearity}, which describes those functions whose Bellman backup is linear in the features: %
 \begin{definition}
  \label{def:bl}
Fix $h \in [H]$ with $h > 1$ and $k \in \BN$. We say that a function $G : \MX \ra \BR^k$ is \emph{Bellman-linear at step $h$} if there exists $W_G \in \BR^{k \times d}$ so that, for all $(x,a) \in \MX \times \MA$, it holds that $W_G \cdot\phi_{h-1}(x,a) = \E_{x' \sim P_{h-1}(x,a)}[G(x')]$. In particular, in the case that $k=1$, $G : \MX \ra \BR$ is Bellman-linear at step $h$ if there is $w_G \in \BR^d$ so that $\lng w_G, \phi_{h-1}(x,a) \rng = \E_{x' \sim P_{h-1}(x,a)}[G(x')]$. 
\end{definition}
If the value of $h$ is clear, we will simply say the function $G$ is Bellman-linear. The canonical example of a Bellman-linear function is the mapping $x \mapsto \max_a \lng w, \phi_h(x,a) \rng$, for any $w \in \BR^d$. Given that such functions are Bellman-linear, we can construct additional Bellman-linear functions by averaging, taking limits (\cref{lem:bl-limit}), and differentiating (\cref{lem:diff-bl}). The following lemma shows that Bellman-linearity of the bonuses $F_h\^t$ is indeed a sufficient condition for linearity of $Q_h\^t$:
\begin{lemma}[Variant of \cref{lem:wht-exist-bound}]
  \label{lem:wht-exist-bound-overview}
Suppose that, in \PSDPUCB, for each $h \in [H], t \in [T]$, $F_h\^t : \MX \ra \BR$ is Bellman-linear at step $h$. Then for all $h,t$, there is some $w_h\^t \in \BR^d$ so that $Q_h\^t(x,a) = \lng \phi_h(x,a), w_h\^t\rng$ for all $(x,a) \in \MX \times \MA$. %
\end{lemma}
The proof of \cref{lem:wht-exist-bound-overview} requires some work to show that the expectation of the bonuses $F_g\^t$ for $g > h+1$ is linear at step $h$: this fact crucially relies on the fact that the policies $\hat \pi\^t$ are linear policies (see \cref{lem:lin-lb}). 

\paragraph{A Bellman-linear approximation of $\Bquad_h$.}%
In light of our discussion pertaining to the quadratic bonus $\Bquad_h$ around \cref{eq:quad-bonus} above, %
a natural strategy is to show that $\max_a \Bquad_h(\cdot, a; \Sigma)$ is Bellman-linear at step $h$ for any $\Sigma$. 
Though this is not the case (\cref{prop:quadratic-ctex}), \cref{lem:quadratic-sim-informal} below suggests a workaround: the quadratic bonus $\Bquad_{h}(x,a; \Sigma)$ can be \emph{approximated} by a Bellman-linear function. %
\begin{lemma}[Simplified version of \cref{lem:quadratic-sim}]
  \label{lem:quadratic-sim-informal}
  Fix $h \in [H]$ and a PSD matrix $\Sigma$, and define
  \begin{align}
    \label{eq:define-fnormal}
    \Fnormal_{h}(x; \Sigma) := \E_{w \sim \MN(0, \Sigma)}\left[ \max_{a \in \MA} \lng w, \phi_h(x,a) \rng \right].
  \end{align}
  Then $\Fnormal_{h}(\cdot; \Sigma)$ is Bellman-linear at step $h$, and for any $x \in \MX$,
 {\small \begin{align}
\frac{1}{\sqrt{2\pi}} \max_{a,a' \in \MA} \left( (\phi_h(x,a) - \phi_h(x,a'))^\t  \Sigma  (\phi_h(x,a) - \phi_h(x,a')) \right)^{1/2} \leq \Fnormal_{h}(x; \Sigma) \leq \sqrt{d}  \max_{a \in \MA} \left( \phi_h(x,a)^\t  \Sigma  \phi_h(x,a) \right)^{1/2}\nonumber.
  \end{align}}
\end{lemma}
The proof of \cref{lem:quadratic-sim-informal} is straightforward: that $x \mapsto \Fnormal_h(x; \Sigma)$ is Bellman-linear at step $h$ follows from the fact that $\Fnormal_h(x; \Sigma)$ is defined as an average of Bellman-linear functions. The upper and lower bounds on $\Fnormal_h(x; \Sigma)$ follow from a direct computation.

\cref{lem:quadratic-sim-informal} suggests the following approach:
for some scalar $\beta_h$, suppose we define
\begin{align}
  \label{eq:fht-normal}
  F_h\^t(x) := \beta_h \cdot \Fnormal_h(x; (\Sigma_h\^t)^{-1}).
\end{align}
Since $\Fnormal_h(\cdot ;\Sigma)$ is Bellman-linear and approximates the quadratic bonus (\cref{lem:quadratic-sim-informal}), we might hope that it acts similarly to the quadratic bonus in \LSVIUCB to induce an optimistic property of the value functions $Q_h\^t(x,a)$ in \PSDPUCB, thereby allowing us to establish \cref{lem:approx-optimism-informal}. 
Unfortunately, this argument encounters a few snags:
\begin{enumerate}[wide,labelwidth=!,labelindent=0pt]
\item\label{it:bquad-diff} The lower bound on $\Fnormal_{h}(x;\Sigma)$ in \cref{lem:quadratic-sim-informal} is not exactly a constant times  $\Bquad_{h}(x; \Sigma)$, but instead involves a \emph{difference} between feature vectors at the state $x$. Thus, strictly speaking, $\Fnormal_{h}(x; \Sigma)$ does not give a multiplicative approximation of the quadratic bonus $\Bquad_{h}(x;\Sigma)$.
\item\label{it:bquad-mult} Even if $\Fnormal_{h}(x; \Sigma)$ did give a multiplicative approximation of $\Bquad_{h}(x; \Sigma)$, we would need to ensure that this multiplicative approximation is good enough for all arguments to go through. %
\item\label{it:bquad-min} %
Ignoring the previous two items,  we still suffer from the following exponential blowup. Roughly speaking, the proof of \cref{lem:approx-optimism-informal} proceeds by showing that the loss in value from any sub-optimal action taken by $\hat \pi\^t$ at step $h$ is canceled out by the bonus $F_h\^t(x_h)$. To show this, we need to show that the exploration bonus $\beta_h \cdot \Fnormal_h(x_h; (\Sigma_h\^t)^{-1})$ is at least as large as the regression error in our estimate of $\hat w_h\^t$ from \cref{line:define-w-hat}. %
  But the labels in this linear regression involve the sum of bonuses $F_g\^t(\cdot)$ for $h+1 \leq g \leq H$, and thus scale with $\beta_{h+1} + \cdots + \beta_H$. Working through standard high-confidence bounds for linear regression (see \cref{lem:phi-what-wt-diff} for details) establishes that we would need $\beta_h \geq \poly(d) \cdot (\beta_{h+1} + \cdots + \beta_H)$. 
  This recursion leads to $\beta_1 \geq d^{O(H)}$, which would require exponentially many samples to cancel out exponentially large error terms.
\end{enumerate}
\cref{it:bquad-diff,it:bquad-mult} above turn out not to be major issues\colt{, as discussed in the proof of \cref{lem:approx-optimism}}. \arxiv{In the case of \cref{it:bquad-diff}, a lower bound on $F_h(x; \Sigma)$ growing only with the difference between feature vectors still induces a sufficient exploration property on the induced value functions: see the proof of  \cref{lem:approx-optimism}.  \cref{it:bquad-mult} may be handled by increasing the scaling factors $\beta_h$ by a factor proportional to the gap between the upper and lower bounds of \cref{lem:quadratic-sim-informal}.} %

The main technical challenge is dealing with \cref{it:bquad-min} above, namely the exponential blowup of $\beta_h$.   Typically, the solution to the above dilemma is to \emph{truncate} the bonus, i.e., define $F_h\^t(x) := \min \{ \beta \cdot \Fnormal_h(x; (\Sigma_h\^t)^{-1}), \bar F_h(x) \}$, for some $\bar F_h(x) \ll \beta$ which is a uniform upper bound on $V_h^\st(x)$. (Note that, when using such an approach, the parameter $\beta$ does not need any dependence on $h$.) Indeed, this truncation is used in \LSVIUCB, where the $\hat Q_h\^t$ function in \cref{eq:lsvi-q} is truncated by the constant $\bar F_h(x) \equiv H$. Unfortunately, in the setting of linear Bellman completeness, truncation, in general, breaks Bellman-linearity of $F_h\^t(x)$. 
 In the following section, we discuss how to implement this truncation in a way that preserves Bellman-linearity. %

\subsection{Constructing truncated bonuses}
\label{sec:gena-overview}

\paragraph{Approximating $(\Sigma_h\^t)^{-1}$.} For each $h \in [H]$ and $t \in [T]$, the inverse covariance matrix $(\Sigma_h\^t)^{-1}$ captures which directions of feature space are explored by policies $\hat \pi\^s \circ_h \tilde \pi\^s$, for $s \leq t-1$. %
We split $(\Sigma_h\^t)^{-1}$ into matrices representing ``well-explored directions'' and ``unexplored directions'', as follows: for an appropriate threshold value $\sigma$, we let $\Sigma'$ denote the projection matrix onto the subspace spanned by eigenvectors of $(\Sigma_h\^t)^{-1/2}$ with eigenvalues greater than $\sigma$, and $\Lambda'$ denote the projection matrix onto the subspace spanned by eigenvectors of $(\Sigma_h\^t)^{-1/2}$ with eigenvalues less than $\sigma$. Formally, as defined in \cref{def:mat-truncation}, we write $(\Sigma', \Lambda') = \trunc{(\Sigma_h\^t)^{-1/2}}{\sigma}$. We denote subspaces onto which $\Sigma', \Lambda'$ project by $\MS_{\Sigma'}, \MS_{\Lambda'}$, respectively. Note that $\Sigma' \Lambda' = \Lambda' \Sigma' = 0$ and $\Sigma' + \Lambda' = I_d$; we call the pair $(\Sigma', \Lambda')$ an \emph{orthogonal pair} (\cref{def:op}). 

If $\sigma$ is chosen sufficiently small to be less than our desired error threshold, then the estimations $\hat w_h\^t$ produced in \cref{line:define-w-hat} of \cref{alg:psdp-ucb} are sufficiently accurate in directions spanned by $\MS_{\Lambda'}$. Thus, we only have to add an exploration bonus at a state $x$ which scales in proportion to the width of the features $\phi_h(x,a),\ a \in \MA$ in directions spanned by the orthogonal complement of $\MS_{\Lambda'}$, namely $\MS_{\Sigma'}$. This leads us to aim to construct an exploration bonus which scales roughly as $x \mapsto \min \{ \beta \cdot \Fnormal_h(x; \Sigma'), \bar F_h(x)\}$, for some parameter $\beta > 1$ and an appropriate choice of $\bar F_h$ which should (a) be uniformly upper bounded by a parameter which does not scale with $\beta$, and (b) be an upper bound on $V_h^\st$. It  turns out that condition (b) can be relaxed to require only that $\bar F_h(x)$ be an upper bound on
$
  \max_{a \in \MA} Q_h^\st(x,a) - \min_{a \in \MA} Q_h^\st(x,a) = \max_{a,a' \in \MA} \lng w_h^\st, \phi_h(x,a) - \phi_h(x,a') \rng.
  $
  Thus, to summarize, we aim to find a bonus function  which satisfies the following properties:
  \begin{problem}[Bonus function]
    \label{prob:bonfun}
Fix $h, t$ (which determine $(\Sigma', \Lambda') = \trunc{(\Sigma_h\^t)^{-1/2}}{\sigma}$ as above), and some $\beta > 1$.    Can we find a Bellman-linear function $F_h\^t : \MX \to \BR_{\geq 0}$ which satisfies:
    \begin{enumerate}[wide,labelwidth=!,labelindent=0pt]
    \item \label{it:bonfun-1} For all $x\in \MX$, $F_h\^t(x)$ is lower bounded as
      \begin{align}
F_h\^t(x) \geq & BH \cdot \max_{a,a' \in \MA} \min \left\{ \beta \cdot  \| \Sigma' \cdot (\phi_h(x,a) - \phi_h(x,a')) \|_2,  \|\Lambda' \cdot ( \phi_h(x,a) - \phi_h(x,a')) \|_2 \right\}\label{eq:ftp-lb-intro}.
      \end{align}
    \item \label{it:bonfun-2} For all $x \in \MX$, $F_h\^t(x)$ is upper bounded as \colt{$F_h\^t(x) \leq  \poly(d,H,B) \cdot \max_{a \in \MA} \min \left\{ \beta \| \Sigma' \cdot \phi_h(x,a) \|_2,  \| \phi_h(x,a) \|_2 \right\}$.}
     \arxiv{ \begin{align}
F_h\^t(x) \leq & \poly(d,H,B) \cdot \max_{a \in \MA} \min \left\{ \beta \| \Sigma' \cdot \phi_h(x,a) \|_2,  \| \phi_h(x,a) \|_2 \right\}\nonumber.
      \end{align}}
    \end{enumerate}
  \end{problem}

\paragraph{The case $d=A = 2$.} To proceed, we consider the special case in which $d=2$. Let us suppose that $\MS_{\Sigma'}, \MS_{\Lambda'}$ are each one dimensional, and are spanned by unit vectors $u, v \in \BR^2$, respectively, which are orthogonal. Let us consider the function
\begin{align}
\Ftp_h(x; u,v) :=  \max_{a \in \MA} \lng u, \phi_h(x,a) \rng + \max_{a \in \MA} \lng v, \phi_h(x,a) \rng - \max_{a \in \MA} \lng u+v, \phi_h(x,a)\rng\nonumber,
\end{align}
for some parameter $\beta > 1$. 
If we suppose that there are only 2 actions (i.e., $|\MA| = 2$), then $x \mapsto \sum_{\ep_u, \ep_v \in \{ \pm 1 \}} \Ftp_h(x; \beta \ep_u \cdot u,\ep_v \cdot v)$ (scaled by a $\poly(B, H)$ factor) satisfies the conditions of \cref{prob:bonfun}. To see this, let us write $\MA = \{0,1\}$, and set $\phi_h^\Delta(x) := \phi_h(x,1) - \phi_h(x,0)$. %
Fix $x \in \MX$, and choose $\ep_u, \ep_v \in \{ \pm 1 \}$ so that $\lng \ep_u \cdot u, \phi_h^\Delta(x) \rng $ and $\lng \ep_v \cdot v, \phi_h^\Delta(x) \rng $ have opposite signs.  Then
\begin{align}
  \Ftp_h(x;\beta \ep_u u,\ep_v v) = &  | \lng \beta \ep_u u, \phi_h^\Delta(x) \rng | + | \lng \ep_v v, \phi_h^\Delta(x) \rng | - | \lng \beta \ep_u u+\ep_v v, \phi_h^\Delta(x) \rng |\nonumber\\
  \geq&  \min \{ | \lng \beta u, \phi_h^\Delta(x) \rng |, | \lng v, \phi_h^\Delta(x) \rng | \},\label{eq:ftp-lb-intro}
\end{align}
where the inequality uses the property that for real numbers $a,b$ with opposite signs, we have $|a| + |b| - |a+b| \geq \min \{ |a|, |b| \}$. 
By  combining \cref{eq:ftp-lb-intro} with non-negativity of $\Ftp_h(\cdot)$, we obtain the first property of \cref{prob:bonfun}. The second property may be verified in a similar manner (see \cref{lem:tp-upper-bound}).

\arxiv{\colt{Recall that in \cref{sec:gena-overview} we gave an overview of the construction of truncated confidence bonuses (namely, those satisfying the conditions of \cref{prob:bonfun}) in the special case that $d = A = 2$. In this section, we extend this argument to the case of general values of $d,A$. }

\paragraph{The case $d=2$ and general $A$.} Let us try to extend the above argument \colt{(i.e., in \cref{sec:gena-overview})} for the case $d=2$ and $A=2$ to the case of general $A$ (still keeping $d=2$ fixed). Unfortunately, the function $x \mapsto \sum_{\ep_u, \ep_v \in \{ \pm 1 \}} \Ftp_h(x; \beta \ep_u \cdot u,\ep_v \cdot v)$ we used above no longer satisfies the first condition of \cref{prob:bonfun}. To see this, suppose that $\MA = \{0,1,2,3\}$ and the feature vectors at some state $x$ form a rectangle, i.e., $\phi_h(x,0) = (0,0), \phi_h(x,1) = (u,0), \phi_h(x,2) = (u+v,u+v), \phi_h(x,3) = (0,v)$. It is straightforward to see that $\Ftp_h(x; u,v) = 0$, whereas $\max_{a,a' \in \MA} \min \{ |\lng \beta u,  \phi_h(x,a) - \phi_h(x,a') \rng|, | \lng v, \phi_h(x,a) - \phi_h(x,a') \rng |\} \geq 1$ for any $\beta \geq 1$ (e.g., take $a = 0, a' = 2$). Thus, \cref{eq:ftp-lb-intro} cannot hold, even if we only wish for it to hold up to a $\poly(H, B, d)$ factor.

The above issue runs fairly deep: it can be shown that, under some mild conditions on $\Ftp_h$ which ensure that it be  Bellman-linear,\footnote{Namely, that $\Ftp_h(x; u,v) = \int_{w \in \BR^d} f(w) \cdot \max_{a \in \MA} \lng \phi_h(x,a), w \rng d\mu(w)$, for some Borel measure $\mu$ and function $f : \BR^d \to \BR$.} there is \emph{no function} $F_h\^t$ satisfying both conditions of \cref{prob:bonfun}. To rectify this issue, we will relax the first condition of \cref{prob:bonfun}, as follows:
\begin{problem}
  \label{prob:bonfun-b}
  In the setting of \cref{prob:bonfun}, can we find a Bellman-linear function $F_h\^t : \MX \to \BR_{\geq 0}$ which satisfies the constraints of \cref{prob:bonfun} where \cref{it:bonfun-1} is replaced by the following constraint:
  \begin{enumerate}
  \item\label{it:bonfun-1b} For all $x \in \MX$, $F_h\^t(x)$ is lower bounded as
    \begin{align}
F_h\^t(x) \geq & BH \cdot \max_{a,a' \in \MA} \min_{\psi \in \bar \Phi(x) } \left\{ \| \beta \Sigma' \cdot (\phi_h(x,a) - \psi) \|_2 + \| \Lambda' \cdot (\psi - \phi_h(x,a') \|_2 \right\}\label{eq:ftp-lb-weak},
    \end{align}
    where $\bar \Phi_h(x) := \co \{ \phi_h(x,a):\ a \in \MA \}$ denotes the convex hull of the feature vectors $\phi_h(x,a),\ a \in \MA$. 
  \end{enumerate}
\end{problem}

Note that \cref{eq:ftp-lb-weak} is weaker than \cref{eq:ftp-lb-intro}, as can be verified by considering the choices $\psi = \phi_h(x,a)$ and $\psi = \phi_h(x,a')$ for given $a,a'$ in the outer maximum. As we discuss in \cref{sec:optimism-proof}, this weaker lower bound on $F_h\^t(x)$ is still sufficient to ensure optimism. Moreover, for the example discussed above in which the feature vectors at a state $x$ form a rectangle, the right-hand side of \cref{eq:ftp-lb-weak} is $0$: if $a,a'$ in the outer maximum are chosen to represent opposing corners of the rectangle, then taking $\psi$ to be one of the other two corners yields $| \lng \beta u, \phi_h(x,a) - \psi \rng | = | \lng v, \psi - \phi_h(x,a') \rng | = 0$. For this example, we therefore have that $x \mapsto \sum_{\ep_u, \ep_v \in \{ \pm 1 \}} \Ftp_h(x; \beta \ep_u \cdot u,\ep_v \cdot v)$ satisfies the properties of \cref{prob:bonfun-b}.

In fact, this holds more generally in the $d=2$ case, as can be seen by the following argument: fix any $a,a' \in \MA$, and choose
\begin{align}
\psi^\st := \argmin_{\psi \in \bar \Phi_h(x)}\left\{ | \lng \beta u, \phi_h(x,a) - \psi \rng | + | \lng v, \psi - \phi_h(x,a') \rng | \right\}\label{eq:xistar-opt}.
\end{align}
Moreover choose $\ep_u, \ep_v \in \{ \pm 1 \}$ so that $\lng \phi_h(x,a) - \psi^\st, \ep_u u \rng \geq 0$ and $\lng \phi_h(x,a') - \psi^\st, \ep_v v \rng \geq 0$. Note that \cref{eq:xistar-opt} is a convex optimization problem; its KKT conditions tell us that $\lng \psi^\st - \psi, \beta \ep_u u + \ep_v v \rng \geq 0$ for all $\psi \in \bar \Phi_h(x)$. Then, recalling that $a,a' \in \MA$ were fixed,  we have
\begin{align}
  \Ftp_h(x; \beta \ep_u u, \ep_v v) \geq & \lng \beta \ep_u u, \phi_h(x,a) \rng + \lng \ep_v v, \phi_h(x,a') \rng - \max_{\bar a \in \MA} \lng \beta \ep_u u+\ep_v v, \phi_h(x, \bar a) \rng \nonumber\\
  \geq & \lng \beta \ep_u u, \phi_h(x,a) \rng + \lng \ep_v v, \phi_h(x,a') \rng - \lng \beta \ep_u u + \ep_v v, \psi^\st \rng\nonumber\\
  =& \lng \beta \ep_u u, \phi_h(x,a) - \psi^\st \rng + \lng \ep_v v, \phi_h(x,a') - \psi^\st \rng\nonumber\\
  =& | \lng \beta u, \phi_h(x,a) - \psi^\st \rng | + | \lng v, \phi_h(x,a') - \psi^\st \rng |\label{eq:verify-bonfun-intro},
\end{align}
where the second inequality uses the KKT conditions of \cref{eq:xistar-opt} and the final equality uses the choices of $\ep_u, \ep_v$. \cref{eq:verify-bonfun-intro}, combined with non-negativity of $\Ftp_h(\cdot)$, verifies the first property of \cref{prob:bonfun-b}, as desired; the second property (namely, \cref{it:bonfun-2} of \cref{prob:bonfun}) can be verified by a direct computation (\cref{lem:tp-upper-bound}).

\paragraph{The case of general $d$ and general $A$.} The argument presented above for $d=2$ generalizes to the case of arbitrary $d \in \BN$. In particular, we will aim to construct a bonus of the form $G_h\^t(x) :=   O(BH) \cdot \E_{u \sim \MD_{\Sigma'}, v \sim \MD_{\Lambda'}} [\Ftp_h(x; u,v)]$, where $\MD_{\Sigma'}, \MD_{\Lambda'}$ are distributions supported on $\MS_{\Sigma'}, \MS_{\Lambda'}$, respectively. 
A key consideration is: \emph{how do we choose the distributions $\MD_{\Sigma'}, \MD_{\Lambda'}$?} To answer this question, it is instructive to consider how the above argument verifying \cref{it:bonfun-1b} of \cref{prob:bonfun-b} generalizes to higher dimensions. In particular, given arbitrary $a,a' \in \MA$, in lieu of \cref{eq:xistar-opt}, we define
\begin{align}
\psi^\st := \argmin_{\psi \in \cPhi_h(x)} \left\{ \| \beta \Sigma' \cdot (\phi_h(x,a) - \psi) \|_2 + \| \Lambda' \cdot (\psi - \phi_h(x,a')) \|_2 \right\}\label{eq:xistar-opt-2}.
\end{align}
Using the KKT conditions of \cref{eq:xistar-opt-2}, it can be verified that, for $u = \frac{\Sigma' \cdot (\phi_h(x,a) - \psi^\st)}{\| \Sigma' \cdot (\phi_h(x,a) - \psi^\st)\|_2}$ and $v = \frac{\Lambda' \cdot (\phi_h(x,a') - \psi^\st)}{\| \Lambda' \cdot (\phi_h(x,a') - \psi^\st) \|_2}$, we have that
\begin{align}
  \label{eq:ftp-gend-lb}
  \Ftp_h(x; \beta u, v) \geq \| \beta \Sigma' \cdot (\phi_h(x,a) - \psi^\st) \|_2 + \| \Lambda' \cdot (\phi_h(x,a') - \psi^\st ) \|_2.
\end{align}
Unfortunately these choices of $u,v$ depend on the state $x$, and $\MD_{\Sigma'}, \MD_{\Lambda'}$ must be independent of $x$ in order to ensure Bellman-linearity of the bonus function.

Fortunately, as shown in \cref{lem:optimal-perimeter}, under some additional technical conditions, \cref{eq:ftp-gend-lb} holds up to an additive term of $O(\vep)$ when the left-hand side is replaced by $\Ftp_h(x; \beta u', v')$, for any vectors $u', v'$ which are $\vep$-close in $\ell_2$ norm to $u, v$, respectively. If we take $\MD_{\Sigma'}, \MD_{\Lambda'}$ to be the uniform distributions on the unit sphere in the subspaces $\MS_{\Sigma'}, \MS_{\Lambda'}$, respectively, then $u' \sim \MD_{\Sigma'}, v' \sim \MD_{\Lambda'}$ satisfy $\| u' - u \|_2 \leq \vep, \| v'-v\|_2 \leq \ep$ with probability $\vep^{O(d)}$. With these choices of $\MD_{\Sigma'}, \MD_{\Lambda'}$, the bonus $G_h\^t$ satisfies the constraint \cref{eq:ftp-lb-weak} up to a factor of $\vep^{-O(d)}$. In fact, this can easily be improved to a factor of $\vep^{-O(A)}$ by using the fact that closeness of $u,u'$ and of $v,v'$ only needs to hold in the subspace spanned by $\phi_h(x,a)$ (for $a \in \MA$), which has at most $A$ dimensions -- see \cref{cor:optimal-perimeter}. 

It turns out to simplify the analysis to take $\MD_{\Sigma'}, \MD_{\Lambda'}$ to be the distributions $\MN(0, \Sigma'), \MN(0, \Lambda')$, respectively, which permit essentially the same analysis as described above. %
Moreover, to satisfy certain technical conditions in \cref{lem:optimal-perimeter}, we add a term to the bonus function which is proportional to $\Fnormal_h(x; \Sigma') = \E_{w \sim \MN(0, \Sigma')} [ \max_{a \in \MA} \lng w, \phi_h(x,a) \rng ]$. To sum up, we define the bonus function as
\begin{align}
F_h\^t(x) := & \beta_1 \cdot \E_{u' \sim \MN(0, \Sigma')} \E_{v' \sim \MN(0, \Lambda')} \left[ \Ftp_h(x; \beta u', v') \right] + \beta_2 \cdot \E_{w \sim \MN(0, \Sigma')} \left[ \max_{a \in \MA} \lng w, \phi_h(x,a) \rng \right]\label{eq:fht-final-intro},
\end{align}
where $\beta_1, \beta_2, \beta$ are carefully chosen parameters bounded above by $(BHD/\ep)^{O(A)}$, where $\ep$ is the desired error level of the algorithm's output policy. The precise values of these parameters appear in the formal definition of $F_h\^t$ in \cref{eq:define-fht-bonus}. We remark that, crucially, $\beta_2 \ll \beta_1 \cdot \beta$ -- the entire point of the truncation procedure we have just described is to ensure that the size of $F_h\^t(x)$ (which may be as large as $O(\beta_2)$) does \emph{not} scale with $\beta$.

In \cref{sec:tp-upper-bound,sec:tp-lower-bound}, we formally verify, per the arguments outlined above, that the bonus $F_h\^t$ in \cref{eq:fht-final-intro} satisfies the properties laid out in \cref{prob:bonfun-b}. Then, in \cref{sec:bonus-boundedness,sec:linreg,sec:optimism-proof}, we show how this implies the statements of \cref{lem:approx-optimism-informal,lem:fht-sigma}, which suffices to prove our main theorem (\cref{thm:online-intro}). 

}
\colt{
\paragraph{The case of  general $d$ and $A$.} A direct extension of the above argument for $d = A = 2$ runs into some snags even in the case that $d = 2$ and $A=4$. We discuss these issues and additional techniques we introduce to overcome them in \cref{sec:additional-to}. 
}

\colt{
  \section*{Acknowledgements}
  We thank Dylan Foster and Sham Kakade for bringing this problem to our attention and for helpful discussions. 
NG is supported by a Fannie \& John Hertz Foundation Fellowship and an NSF Graduate Fellowship. AM is supported in part by a Microsoft Trustworthy AI Grant, an ONR grant and a David and Lucile Packard Fellowship.
}

\colt{\bibliography{lbc}}

\colt{\appendix}
\colt{\newpage}
\colt{}

\colt{
  \section{Additional technical overview}
  \label{sec:additional-to}
  
  }

\section{Bellman linear functions}
\label{sec:bellman-linear}
Recall the definition of Bellman linearity in \cref{def:bl}. In this section, we discuss several properties of Bellman linear functions. 
 \begin{lemma}[Limiting preserves Bellman-linearity]
   \label{lem:bl-limit}
Fix $h \in [H]$. Suppose that for each $i \in \BN$, we are given a function $G_i : \MX \ra \BR$ which is Bellman-linear at step $h$ and so that $\sup_{x \in \MX, i \in \BN}| G_i (x)| \leq C$, for some $C > 0$. Furthermore suppose that, for some $G : \MX \ra \BR$, $\lim_{i \ra \infty} G_i(x) = G(x)$ for all $x \in \MX$. Then $G$ is Bellman-linear at step $h$.
\end{lemma}
\begin{proof}
  For each $i \in \BN$, we know that there is some $w_i \in \BR^d$ so that, for all $(x,a) \in \MX \times \MA$, $\lng \phi_{h-1}(x,a), w_i \rng = \E_{x' \sim P_{h-1}(x,a)}[G_i(x')]$. For each $i,a,x$, we know that $|\E_{x' \sim P_{h-1}(x,a)}[G_i(x')]| \leq C$ by uniform boundedness of $G_i$, so $|\lng \phi_{h-1}(x,a), w_i \rng |\leq C$ for all  $i,a,x$. By definition of $\MB_{h-1}$, it follows that $w_i \in C \cdot \MB_{h-1}$, which gives $\| w_i \|_2 \leq CB$ by \cref{asm:boundedness}. %
  Hence we may find a subsequence $(w_{i_j})_{j \in \BN}$ so that $\lim_{j \ra \infty} w_{i_j} = w^\st$, for some $w^\st \in C \cdot \MB_{h-1}$. Thus, for any $(x,a) \in \MX \times \MA$, we have
  \begin{align}
    & \lng \phi_{h-1}(x,a), w^\st \rng = \lim_{j \ra \infty} \lng \phi_{h-1}(x, a), w_{i_j} \rng = \lim_{j \ra \infty} \E_{x' \sim P_{h-1}(x,a)}[G_{i_j}(x')]\nonumber\\
    =&  \E_{x' \sim P_{h-1}(x,a)} \left[ \lim_{j \ra \infty} G_{i_j}(x') \right] = \E_{x' \sim P_{h-1}(x,a)}[G(x')]\nonumber,
  \end{align}
  where the second-to-last equality uses the dominated convergence theorem and uniform boundedness of the functions $G_i$. 
\end{proof}

\begin{lemma}[Differentiation preserves Bellman-linearity]
  \label{lem:diff-bl}
  Fix any $h, k \in \BN$ and consider an open set $U \in \BR^k$ as well as a function $F : U \times \MX \ra \BR$ so that $F(u,x)$ is continuously differentiable in $u$ for all $x \in \MX$, and that $\max_{x \in \MX, u \in U} \max\{ F(u,x) , \|\grad_u F(u,x)\|_2\} \leq C$. Suppose that, for all $u \in U$, the function $F(u, \cdot)$ is Bellman-linear at step $h$. Then the function $\grad_u F(u, \cdot)$ is Bellman-linear at step $h$, for all $u \in U$.  %
\end{lemma}
\begin{proof}
By Bellman-linearity, we know that for all $u \in U$, there is some $w_u \in \BR^d$ so that, for all $(x,a) \in \MX \times \MA$, it holds that
   \begin{align}
 \lng \phi_h(x,a), w_u \rng = \E_{x' \sim P_{h-1}(x,a)} \left[ F(u, x') \right]\nonumber.
   \end{align}
  
   By the dominated convergence theorem (which may be applied as a consequence of uniform boundedness of $\| \grad_u F(u,x) \|_2$ over $u \in U$), it holds that $u \mapsto \lng \phi_h(x,a), w_u \rng$ is continuously differentiable for all $x,a$, and moreover,
   \begin{align}
\grad_u \lng \phi_h(x,a), w_u \rng = \E_{x' \sim P_{h-1}(x,a)}[\grad_u F(u, x')]\label{eq:gradF-diff}.
   \end{align}

   Choose $(x_1, a_1), \ldots, (x_\ell, a_\ell) \in \MX\times \MA$ so that $\{ \phi_h(x_i, a_i) \}_{i=1}^\ell$ is linearly independent and spans $\{ \phi_h(x,a) \}_{(x,a) \in \MX \times \MA}$; then $\ell \leq d$. Let $W_u \in \BR^{d \times d}$ be any matrix satisfying
   \begin{align}
W_u \cdot \phi_h(x_i, a_i) = \grad_u \lng \phi_h(x_i, a_i), w_u \rng \quad \forall i \in [\ell]\nonumber,
   \end{align}
   which is possible by linear independence of $\phi_h(x_i, a_i),\ i \in [\ell]$. Now consider any $(x,a) \in \MX \times \MA$; we can find $\alpha_1, \ldots, \alpha_\ell \in \BR$ so that $\phi_h(x,a) = \sum_{i=1}^\ell \alpha_i \cdot \phi_h(x_i, a_i)$. Then
   \begin{align}
W_u \cdot \phi_h(x,a) = \sum_{i=1}^\ell \alpha_i \cdot \grad_u \lng \phi_h(x_i, a_i), w_u \rng = \grad_u \sum_{i=1}^\ell \alpha_i \cdot \lng \phi_h(x_i, a_i), w_u \rng = \grad_u \lng \phi_h(x,a), w_u \rng\nonumber,
   \end{align}
   which completes the proof by \cref{eq:gradF-diff}.
 \end{proof}

\cref{lem:lin-lb} below states that the feature vectors induced by any linear policy are Bellman-linear. It is a special case of Corollary 3.2 of \cite{golowich2024role} (by taking the inherent Bellman error to be 0). 
\begin{lemma}
  \label{lem:lin-lb}
  Suppose that $M$ is linear Bellman complete. %
  Then for each $h \in [H]$ and each $w \in \BR^d$, the mapping $x' \mapsto \phi_{h+1}(x', \pi_{h+1,w}(x'))$ is Bellman-linear at step $h+1$, i.e.,  there is a linear map $L_h(w) : \BR^d \ra \BR^d$ so that, for all $(x,a) \in \MX \times \MA$,
  \begin{align}
L_h(w)^\t \cdot \phi_h(x,a) = \E_{x' \sim P_h(x,a)}[\phi_{h+1}(x', \pi_{h+1,w}(x'))]\nonumber.
  \end{align}
\end{lemma}
\begin{proof}
  Fix $h \in [H]$.  Define $V(x', \theta) := \max_{a' \in \MA} \lng \phi_{h+1}(x', a'), \theta \rng$ and $Q(x,a,\theta) = \lng \phi_h(x,a), \MT_h \theta\rng$, so that $Q(x,a,\theta) = \lng \phi_h(x,a), \MT_h\theta \rng = \E_{x' \sim P_h(x,a)}[V(x',\theta)]$. By Fubini's theorem, for any $\sigma > 0$, we have
  \begin{align}
\Sm Q(x,a,\theta) = \E_{z \sim \MN(0, \sigma^2 I_d)} [Q(x,a,\theta - z)] = \E_{x' \sim P_h(x,a)} \E_{z \sim \MN(0, \sigma^2 I_d)}[V(x', \theta - x)] = \E_{x' \sim P_h(x,a)} [\Sm V(x', \theta)]\nonumber.
  \end{align}
  Linear Bellman completeness of $M$ gives that for each $\theta \in \BR^d$, the function $x' \mapsto V(x', \theta)$, and thus $x' \mapsto \Sm V(x', \theta)$ is Bellman-linear at step $h+1$. The function $\Sm V(x', \theta)$ is infinitely differentiable as a function of $\theta$, for all $x'$; moreover, since
  \begin{align}
    \grad \Sm V(x', \theta) = \E_{z \sim \MN(0, \sigma^2 I_d)}[\grad V(x', \theta - z)] = \E_{z \sim \MN(0, \sigma^2 I_d)} [\phi_{h+1}(x', \pi_{h+1,\theta - z}(x'))]\label{eq:grad-smooth-compute},
  \end{align}
  we have that $\max\{ \Sm V(x', \theta), \| \grad \Sm V(x', \theta) \|_2 \}$ is bounded uniformly over $x'$ in the neighborhood of any $\theta$. We remark that the first equality in \cref{eq:grad-smooth-compute} uses the dominated convergence theorem, and the second equality uses that the $V(x', \theta-z)$ is differentiable almost surely over $z \sim \MN(0, \sigma^2 I_d)$. Thus, by \cref{lem:diff-bl}, the function $x' \mapsto \grad \Sm V(x', \theta)$ is Bellman-linear at step $h$, for all $\theta \in \BR^d$. 
  But for $x' \in \MX$ and $w \in \BR^d$, 
  \begin{align}
    \grad \Sm V(x', w) = \phi_{h+1}(x', \pi_{h+1,w,\sigma}(x'))\nonumber,
  \end{align}
  where the equality above uses \cref{eq:grad-smooth-compute} together with \cref{def:plinear}. Next,  \cref{eq:feature-limit} ensures that for each $x' \in \MX$, $\lim_{\sigma \ra 0^+} \phi_{h+1}(x', \pi_{h+1,w,\sigma}(x') = \phi_{h+1}(x', \pi_{h+1,w}(x'))$. Then \cref{lem:bl-limit} (applied to each coordinate of the function $x' \mapsto \phi_{h+1}(x', \pi_{h+1,w,\sigma}(x')$, for $\sigma > 0$) yields that $x' \mapsto \phi_{h+1}(x', \pi_{h+1,w}(x'))$ is Bellman-linear at step $h+1$, for any $w \in \BR^d$.
\end{proof}
Given any $\pi_{h+1} \in \Pilin_{h+1}$ (so that $\pi_{h+1} = \pi_{h+1,v}$ for some $v \in \BR^d$) and $w \in \BR^d$, we define $\MT_h^\pi w := L_h(v) \cdot w$. Then by \cref{lem:lin-lb}, for all $(x,a) \in \MX \times \MA$,
\begin{align}
\lng \phi_h(x,a), \MT_h^\pi w \rng = \E_{x' \sim P_h(x,a)}[ \lng \phi_{h+1}(x', \pi_{h+1}(x')), w \rng ]\label{eq:bellman-linpol}.
\end{align}
Note that $\MT_h^\pi$ depends only on $\pi_{h+1}$; thus with slight abuse of notation, we will sometimes write $\MT_h^{\pi_{h+1}} w$ in place of $\MT_h^\pi w$. We also remark that by linearity of expectation, for any $\pi \in \Pilinpp$ and $w \in \BR^d$, there is a vector $\MT_h^\pi w \in \BR^d$ so that \cref{eq:bellman-linpol} holds for the perturbed linear policy $\pi$. 

As an immediate consequence of \cref{lem:lin-lb}, we can show that the $Q$-function for a linear policy is linear.
\begin{corollary}
  \label{cor:qlin}
Suppose that $M$ is linear Bellman complete, and that $\pi \in \Pilinpp$. Then for each $h \in [H]$, there is a vector $w_h^\pi \in H \cdot \MB_h \subset \BR^d$ so that for all $(x,a) \in \MX \times \MA$, $Q_h^\pi(x,a) = \lng w_h^\pi, \phi_h(x,a) \rng$. Moreover, $\| w_h^\pi \|_2 \leq HB$. 
\end{corollary}
\begin{proof}
  Since $Q_h^\pi(x,a) = r_h(x,a) + \E_{x' \sim P_h(x,a)} [Q_{h+1}^\pi(x, \pi_{h+1}(x))]$, by induction and the fact that $r_h(x,a)$ is linear, it suffices to show that for all $v \in \BR^d$, there is some $w \in \BR^d$ so that $\lng \phi_h(x,a), w \rng = \E_{x' \sim P_h(x,a)} [ \lng \phi_h(x, \pi_{h+1}(x)), v \rng]$. We may write $\pi_{h+1} = \pi_{h+1,y,\sigma}$ for some $y \in \BR^d$ and $\sigma \geq 0$; then by \cref{lem:lin-lb}, we may take $w = \E_{z \sim \MN(y, \sigma^2)}[L_h(z)] \cdot v$.

  To see the upper bound on $\| w_h^\pi \|_2$, note that, by definition of $Q_h^\pi$, we have $|\lng w_h^\pi, \phi_h(x,a) \rng |\leq  H$ for all $x,a,h$. Then it follows that $\| w_h^\pi \|_2 \leq HB$ by \cref{asm:boundedness}. 
\end{proof}

\subsection{Construction of Bellman-linear functions for bonuses}
A central component of the proof consists of the construction of functions, to be used as exploration bonuses, which (a) are Bellman-linear and (b) approximate a truncated version of arbitrary norms induced by PSD matrices. In this section, we establish Bellman linearity of certain functions which form the building blocks of these bonuses.

Given orthogonal vectors $u,v \in \BR^d$, we let $\Pi_{u,v} := \frac{uu^\t}{\| u \|_2} + \frac{vv^\t}{\| v \|_2}$. Note that if $u,v$ are unit vectors, then $\Pi_{u,v}$ is the  projection matrix onto $\Span\{u,v\}$.

\begin{definition}[Orthogonal pair]
  \label{def:op}
  We define a tuple of $d \times d$ PSD matrices $(\Sigma, \Lambda)$ to be an \emph{orthogonal pair} if the following equalities hold:
  \begin{align}
\Sigma^2 = \Sigma, \ \Lambda^2 = \Lambda, \qquad \Sigma \Lambda = \Lambda \Sigma = 0, \qquad \Sigma + \Lambda = I_d\nonumber.
  \end{align}
\end{definition}
For an orthogonal pair $(\Sigma, \Lambda)$, $\Sigma, \Lambda$ are projections onto subspaces of $\BR^d$, which we denote by $\MS_\Sigma := \{ \Sigma v :\ v \in \BR^d\}, \MS_\Lambda := \{ \Lambda v :\ v \in \BR^d\}$, respectively. Since $\Sigma \Lambda = \Lambda \Sigma = 0$, the subspaces $\MS_\Sigma, \MS_\Lambda$ are orthogonal. Since $\Sigma + \Lambda = I_d$, we have that $\dim(\MS_\Sigma) + \dim(\MS_\Lambda) = d$.

\begin{definition}[Truncated linear bonus]
  \label{def:trunc-perim}
  Consider orthogonal vectors $u,v \in \BR^d$. For $\Phi \in \SP^d$, we define
  \[
\Ftp(\Phi; u,v) = \max_{\phi \in \Phi} \lng u, \phi  \rng + \max_{\phi\in \Phi} \lng v, \phi \rng - \max_{\phi \in \Phi} \lng u+v, \phi  \rng.
    \]
  
  For $x \in \MX$ and $h \in [H]$, we then define $\Ftp_h(x; u,v) := \Ftp(\cPhi_h(x); u,v)$. 
\end{definition}

The following lemma is immediate from the definition of $\Ftp(\cdot)$.
\begin{lemma}[Bellman-linearity of truncated linear bonus]
Given any $h > 1$ and vectors $u,v \in \BR^d$, the mapping $x \mapsto \Ftp_h(x; u,v)$ is Bellman-linear at step $h$. 
\end{lemma}

We next record the following lemmas for future use, which follow as an immediate consequence of the definition of $\Ftp$.
\begin{lemma}
  \label{lem:polygon-isometry}
 Let $\Phi \in \SP^d$. Suppose that $u,u', v,v' \in \BR^d$ satisfy $\lng u, \phi\rng = \lng u', \phi \rng$ and $\lng v, \phi \rng = \lng v', \phi \rng$ for all $\phi \in \Phi$. Then
  \begin{align}
    \Ftp(\Phi; u,v) = \Ftp(\Phi; u',v')\nonumber.
  \end{align}
\end{lemma}
\begin{proof}
We have that $\max_{\phi\in \Phi} \lng w, \phi  \rng = \max_{\phi  \in \Phi} \lng w', \phi  \rng$ for each pair $(w,w') \in \{(u,u'), (v,v'), (u+v,u'+v')\}$. 
\end{proof}

\begin{lemma}
  \label{lem:alpha-lb}
  Let $\Phi \in \SP^d$. Suppose that $u,v \in \BR^d$ are given and that $\alpha_u, \alpha_v \geq 0$ are real numbers. Then
  \begin{align}
\Ftp(\Phi; \alpha_u \cdot u, \alpha_v \cdot v ) \geq \min\{ \alpha_u, \alpha_v \} \cdot \Ftp(\Phi; u,v)\nonumber.
  \end{align}
\end{lemma}
\begin{proof}
  Since $\Ftp(\Phi; \alpha u, \alpha v) = \alpha \cdot \Ftp(\Phi; u,v)$ for any $\alpha \geq 0$, it suffices to consider the case that $\min\{ \alpha_u, \alpha_v \} = 1$. Without loss of generality let us suppose that $\alpha_u = 1 \leq \alpha_v$. Then
  \begin{align}
    \Ftp(\Phi; u, \alpha_v \cdot v) - \Ftp(\Phi; u,v) =& (\alpha_v - 1) \cdot \max_{\phi \in \Phi} \lng v, \phi \rng - \max_{\phi \in \Phi} \lng u + \alpha_v \cdot  v, \phi \rng + \max_{\phi \in \Phi} \lng u + v, \phi \rng
    \geq 0\nonumber,
  \end{align}
  where the final inequality follows since $u + \alpha_v \cdot v = (u+v) + (\alpha_v - 1) \cdot v$. 
\end{proof}

\subsection{An upper bound on the truncated linear bonus}
\label{sec:tp-upper-bound}

\cref{lem:tp-upper-bound} below shows that the truncated linear bonus $\Ftp(\Phi; u,v)$ can be bounded above by the length of the projection of $\Phi$ onto each of $u$ and $v$. 
\begin{lemma}
  \label{lem:tp-upper-bound}
  Suppose $u,v \in \BR^d$ and $\Phi \in \SP^d$. Then
  \begin{align}
0 \leq \Ftp(\Phi; u,v) \leq 2 \cdot \min \left\{ \max_{\phi, \phi' \in \Phi} \lng u, \phi - \phi' \rng, \max_{\phi, \phi' \in \Phi} \lng v, \phi - \phi' \rng \right\} \nonumber.
  \end{align}
\end{lemma}
\begin{proof}
  For any fixed $\phi$, we have
  \begin{align}
\lng u+v, \phi \rng \leq \max_{\phi' \in \Phi} \lng u, \phi' \rng + \max_{ \phi' \in \Phi} \lng v,  \phi' \rng\nonumber,
  \end{align}
  which implies $\Ftp(\Phi; u,v) \geq 0$.

  To upper bound $\Ftp(\Phi; u,v)$, we first note that
  \begin{align}
    \Ftp(\Phi; u,v) \leq & \Ftp(\Phi; u,v) + \Ftp(\Phi; -u,-v)\nonumber\\
    =& \max_{\phi, \phi' \in \Phi} \lng u, \phi - \phi' \rng + \max_{\phi, \phi' \in \Phi} \lng v, \phi - \phi' \rng - \max_{\phi, \phi' \in \Phi} \lng u+v, \phi - \phi' \rng \nonumber.
  \end{align}
We will now upper bound $\Ftp(\Phi; u,v) + \Ftp(\Phi; -u,-v)$ as follows. By symmetry, it is without loss of generality to assume that $\max_{\phi, \phi' \in \Phi} \lng u, \phi - \phi' \rng \leq \max_{\phi, \phi' \in \Phi} \lng v, \phi - \phi' \rng$. Now we have
  \begin{align}
    \max_{\phi, \phi' \in \Phi} \lng u+v, \phi - \phi' \rng \geq & \max_{\phi, \phi' \in \Phi} \lng v, \phi - \phi' \rng - \max_{\phi, \phi'\in \Phi} \lng -u, \phi - \phi' \rng \nonumber\\
    =& \max_{\phi, \phi' \in \Phi} \lng v, \phi - \phi' \rng - \max_{\phi, \phi'\in\Phi} \lng u, \phi - \phi' \rng \nonumber.
  \end{align}
  Rearranging, we see that $\Ftp(\Phi; u,v) + \Ftp(\Phi; -u,-v) \leq 2 \max_{\phi, \phi'\in \Phi} \lng u, \phi - \phi' \rng$, as desired. 
\end{proof}

\subsection{A lower bound on the modified truncated linear bonus}
\label{sec:tp-lower-bound}
Given $\Phi \in \SP^d$ and a matrix $\Gamma \in \BR^{d \times d}$, let $\MS_\Phi \subset \BR^d$ denote the subspace $\MS_\Phi := \Span \{ \phi : \phi \in \Phi \}$, and $\MS_{\Phi, \Gamma} \subset \BR^d$ denote the subspace $\MS_{\Phi, \Gamma} := \Span\{\Gamma \cdot \phi : \ \phi \in \Phi\} = \Gamma \cdot \MS_\Phi$. Let $\Pi_{\Phi, \Gamma} \in \BR^{d \times d}$ denote the matrix which projects onto $\MS_{\Phi, \Gamma}$. Recall that for an orthogonal pair $(\Sigma, \Lambda)$ (\cref{def:op}), we have $\MS_\Sigma = \Sigma \cdot \BR^d, \MS_\Lambda = \Lambda \cdot \BR^d$. Then for any $\Phi \in \SP^d$, $\MS_{\Phi, \Sigma} \subset \MS_\Sigma$ and $\MS_{\Phi, \Lambda} \subset \MS_\Lambda$. 

Moreover, given  $\Phi \in \SP^d$ and $\phi_1, \phi_2 \in \Phi$, we define
  \begin{align}
\midpoint[\Phi, \phi_1, \phi_2] := \argmin_{\xi \in \Phi} \left\{ \| \beta \Sigma \cdot (\phi_1 - \xi) \|_2+ \| \Lambda \cdot (\xi - \phi_2) \|_2 \right\}\label{eq:define-midpoint}.
  \end{align}
  The point $\midpoint[\Phi, \phi_1, \phi_2]$ can be thought of as a sort of ``midpoint'' between $\phi_1$ and $\phi_2$ in $\Phi$, where the distance between $\phi_1$ and $\midpoint$ is measured with respect to $\beta \Sigma$ and the distance between $\midpoint$ and $\phi_2$ is measured with respect to $\Lambda$. As it turns out, this particular notion of midpoint will be useful in proving that our exploration bonuses induce optimistic value functions.

  Finally, we introduce some notation regarding the normal cone of a convex body. For a convex subset $\MC \subset \BR^d$ and $z \in \MC$, the normal cone of $\MC$ at $z$, denoted $N_\MC(z)$ is defined as $N_\MC(z) := \{ y \in \BR^d :\ \lng y, z-z' \rng \geq 0 \ \forall z' \in \MC \}$. %
\begin{lemma}
  \label{lem:optimal-perimeter}
Consider an orthogonal pair $(\Sigma, \Lambda)$ and $\beta \geq 1, \zeta > 0$. 
Suppose that $\Phi \in \SP^d$ satisfies
  \begin{align}
\max \left\{  \max_{\phi_1, \phi_2 \in \Phi} \| \beta \Sigma \cdot (\phi_1 - \phi_2) \|_2 , \  \max_{\phi_1, \phi_2 \in \Phi} \| \Lambda \cdot (\phi_1 - \phi_2) \|_2\right\} \leq \zeta. \label{eq:skew-bound}
  \end{align}
  Fix any $\phi_1, \phi_2 \in \Phi$. Then there are unit vectors $u \in \MS_{\Phi, \Sigma}$ and $v \in \MS_{\Phi, \Lambda}$ (depending on $\Phi, \phi_1, \phi_2$) so that for all unit vectors $u' \in \MS_\Sigma, v ' \in \MS_\Lambda$ satisfying %
  \begin{align}
    \| u-u' \|_2 \leq \vep, \qquad \| v-v' \|_2 \leq \vep, 
    \qquad \lng u', u \rng \geq \eta, \qquad \lng v', v \rng \geq \eta\label{eq:uuprime-vvprime-beta},
  \end{align}
  it holds that 
  \begin{align}
\Ftp(\Phi; \beta u', v')    & \geq   \eta \|\beta \Sigma \cdot (\phi_1 - \midpoint) \|_2+\eta  \| \Lambda \cdot (\midpoint - \phi_2) \|_2 - 2\vep\zeta\nonumber.
  \end{align}
\end{lemma}
\begin{proof}
  Fix $\Phi \in \SP^d$ satisfying \cref{eq:skew-bound} and $\phi_1, \phi_2 \in \Phi$.

  \paragraph{Step 1: Choosing $u,v$.} The KKT conditions for optimality of $\midpoint$ (as defined in \cref{eq:define-midpoint}) give that, for all $z \in \bar \Phi$,
  \begin{align}
 \frac{\beta^2\Sigma \cdot (\phi_1 - \midpoint )}{\| \beta\Sigma \cdot (\phi_1 - \midpoint)\|_2} + \frac{\Lambda \cdot (\phi_2 - \midpoint)}{\| \Lambda \cdot (\phi_2 - \midpoint)\|_2} \in N_{\Phi}(\midpoint)\label{eq:wstar-kkt},
  \end{align}
  where $N_\Phi(\xi)$ denotes the normal cone of $\Phi$ at $\xi$.%
  We now set
  \begin{align}
u :=  \frac{\Sigma \cdot (\phi_1 - \midpoint)}{\| \Sigma \cdot ( \phi_1 - \midpoint)\|_2}, \qquad v := \frac{\Lambda \cdot (\phi_2 - \midpoint)}{\| \Lambda \cdot (\phi_2 - \midpoint)\|_2}\nonumber. 
  \end{align}
  Then \cref{eq:wstar-kkt} states that for all $z \in \Phi$, $\lng \beta u+v,  \midpoint - z \rng \geq 0$. Moreover, since $\midpoint \in \Phi$, we certainly have that $u \in \MS_{\Phi, \Sigma} \subset \MS_\Sigma,\ v \in \MS_{\Phi, \Lambda}\subset \MS_\Lambda$. 

\paragraph{Step 2: Relating $u,v$ to $u',v'$ satisfying \cref{eq:uuprime-vvprime-beta}.}  Now consider any unit vectors $u' \in \MS_\Sigma, v' \in \MS_\Lambda$ satisfying \cref{eq:uuprime-vvprime-beta}. Choose %
  \begin{align}
\xi' := \argmax_{\xi \in \Phi} \lng \beta u' + v',\xi \rng\label{eq:whprime-min}.
  \end{align}
  Next, we compute 
  \begin{align}
    \lng \beta u' + v', \xi' - \midpoint \rng =& \lng (\beta u' + v') - (\beta u + v), \xi' - \midpoint \rng + \lng \beta u + v, \xi' - \midpoint\rng\nonumber\\
    \leq & \beta | \lng u-u', \xi' - \midpoint \rng | + | \lng v-v', \xi' - \midpoint \rng | \nonumber\\
    = & \beta | \lng u-u', \Sigma \cdot (\xi' - \midpoint) \rng | + | \lng v-v', \Lambda \cdot (\xi' - \midpoint) \rng | \nonumber\\    
    = & \beta | \lng \Pi_{\Phi, \Sigma} \cdot (u-u'), \Sigma \cdot (\xi' - \midpoint) \rng | + | \lng \Pi_{\Phi, \Lambda} \cdot (v-v'), \Lambda \cdot (\xi' - \midpoint) \rng | \nonumber\\      \leq &  \beta \| \Pi_{\Phi, \Sigma} \cdot (u-u') \|_2 \cdot \| \Sigma \cdot (\xi' - \midpoint) \|_2 + \| \Pi_{\Phi, \Lambda} \cdot ( v-v') \|_2\cdot  \| \Lambda \cdot (\xi' - \midpoint) \|_2\nonumber\\
    \leq & 2\vep \zeta\label{eq:wprime-wstar-relate},
  \end{align}
  where the first inequality uses the triangle inequality and the fact that $\lng \beta u + v, \xi' -\midpoint \rng \leq 0$ since $\xi' \in \Phi$; the second equality uses the fact that $u-u' \in \MS_\Sigma$ and $v-v' \in \MS_\Lambda$; the third equality uses the fact that $\Sigma \cdot (\xi' - \midpoint) \in \MS_{\Phi, \Sigma}$ and $\Lambda \cdot (\xi' - \midpoint) \in \MS_{\Phi, \Lambda}$; the second inequality uses Cauchy-Schwarz; %
  and the third inequality uses \cref{eq:uuprime-vvprime-beta} and \cref{eq:skew-bound}. 
  
  We  then have
  \begin{align}
      & \lng  \phi_1-\xi',\beta  u' \rng + \lng  \phi_2-\xi', v' \rng \nonumber\\
    = &  \lng \Sigma \cdot (\phi_1-\xi'), \beta u'+v' \rng + \lng \Lambda \cdot ( \phi_2-\xi'), v'+\beta u' \rng \nonumber\\
    = &  \lng \Sigma \cdot \phi_1 + \Lambda \cdot \phi_2 - \xi', \beta u' + v' \rng \nonumber\\
    \geq & \lng \Sigma \cdot \phi_1 + \Lambda \cdot \phi_2 - \midpoint, \beta u' + v' \rng - 2 \vep \zeta\nonumber\\
    = & \lng \Sigma \cdot (\phi_1 - \midpoint), \beta u' \rng + \lng \Lambda \cdot (\phi_2 - \midpoint), v' \rng -2\vep\zeta\nonumber\\
    \geq  & \eta \beta \|  \Sigma \cdot (\phi_1 - \midpoint) \|_2 + \eta \| \Lambda \cdot (\phi_2 - \midpoint) \|_2-2\vep\zeta \label{eq:uprime-vprime-uep},
  \end{align}
  where the first equality uses that $\Sigma u' = u', \Sigma v' = v'$, and  $\Sigma v' = \Lambda u' = 0$ since $(\Sigma, \Lambda)$ is an orthogonal pair, the second equality uses that $\Sigma + \Lambda = I_d$, the first  inequality uses \cref{eq:wprime-wstar-relate}, %
  the third equality again uses $\Sigma + \Lambda = I_d$ and the fact that $\Sigma v' = \Lambda u' = 0$, and the second inequality uses \cref{eq:uuprime-vvprime-beta} and the definition of $u,v$.

  \paragraph{Step 3: wrapping up.} Finally, we may write
  \begin{align}
    \Ftp(\Phi; \beta u', v') =& \max_{\phi \in \Phi} \lng \phi, \beta u' \rng + \max_{\phi \in \Phi} \lng \phi, v' \rng - \max_{\phi \in \Phi} \lng \phi, \beta u' + v' \rng\nonumber\\
    = & \max_{\phi \in \Phi} \lng \phi, \beta u' \rng + \max_{\phi \in \Phi} \lng \phi, v' \rng - \lng \xi, \beta u' + v' \rng\nonumber\\
    \geq & \lng \phi_1, \beta u' \rng + \lng \phi_2, v' \rng - \lng \xi, \beta u' + v' \rng\nonumber\\
    \geq & \eta \beta \| \Sigma \cdot (\phi_1 - \midpoint)\|_2 + \eta \| \Lambda \cdot (\phi_2 - \midpoint) \|_2 - 2\vep\zeta\nonumber,
  \end{align}
  where the second equality uses the definition of $\xi$ in \cref{eq:whprime-min}, and the second inequality uses \cref{eq:uprime-vprime-uep}.
\end{proof}

By averaging over $u',v'$ drawn randomly from appropriate subspaces, we have the following consequence of \cref{lem:optimal-perimeter}.
\begin{lemma}
  \label{cor:optimal-perimeter}
  There is a constant $C_{\ref{cor:optimal-perimeter}}$ so that the following holds. Fix $A \in \BN$, an orthogonal pair $(\Sigma, \Lambda)$, $\beta \geq 1$, $\zeta > 0$, $\Phi \in \SP_A^d$, $\phi_1, \phi_2 \in \Phi$, and suppose that $\Phi$ satisfies \cref{eq:skew-bound}. 
  Then for any $\vep > 0$, 
  \begin{align}
     & \left( \frac{C_{\ref{cor:optimal-perimeter}}}{\vep} \right)^{2A} \cdot \E_{u' \sim \MN(0, \Sigma)} \E_{v' \sim \MN(0, \Lambda)} \left[ \Ftp(\Phi; \beta u', v') \right] \nonumber\\
    \geq &   \| \beta \Sigma \cdot (\phi_1 - \midpoint) \|_2 +  \| \Lambda \cdot (\midpoint - \phi_2) \|_2 - 4\vep \zeta\nonumber.
  \end{align}
\end{lemma}
\begin{proof}
Note that if $\dim(\MS_{\Phi,\Sigma}) = 0$, then we could choose $\midpoint = \phi_2$ in \cref{eq:define-midpoint}, and thus the lemma statement becomes immediate. A symmetric argument applies to the case that $\dim(\MS_{\Phi,\Lambda}) = 0$. Thus we may assume for the remainder of the proof that $\dim(\MS_{\Phi, \Sigma}), \dim(\MS_{\Phi,\Lambda}) \geq 1$. 
  
  Note that the subspace $\MS_{\Phi, \Sigma} \subseteq \MS_\Sigma$ satisfies $\dim(\MS_{\Phi, \Sigma}) \leq A$, and similarly, $\dim(\MS_{\Phi, \Lambda}) \leq A$. Note that $\E_{u' \sim \MN(0, \Sigma)}[(\Pi_{\Phi, \Sigma} u') \cdot (\Pi_{\Phi, \Sigma} u')^\t ] = \Pi_{\Phi, \Sigma} \cdot \Sigma \cdot \Pi_{\Phi, \Sigma} = \Pi_{\Phi, \Sigma}$, meaning that $\Pi_{\Phi, \Sigma} u' \sim \MN(0, \Pi_{\Phi, \Sigma})$. Similarly, we have that $\Pi_{\Phi, \Lambda}v' \sim \MN(0, \Pi_{\Phi, \Lambda})$.  

  Therefore, we have that $\frac{\Pi_{\Phi, \Sigma} u'}{\| \Pi_{\Phi, \Sigma} u' \|_2} \sim \unif(S^{d-1} \cap \MS_{\Phi, \Sigma})$. Since $S^{d-1} \cap \MS_{\Phi, \Sigma}$ may be covered with $(3/\vep)^A$ balls of radius $\vep$, %
  it follows that, for any $u \in S^{d-1} \cap \MS_{\Phi, \Sigma}$, we have 
  \begin{align}
\Pr_{u' \sim \MN(0, \Sigma)} \left(\left\|\frac{\Pi_{\Phi, \Sigma} u'}{\| \Pi_{\Phi, \Sigma} u' \|_2 } - u\right\|_2 \leq \vep \right) \geq (\vep/3)^{A}\nonumber.
  \end{align}
  Moreover, we have that $\Pr_{u' \sim \MN(0, \Sigma)}( \| \Pi_{\Phi, \Sigma} u' \|_2 \geq 1/2) \geq 1/2$: this holds because $\|\Pi_{\Phi,\Sigma} u'\|_2^2$ is a chi-squared random variable with $\dim(\MS_{\Phi,\Sigma}) \geq 1$ degrees of freedom. Using the fact that $\| \Pi_{\Phi,\Sigma} u' \|_2$ and $\frac{\Pi_{\Phi, \Sigma} u'}{\| \Pi_{\Phi, \Sigma} u' \|_2}$ are independent random variables, it follows that
  \begin{align}
    \frac 12 \cdot (\vep/3)^{A} \leq & \Pr_{u' \sim \MN(0, \Sigma)} \left( \left\|\frac{\Pi_{\Phi, \Sigma} u'}{\| \Pi_{\Phi, \Sigma} u' \|_2 } - u\right\|_2 \leq \vep \mbox{ and } \| \Pi_{\Phi, \Sigma} u' \|_2 \geq 1/2 \right)\nonumber.
  \end{align}

  In a symmetric manner, we obtain that
  \begin{align}
\frac 12 \cdot (\vep/3)^{A} \leq &  \Pr_{v' \sim \MN(0, \Lambda)} \left( \left\|\frac{\Pi_{\Phi, \Lambda} v'}{\| \Pi_{\Phi, \Lambda} v' \|_2 } - v\right\|_2 \leq \vep \mbox{ and } \| \Pi_{\Phi, \Lambda} v' \|_2 \geq 1/2\right) \nonumber. %
  \end{align}

  Let us define the random variables $u'' := \frac{\Pi_{\Phi, \Sigma}u'}{\| \Pi_{\Phi, \Sigma} u' \|_2}$ and $v'' := \frac{\Pi_{\Phi, \Lambda}v'}{\| \Pi_{\Phi, \Lambda}v' \|_2}$. Let $u\in \MS_{\Phi, \Sigma}$ and $v \in \MS_{\Phi,\Lambda}$ be chosen according to the statement of \cref{lem:optimal-perimeter} given $\Phi, \phi_1, \phi_2$. Then with probability at least $(1/4) \cdot (\vep/3)^{2A}$ over the independent draws of $u'$ and $v'$ (which induce values of $u'', v''$), we have that:
  \begin{align}
\| u'' - u \|_2 \leq \vep, \quad \| v'' - v \|_2 \leq \vep, \quad \lng u'', u \rng \geq 1/2, \quad \lng v'', v \rng \geq 1/2, \quad \|  \Pi_{\Phi, \Sigma} u' \|_2 \geq 1/2, \quad \| \Pi_{\Phi, \Lambda} v' \|_2 \geq 1/2\label{eq:uprime-vprime-good},
  \end{align}
  where we have used that $\vep \leq 1$ and that for unit vectors $y,y'$, we have $\lng y,y' \rng = \frac{2 - \| y-y' \|_2^2}{2}$.

  For any $\phi \in \Phi \subset \MS_\Phi$, we have
  \begin{align}
\lng u', \phi \rng = \lng u', \Sigma \phi \rng = \lng \Pi_{\Phi, \Sigma} u', \Sigma \phi \rng = \lng \Pi_{\Phi, \Sigma} u', \phi \rng,\nonumber
  \end{align}
  where  the first equality uses that $u' \in \MS_\Sigma$, the second equality uses that $\Sigma \phi \in \MS_{\Phi, \Sigma}$, and the third equality uses that $\Pi_{\Phi, \Sigma}u' \in \MS_{\Phi, \Sigma} \subseteq \MS_\Sigma$. %
  In a similar manner, we have that, for all $\phi \in \Phi$, $\lng v', \phi \rng = \lng \Pi_{\Phi, \Lambda} v', \phi \rng$. It then follows from \cref{lem:polygon-isometry} that
  \begin{align}
    \Ftp(\Phi; \beta u', v') = \Ftp(\Phi; \beta \cdot \Pi_{\Phi, \Sigma} u', \Pi_{\Phi,\Lambda}v')
         \label{eq:uprime-piuprime}.
  \end{align}
Under the event that $\| \Pi_{\Phi, \Sigma} u' \|_2 \geq 1/2$ and $\| \Pi_{\Phi, \Lambda} v' \|_2 \ge 1/2$, we have by \cref{lem:alpha-lb} that %
\begin{align}
  \Ftp(\Phi; \beta \Pi_{\Phi, \Sigma} u', \Pi_{\Phi, \Lambda}v') \geq \frac 12 \cdot \Ftp(\Phi; \beta u'', v'') \label{eq:upp-up}. 
  \end{align}

  Combining \cref{eq:uprime-piuprime,eq:upp-up}, we see that
  \begin{align}
    & \E_{u' \sim \MN(0, \Sigma)} \E_{v' \sim \MN(0, \Lambda)} \left[ \Ftp(\Phi; \beta u', v')\right]\nonumber\\
    \geq & \frac 12 \E_{u' \sim \MN(0, \Sigma)} \E_{v' \sim \MN(0, \Lambda)} \left[ \Ftpm(\Phi; \beta u'', v'')\right]\nonumber\\
    \geq & \frac 14 \cdot \frac 14 \cdot \left( \frac{\vep}{3} \right)^{2A} \cdot \left( \| \beta \Sigma \cdot (\phi_1 - \midpoint) \|_2 + \| \Lambda \cdot (\midpoint - \phi_2) \|_2  - 4\vep\zeta\right)\nonumber,
  \end{align}
  where the second inequality uses the fact that \cref{eq:uprime-vprime-good} holds with probability at least $(1/4) \cdot (\vep/3)^{2A}$ together with \cref{lem:optimal-perimeter} with $\eta = 1/2$, $u'$ set to $u''$, and $v'$ set to $v''$. Rearranging and using that $A \geq 2$, we obtain that
  \begin{align}
    & \left( \frac{6}{\vep} \right)^{2A} \cdot  \E_{u' \sim \MN(0, \Sigma)} \E_{v' \sim \MN(0, \Lambda)} \left[ \Ftpm(\Phi; \beta u', v' )\right] \nonumber\\
    \geq & \| \beta \Sigma \cdot (\phi_1 - \midpoint) \|_2 + \| \Lambda \cdot (\midpoint - \phi_2) \|_2  - 4\vep\zeta\nonumber,
  \end{align}
  as desired. (In particular, we may take $C_{\ref{cor:optimal-perimeter}}  = 6$.)
  \end{proof}

\section{The algorithm}
\label{sec:alg-description}
\paragraph{The algorithm.} \cref{alg:psdp-ucb} presents our algorithm for learning linear Bellman complete MDPs. For some $T \in \BN$, the algorithm proceeds for $T$ \emph{rounds}. In each round $t \in [T]$, the algorithm computes two policies $\hat \pi_h\^t, \tilde \pi_h\^t : \MX \ra \MA$ at step $h$, in order of decreasing $h$. For each step $h$, trajectories are gathered from the uniform mixture over $\hat \pi\^s \circ_h \tilde \pi\^s \circ_{h+1} \hat \pi\^t$, for $s \in [t-1]$ (\cref{line:collect-tih-samples}). These trajectories are then used, together with a procedure to construct an optimism-based upper confidence bonus, to define a policy $\hat \pi_h\^t$ (\cref{line:define-w-hat,line:define-pi-hat}), which represents the algorithm's current optimistic estimate of the optimal policy. Finally, $\tilde \pi_h\^t$ is defined to be a sort of uniform policy (\cref{line:define-pi-tilde}), with respect to the current covariance matrix of features $\Sigma_h\^t$.

\paragraph{Parameter definitions.} 
Below, we specify the parameters $\lambda_1, \beta$ used in the definition of the bonuses $F_h\^t$ (in \cref{eq:define-fht-bonus} below), the parameters $\lambda, n, T$ used in \cref{alg:psdp-ucb}, as well as several other parameters used in the analysis of this section:
\begin{definition}[Parameter definitions]
  \label{def:params}
  Fix $\delta \in (0,1), \epfinal \in (0,1)$, as well as parameters $A, H, \Bbnd, d$ of the MDP. 
We define the following parameters, which are used in \cref{alg:psdp-ucb} as well in the remainder of this section:
  \begin{itemize}
  \item $\lambda = C_{\ref{lem:psd-concentration}} d \log(2THn/\delta)$, where $C_{\ref{lem:psd-concentration}}$ is a constant chosen sufficiently large in \cref{lem:psd-concentration}. 
  \item $T = d \cdot \left( \frac{C_{\ref{thm:policy-learning}} H^4 \Bbnd^3 d A^{1/2} \log^{1/2}(HA\Bbnd d/(\epfinal\delta))}{\epfinal} \right)^{6A+2}$, where $C_{\ref{thm:policy-learning}}$ is a constant chosen sufficiently large in the proof of \cref{thm:policy-learning}.
  \item $n=3T$.
\item $\iota = { \log \left( \frac{TH(\lambda d + n)}{\delta} \right)}$.    
\item $\lambda_1 = BH$.  %
\item $\epbell = \frac{\epfinal}{2H}$. 
\item $\sigtr = \frac{\epbell}{4\lambda_1}$. 
\item $\epapx = \frac{\epbell}{128\sqrt{2\pi}C_{\ref{lem:phi-what-wt-diff}}  \lambda_1^2 HB d\sqrt{\iota}}$, where $C_{\ref{lem:phi-what-wt-diff}}$ is chosen sufficiently large in the proof of  \cref{lem:phi-what-wt-diff} holds.
\item $\beta = 4 C_{\ref{lem:phi-what-wt-diff}} HB \sqrt{d\iota} \cdot {5 \lambda_1 \sqrt{d}}\cdot \left( \frac{C_{\ref{cor:optimal-perimeter}}}{\epapx} \right)^{2A}$.  
\item $\xi = \frac{\beta}{4 C_{\ref{lem:phi-what-wt-diff}} HB \sqrt{d\iota} \cdot 2\sqrt{2\pi} \lambda_1 \sqrt{d}}$. (Note that $\xi \geq 1$.) %
\end{itemize}
\end{definition}

  \paragraph{Definition of bonuses.} Next we define the exploration bonus function $F_g\^t : \MX \ra \BR$ used in the construction of the rewards $\hat r_h\^{t,i,h}$ in \cref{line:rhat-rewards}. The key component in $F_g\^t$ is a linear combination of truncated linear bonuses $\Ftp_h(x; u,v)$, over various values of $u,v$. The vectors $u,v$ are chosen to belong to the subspaces $\MS_\Sigma, \MS_\Lambda$, for matrices $\Sigma, \Lambda$ forming an orthogonal pair which are obtained from the covariance matrix $\Sigma_h\^t$ via a certain truncation procedure, defined below:
\begin{definition}
  \label{def:mat-truncation}
  Let $\Gamma \in \BR^{d \times d}$ be a PSD matrix, and suppose that $\Gamma$ may be diagonalized as $\Gamma = UDU^\t$, for a diagonal matrix $D$ and an orthogonal matrix $U$. Given $\sigma > 0$, we define the \emph{$\sigma$-truncated pair of $\Gamma$}, denoted $(\Sigma', \Lambda') := \trunc{\Gamma}{\sigma}$, as $\Sigma' = UD_+U^\t, \Lambda' = U D_-U^\t$, where $D_+, D_-$ are diagonal matrices whose entries are given by:
  \begin{align}
    (D_+)_{ii} = \begin{cases}
      1 &: D_{ii} \geq \sigma \\
      0 &: D_{ii} < \sigma
    \end{cases}, \quad
          (D_-)_{ii} = \begin{cases}
            0 &: D_{ii} \geq \sigma \\
            1 &: D_{ii} < \sigma
          \end{cases}\label{eq:d-plus-minus}.
  \end{align}
\end{definition}
In words, for $(\Sigma', \Lambda') = \trunc{\Gamma}{\sigma}$, $\Sigma'$ contains the large components of $\Gamma$ (as parametrized  by $\sigma$), and $\Lambda'$ contains the small components of $\Gamma$. Note that the definition of $\trunc{\Gamma}{\sigma}$ does not depend on the choice of diagonalization. Moreover, note that $(\Sigma', \Lambda') = \trunc{\Gamma}{\sigma}$ is an orthogonal pair.

The bonus function $F_h\^t$ used in \cref{alg:psdp-ucb} is defined as follows: for each $t \in [T]$, $h \in [H]$,%
\begin{align}
  F_h\^t(x) :=&  \lambda_1 \cdot \left( \frac{C_{\ref{cor:optimal-perimeter}}}{\epapx} \right)^{2A} \cdot   \E_{u' \sim \MN(0, \Sigma')} \E_{v' \sim \MN(0, \Lambda')} \left[ \Ftp_h(x; \beta u', v') \right]\nonumber\\
              & +  2\sqrt{2\pi} \lambda_1 \xi  \cdot \E_{w \sim \MN(0, \Sigma')} \left[ \max_{a \in \MA} \lng w, \phi_h(x, a) \rng \right] \label{eq:define-fht-bonus}\\
  \mbox{ for }  & \quad \quad (\Sigma', \Lambda') := \trunc{(\beta/\lambda_1) \cdot (\Sigma_h\^t)^{-1/2}}{\sigtr}.\nonumber
\end{align}

\paragraph{$Q$- and $V$-function definitions.}
For each $h \in [H], t \in [T]$, we make the following definitions:
\begin{align}
  Q_h\^t(x,a) :=&  r_h(x,a) +  \E^{\hat \pi\^t}\left[ \sum_{g=h+1}^H r_g(x_g, a_g) +  F_g\^t(x_g)\ \mid \ (x_h,a_h) = (x,a) \right]\label{eq:define-qht}\\
  V_h\^t(x) :=& Q_h\^t(x, \hat \pi_h\^t(x))\label{eq:define-vht}\\
  \hat Q_h\^t(x,a) :=& \lng \hat w_h\^t, \phi_h(x,a) \rng \nonumber\\
  \hat V_h\^t(x) :=& \hat Q_h\^t(x, \hat \pi_h\^t(x)) = \max_{a \in \MA} \hat Q_h\^t(x,a)\nonumber.
\end{align}
The functions $Q_h\^t(\cdot)$ and $V_h\^t(\cdot)$ represent the $Q$- and $V$-value functions for the policy $\hat \pi\^t$, with respect to the rewards $\hat r_g\^t$ defined in \cref{eq:rhat-rewards}. In \cref{lem:wht-exist-bound} below, we show that $Q_h\^t$ is a linear function, i.e., for some $w_h\^t \in \BR^d$, $Q_h\^t(x,a) = \lng \phi_h(x,a), w_h\^t \rng$ for all $(x,a) \in \MX \times \MA$.  The vector $\hat w_h\^t$ computed in \cref{alg:psdp-ucb} may be viewed as an empirical approximation to $w_h\^t$, so that $\hat Q_h\^t(\cdot)$ and $\hat V_h\^t(\cdot)$ represent empirical approximations to $Q_h\^t, V_h\^t$, respectively. 
\begin{lemma}
  \label{lem:wht-exist-bound}
For all $t \in [T], h \in [h]$, there is some $w_h\^t \in  \BR^d$ so that $Q_h\^t(x,a) = \lng \phi_h(x,a), w_h\^t \rng$ for all $(x,a) \in \MX \times \MA$.
\end{lemma}
\begin{proof}
  For all $x,a,h,t$, we have
  \begin{align}
Q_h\^t(x,a) = r_h(x,a) + \E_{x' \sim \BP_h(x,a)}\left[ F_{h+1}\^t( x')  + Q_{h+1}\^t(x', \hat \pi_{h+1}\^t(x'))) \right]\nonumber.
  \end{align}
  For fixed $t \in [T]$, we will use reverse induction on $h$ to establish the stated claim. The base case $h=H$ follows since $Q_H\^t(x,a) = r_H(x,a) = \lng w_H\^t, \phi_H(x,a) \rng$ for some $w_H \in \BR^d$, by assumption. %

  Now, assume for the inductive hypothesis that for all $(x,a) \in \MX \times \MA$, $Q_{h+1}\^t(x,a) = \lng w_{h+1}\^t, \phi_{h+1}(x,a) \rng$ for some $w_{h+1}\^t \in \BR^d$. Then \cref{lem:lin-lb} gives that there is some $\theta^1 \in \BR^d$ so that for all $(x,a) \in \MX \times \MA$, 
  \begin{align}
\lng \theta^1, \phi_h(x,a) \rng = \E_{x' \sim \BP_h(x,a)}\left[ \lng \phi_{h+1}(x', \hat \pi_{h+1}\^t(x')), w_{h+1}\^t \rng \right] = \E_{x' \sim \BP_h(x,a)} \left[ Q_{h+1}\^t(x', \hat \pi_{h+1}\^t(x'))\right]\nonumber.
  \end{align}
  Next, note that it is immediate from the definition \cref{eq:define-fht-bonus} that, with probability 1 over $u' \sim \MN(0, \Sigma')$ and $v' \sim \MN(0, \Lambda')$, the function $\Ftp_{h+1}(\cdot; u', v', \beta \Sigma', \Lambda)$ is Bellman-linear at step $h+1$. Note that the function $x \mapsto \max_{a \in \MA} \lng w, \phi_{h+1}(x,a) \rng$ is also Bellman-linear at step $h+1$, for all $w \in \BR^d$. Since the sum of Bellman-linear functions is Bellman-linear, the definition of $F_{h+1}\^t$ in \cref{eq:define-fht-bonus} gives that $F_{h+1}\^t$ is Bellman-linear at step $h+1$. Thus, there is some $\theta^2 \in \BR^2$ so that 
  \begin{align}
\lng \phi_h(x,a), \theta^2 \rng = \E_{x' \sim \BP_h(x,a)} \left[ F_{h+1}\^t( x') \right]\nonumber
  \end{align}
  for all $(x,a) \in \MX \times \MA$. 

  Finally, by assumption, there is some $\theta^3 \in \BR^d$ so that $r_h(x,a) = \lng \theta^3, \phi_h(x,a)\rng$ for all $(x,a) \in \MX \times \MA$. Hence, writing $w_h\^t = \theta^1 + \theta^2 + \theta^3$, we have that $Q_h\^t(x,a) = \lng w_h\^t, \phi_h(x,a)\rng$ for all $(x,a) \in \MX \times \MA$.
  \end{proof}

\subsection{Boundedness of the bonuses}
\label{sec:bonus-boundedness}
Next, we prove some lemmas which establish bounds on the magnitude of the bonus functions $F_h\^t(\cdot)$, and thereby on the magnitude of the optimistic $Q$-functions $Q_h\^t(\cdot)$. 
\begin{lemma}
  \label{lem:orig-truncated}
  For a PSD matrix $\Gamma$ and $\sigma > 0$, writing $(\Sigma', \Lambda') := \trunc{\Gamma}{\sigma}$, we have that $\Sigma' \preceq \frac{1}{\sigma} \cdot \Gamma$.
\end{lemma}
\begin{proof}
We write $\Gamma = UDU^\t$ and $\Sigma' = UD_+U^\t$, where $D_+$ is as defined in \cref{eq:d-plus-minus}. Certainly $0 = (D_+)_{ii} \leq D_{ii}$ if $D_{ii} < \sigma$; otherwise, we have $1 = (D_+)_{ii} \leq \frac{1}{\sigma} \cdot D_{ii}$, which implies that $D_+ \preceq \frac{D}{\sigma}$, as desired.
\end{proof}

\begin{lemma}
  \label{lem:quadratic-sim}
  Let $\Phi \in \SP^d$ be a given polyhedron. Consider any PSD matrix $\Sigma \in \BR^{d \times d}$. Then %
  \begin{align}
\frac{1}{\sqrt{2\pi}} \cdot \max_{\phi, \phi' \in \Phi} ((\phi-\phi')^\t \Sigma (\phi - \phi'))^{1/2} \leq \E_{w \sim \MN(0, \Sigma)} \left[ \max_{\phi \in \Phi} \lng w, \phi \rng \right] \leq \sqrt{d} \cdot \E_{w \sim \MN(0, \Sigma)}\left[ \phi_w^\t \Sigma \phi_w\right]^{1/2},\nonumber
  \end{align}
  where $\phi_w := \argmax_{\phi \in \Phi} \lng w,\phi \rng$. 
\end{lemma}
\begin{proof}
  Define $\Psi = \{ \Sigma^{1/2} \cdot \phi \ : \ \phi \in \Phi \}$, and, for $v \in \BR^d$, write $\psi_v := \argmax_{\psi \in \Psi} \lng v, \psi\rng$. Then
  \begin{align}
\Sigma^{1/2} \cdot \phi_{\Sigma^{1/2} v} = \Sigma^{1/2} \cdot \argmax_{\phi \in \Phi} \lng v , \Sigma^{1/2} \phi \rng = \argmax_{\psi \in \Psi} \lng v, \psi\rng = \psi_v\nonumber.
  \end{align}
  We then have the following equalities:
  \begin{align}
    \label{eq:phi-psi-diff}
    \max_{\phi, \phi' \in \Phi} ((\phi - \phi')^\t \Sigma (\phi - \phi'))^{1/2} =  & \max_{\psi, \psi' \in \Psi} \| \psi - \psi' \|_2\\
    \E_{w \sim \MN(0, \Sigma)} \left[ \phi_w^\t \Sigma \phi_w \right]^{1/2} =& \E_{v \sim \MN(0, I_d)}\left[ \phi_{\Sigma^{1/2} v}^\t \Sigma \phi_{\Sigma^{1/2} v} \right]^{1/2} = \E_{v \sim \MN(0, I_d)} \left[ \| \psi_v \|_2^2 \right]^{1/2}\label{eq:phi-psi-nodiff}\\
\E_{w \sim \MN(0, \Sigma)} \left[\max_{\phi \in \Phi} \lng w, \phi \rng \right] = &  \E_{v \sim \MN(0, I_d)} \left[\max_{\phi \in \Phi} \lng \Sigma^{1/2} v, \phi \rng \right] = \E_{v \sim \MN(0, I_d)} \left[ \max_{\psi \in \Psi} \lng v, \psi \rng \right]\label{eq:normal-phi-psi}.
  \end{align}
  It therefore suffices to show that
  \begin{align}
    \label{eq:psi-2way-inequality}
   \frac{1}{\sqrt{2\pi}} \cdot  \max_{\psi, \psi' \in \Psi} \| \psi - \psi' \|_2 \leq \E_{v \sim \MN(0, I_d)} \left[ \max_{\psi \in \Psi} \lng v, \psi \rng \right] \leq \sqrt{d} \cdot \E_{v \sim \MN(0, I_d)} [ \| \psi_v \|_2^2]^{1/2}.
  \end{align}
  To show the first inequality in \cref{eq:psi-2way-inequality}, pick $\psi_0, \psi_1 \in \Psi$ maximizing $\| \psi_0 - \psi_1 \|_2$. Then
  \begin{align}
    \E_{v \sim \MN(0, I_d)} \left[ \max_{\psi \in \Psi} \lng v, \psi \rng \right] \geq & \E_{v \sim \MN(0, I_d)} \left[ \max \{ \lng v, \psi_0 \rng, \lng v, \psi_1 \rng \} \right]\nonumber\\
    =  & \E_{v \sim \MN(0, I_d)} \left[ \lng v, (\psi_0 + \psi_1)/2 \rng + \max \{ \lng v, (\psi_0 - \psi_1)/2 \rng, \lng v, (\psi_1-\psi_0)/2 \rng \} \right]\nonumber\\
    = & \E_{v \sim \MN(0, I_d)} [ | \lng v, (\psi_0 - \psi_1)/2 \rng |]\nonumber\\
    = & \frac{\sqrt{2}}{2\sqrt \pi} \| \psi_0 - \psi_1 \|_2 \nonumber,
  \end{align}
  where the second equality uses that $\E_{v \sim \MN(0, I_d)}[\lng v, (\psi_0 + \psi_1)/2 \rng] = 0$, and the final inequality uses that $\lng v, (\psi_0 - \psi_1)/2 \rng \sim \MN(0, \| (\psi_0 - \psi_1)/2 \|_2^2)$.

  To show the second inequality in \cref{eq:psi-2way-inequality}, note that
  \begin{align}
    \E_{v \sim \MN(0, I_d)} \left[ \max_{\psi \in \Psi} \lng v, \psi \rng \right] = & \E_{v \sim \MN(0, I_d)} \left[ \lng v, \psi_v \rng \right]\nonumber\\
    \leq & {\E_{v \sim \MN(0, I_d)}[ \| v \|_2^2]}^{1/2} \cdot \E_{v \sim \MN(0, I_d)}[ \| \psi_v \|_2^2]^{1/2}\nonumber\\
    =& \sqrt{d} \cdot \E_{v \sim \MN(0, I_d)}[ \| \psi_v \|_2^2]^{1/2}\nonumber,
  \end{align}
  where the inequality uses Cauchy-Schwarz.
\end{proof}

\begin{lemma}
  \label{lem:bound-truncation-error}
  Suppose that $\Gamma \in \BR^d$ is a PSD matrix, and let $\sigma > 0$ be given. Let $(\Sigma', \Lambda') := \trunc{\Gamma}{\sigma}$. Then for all $h \in [H], x \in \MX$, $a,a' \in \MA$, and $v \in \cPhi_h(x)$, it holds that
  \begin{align}
    & \| \Gamma \cdot (\phi_h(x,a) - v) \|_2 + \| v - \phi_h(x,a') \|_2 \nonumber\\
    \leq  & \| \Gamma \| \cdot \| \Sigma' \cdot (\phi_h(x,a) - v) \|_2 + \| \Lambda' \cdot (v-\phi_h(x,a')) \|_2 +  \sqrt{2\pi} \cdot \E_{w \sim \MN(0, \Sigma')} \left[ \max_{\bar a \in \MA} \lng w, \phi_h(x,\bar a) \right]+ 2 \sigma\nonumber.
  \end{align}
\end{lemma}
\begin{proof}
  Note that $\Gamma^2 \preceq \| \Gamma \|^2 \cdot \Sigma' + \sigma^2 \cdot I_d$, so that
  \begin{align}
    \label{eq:sigma-av}
\| \Gamma \cdot (\phi_h(x,a) - v) \|_2 \leq & \| \Gamma \| \cdot \| \Sigma' \cdot (\phi_h(x,a) - v) \|_2 + \sigma \cdot \| \phi_h(x,a) - v\|_2 \leq \| \Sigma' \cdot (\phi_h(x,a) - v) \|_2 + 2\sigma\,
  \end{align}
  where the final inequality uses that $\max_{\phi \in \bar \Phi_h(x)} \| \phi \|_2 \leq 1$ (\cref{asm:boundedness}). Next, the fact that $\Lambda' + \Sigma' = I_d$ gives that
  \begin{align}
    \| v - \phi_h(x,a') \|_2 \leq & \| \Lambda' \cdot (v-\phi_h(x,a')) \|_2 + \max_{a_1, a_2 \in \MA} \| \Sigma' \cdot (\phi_h(x,a_1) - \phi_h(x,a_2)) \|_2\nonumber\\
    \leq &  \| \Lambda' \cdot (v-\phi_h(x,a')) \|_2 + \sqrt{2\pi} \cdot \E_{w \sim \MN(0, \Sigma')} \left[ \max_{\bar a \in \MA} \lng w, \phi_h(x,\bar a) \right]   \label{eq:lambda-av},
  \end{align}
  where the second inequality uses \cref{lem:quadratic-sim}.  The conclusion of the lemma statement follows by adding \cref{eq:sigma-av,eq:lambda-av}. 
\end{proof}

\begin{lemma}[Bound on bonuses in terms of $\Sigma'$]
  \label{lem:sigmap-bound}
  There is a sufficiently large constant $C_{\ref{lem:sigmap-bound}}$ so that the following holds. 
  Fix $t \in [T], h \in [H]$, and let $(\Sigma', \Lambda') := \trunc{(\beta/\lambda_1) \cdot (\Sigma_h\^t)^{-1/2}}{\sigtr}$. Then for all $x \in \MX$, 
  \begin{align}
|F_h\^t(x)| \leq & \sqrt{d} \beta \lambda_1 \cdot \left( \frac{C_{\ref{lem:sigmap-bound}}}{\epapx} \right)^{2A} \cdot \E_{w \sim \MN(0, \Sigma')} \left[ \max_{a \in \MA} \lng w, \phi_h(x,a) \rng \right]\nonumber.
  \end{align}
\end{lemma}
\begin{proof}
  Fix any $u' \in \MS_{\Sigma'}, v' \in \MS_{\Lambda'}$. Then 
  \begin{align}
    \Ftp_h(x; \beta u', v') \leq &  2 \| u' \|_2 \cdot \max_{a,a' \in \MA} \| \beta \Sigma' \cdot (\phi_h(x,a) - \phi_h(x,a')) \|_2\nonumber\\
    \leq & 2 \sqrt{2\pi}\beta \| u' \|_2 \cdot \E_{w \sim \MN(0, \Sigma')} \left[ \max_{a \in \MA} \lng w, \phi_h(x,a) \rng \right]\nonumber,
  \end{align}
  where the first  inequality uses \cref{lem:tp-upper-bound} with $u = \beta u'$, $v = v'$ as well as the fact that $\Sigma u' = u'$, 
  and the second inequality uses \cref{lem:quadratic-sim}. Using the definition of $F_h\^t(\cdot)$ in \cref{eq:define-fht-bonus}, it follows that, for a sufficiently large constant $C > 0$, 
  \begin{align}
    |F_h\^t(x)| \leq & {\lambda_1} \cdot \left(\frac{ C_{\ref{cor:optimal-perimeter}}}{\epapx} \right)^{2A} \cdot \E_{u' \sim \MN(0, \Sigma')} \left[ 4\sqrt{2\pi} \beta \| u' \|_2 \cdot \E_{w \sim \MN(0, \Sigma')}\left[ \max_{a \in \MA} \lng w, \phi_h(x,a) \rng \right] \right] \nonumber\\
                     &+  2\sqrt{2\pi} \lambda_1 \xi  \cdot \E_{w \sim \MN(0, \Sigma')} \left[ \max_{a \in \MA} \lng w, \phi_h(x, a) \rng \right]\nonumber\\
    \leq &  \left( {4\sqrt{2\pi} \sqrt{d} \beta \lambda_1} \cdot \left( \frac{C_{\ref{cor:optimal-perimeter}}}{\epapx} \right)^{2A} + 2\sqrt{2\pi} \lambda_1 \xi \right) \cdot \E_{w \sim \MN(0, \Sigma')} \left[ \max_{a \in \MA} \lng w, \phi_h(x,a) \rng \right]\nonumber\\
    \leq & \sqrt{d} \beta \lambda_1 \cdot \left( \frac{C}{\epapx} \right)^{2A} \cdot \E_{w \sim \MN(0, \Sigma')} \left[ \max_{a \in \MA} \lng w, \phi_h(x,a) \rng \right]\nonumber,
  \end{align}
  where the final inequality uses that $\xi \leq \beta$ (\cref{def:params}).
\end{proof}

\begin{lemma}[Absolute bound on bonuses]
  \label{lem:bonus-bound}
For all $h,t,x$, it holds that $|F_h\^t(x)| \leq \frac{\beta}{2C_{\ref{lem:phi-what-wt-diff}} HB  \sqrt{d\iota}}$, and that $\| w_h\^t \|_2 \leq \frac{\beta}{C_{\ref{lem:phi-what-wt-diff}}\sqrt{d\iota}}$.
\end{lemma}
\begin{proof}
Fix $t\in[T],h \in [H]$, and set $(\Sigma', \Lambda'):= \trunc{(\beta/\lambda_1) \cdot (\Sigma_h\^t)^{-1/2}}{\sigtr}$, as per \cref{eq:define-fht-bonus}.  For any fixed $u', v' \in \BR^d$ with $u' \in \MS_{\Sigma'}, v' \in \MS_{\Lambda'}$, we have %
  \begin{align}
    \Ftp_h(x; \beta u', v') \leq & %
     2 \| v' \|_2 \cdot \max_{\phi, \phi' \in  \cPhi_h(x)} \| \Lambda' \cdot (\phi - \phi') \|_2 
    \leq  4 \| v' \|_2\nonumber,
  \end{align}
  where the first inequality  uses \cref{lem:tp-upper-bound} and the fact that $\Lambda v' = v'$,  %
  and the second inequality uses that $\max_{h,x,a} \| \phi_h(x,a) \|_2 \leq 1$. Then, by the definition of $F_h\^t$ in \cref{eq:define-fht-bonus}, we have
  \begin{align}
    | F_h\^t(x) | \leq & {\lambda_1} \cdot \left( \frac{C_{\ref{cor:optimal-perimeter}}}{\epapx}\right)^{2A} \cdot \E_{v' \sim \MN(0, \Lambda)} [ 4 \| v'\|_2] + (2\sqrt{2\pi}\lambda_1\xi) \cdot \E_{w \sim \MN(0, \Sigma')} \left[ \max_{a \in \MA} \lng w, \phi_h(x,a) \rng \right]\nonumber\\
    \leq & {4 \lambda_1 \sqrt{d}} \cdot \left( \frac{C_{\ref{cor:optimal-perimeter}}}{\epapx} \right)^{2A} + 2\sqrt{2\pi} \lambda_1 \xi \cdot \sqrt{d} \leq \frac{\beta}{2C_{\ref{lem:phi-what-wt-diff}} HB \sqrt{d\iota}}\nonumber,
  \end{align}
  where the second inequality uses \cref{lem:quadratic-sim}, and the final inequality is derived as follows:
  \begin{align}
    {4 \lambda_1 \sqrt{d}} \cdot \left( \frac{C_{\ref{cor:optimal-perimeter}}}{\epapx} \right)^{2A} + (2\sqrt{2\pi} \lambda_1 \xi) \cdot \sqrt{d} \leq & {4 \lambda_1 \sqrt{d}} \cdot \left( \frac{C_{\ref{cor:optimal-perimeter}}}{\epapx} \right)^{2A} + \frac{\beta}{4 C_{\ref{lem:phi-what-wt-diff}} HB \sqrt{d\iota}}  \leq \frac{\beta}{2C_{\ref{lem:phi-what-wt-diff}} HB \sqrt{d\iota}}\nonumber,
  \end{align}
  where we have used the fact that $2\sqrt{2\pi} \lambda_1\sqrt{d} \cdot \xi \leq \frac{\beta}{4 C_{\ref{lem:phi-what-wt-diff}} HB \sqrt{d\iota}}$ and that $\frac{\beta}{4 C_{\ref{lem:phi-what-wt-diff}} HB \sqrt{d\iota}} \geq {4\lambda_1 \sqrt{d}} \cdot \left( \frac{C_{\ref{cor:optimal-perimeter}}}{\epapx} \right)^{2A}$ (by \cref{def:params}). 

  The definition of $w_h\^t$ (per \cref{eq:define-qht} and \cref{lem:wht-exist-bound}), combined with the fact that $|r_h(x,a)| \leq 1$ for all $h,x,a$ (\cref{asm:boundedness}), gives that $| \lng w_h\^t, \phi_h(x,a) \rng | \leq H + H \cdot \max_{h,x} \{ |F_h\^t(x)|\}$, from which it follows that $w_h\^t \in (H + H \cdot \max_{h,x} \{ |F_h\^t(x) | \}) \cdot \MB_h$. Since $\max_{w \in \MB_h} \| w \|_2 \leq B$ (again by \cref{asm:boundedness}), we have as a consequence that $\| w_h\^t \|_2 \leq \frac{\beta}{C_{\ref{lem:phi-what-wt-diff}}\sqrt{d\iota}}$. 
\end{proof}

\subsection{Linear regression lemma}
\label{sec:linreg}
We will need the following lemma which establishes a standard concentration statement regarding closeness of $\hat w_h\^t$ to $w_h\^t$.
\begin{lemma}
  \label{lem:phi-what-wt-diff}
There is a sufficiently large constant $C_{\ref{lem:phi-what-wt-diff}}$ (which is used in the definitions of $\epapx, \beta, \xi$ in \cref{def:params}) so that the following holds. There is an event $\ME$ which occurs with probability at least $1-\delta/2$ so that, for all $h \in [H], t \in [T], \phi \in \BR^d$, it holds that
  \begin{align}
| \lng \phi, \hat w_h\^t- w_h\^t \rng | \leq \beta \cdot \left(\phi^\t (\Sigma_h\^t)^{-1} \phi\right)^{1/2}\label{eq:qhat-qt-lemma}.
  \end{align}
\end{lemma}
\begin{proof}[Proof of \cref{lem:phi-what-wt-diff}]
Fix any $h \in [H], t \in [T]$.  Note that, by definition of $\Sigma_h\^t$ in \cref{line:define-sigmaht} of \cref{alg:psdp-ucb}, 
  \begin{align}
w_h\^t = (\Sigma_h\^t)^{-1} \Sigma_h\^t w_h\^t = (\Sigma_h\^t)^{-1} \cdot \left(\lambda w_h\^t +  \sum_{i=1}^n \phi_h(x_h\^{t,i,h}, a_h\^{t,i,h}) \cdot \lng \phi_h(x_h\^{t,i,h}, a_h\^{t,i,h}), w_h\^t \rng \right)\nonumber.
  \end{align}
  Also recall that the definition of $\hat w_h\^t$ in \cref{line:define-w-hat} gives
  \begin{align}
\hat w_h\^t =  (\Sigma_h\^{t})^{-1} \cdot  \sum_{i=1}^{n} \phi_h(x_h\^{t,i,h}, a_h\^{t,i,h}) \cdot \left(r_h\^{t,i,h} + \sum_{g=h+1}^H \hat r_g\^{t,i,h}\right)\nonumber.
  \end{align}
  Thus,
  \begin{align}
    w_h\^t - \hat w_h\^t =&  (\Sigma_h\^t)^{-1} \cdot \left( \lambda w_h\^t +  \sum_{i=1}^n \phi_h(x_h\^{t,i,h}, a_h\^{t,i,h})  \right.\nonumber\\
    & \quad \cdot \left( \lng \phi_h(x_h\^{t,i,h}, a_h\^{t,i,h}), w_h\^t \rng  \left. - r_h\^{t,i,h} - \sum_{g=h+1}^H \hat r_g\^{t,i,h} \right) \right)\label{eq:wt-hatwt-diff}.
  \end{align}
  For each $i \in [n]$, we have, by \cref{eq:define-qht}, that
  \begin{align}
\E^{\hat \pi\^t}\left[ r_h\^{t,i,h} + \sum_{g=h+1}^H \hat r_g\^{t,i,h} \ \mid \ (x_h, a_h) = (x_h\^{t,i,h}, a_h\^{t,i,h}) \right] = Q_h\^t(x_h\^{t,i,h}, a_h\^{t,i,h}) = \lng \phi_h(x_h\^{t,i,h}, a_h\^{t,i,h}), w_h\^t \rng\nonumber.
  \end{align}
  (In particular, note that by definition of the sampling procedure in \cref{line:collect-tih-samples}, the trajectory $(x_h, a_h,\ldots, x_H, a_H)$ is indeed drawn from the MDP with policy $\hat \pi\^t$.)

  Let us write $Y := \frac{\beta}{C_{\ref{lem:phi-what-wt-diff}}\sqrt{d\iota}}$. The definition of $\beta$ in \cref{def:params} together with \cref{lem:bonus-bound} ensures that $Y \geq 2H \cdot \max\left\{ 1, \max_{h,x} | F_h\^t(x)|\right\} \geq \left| r_h\^{t,i,h} + \sum_{g=h+1}^H \hat r_g\^{t,i,h} \right|$ with probability 1. We also have from \cref{lem:bonus-bound} that $Y \geq \| w_h\^t \|_2$ with probability 1.

  Thus, by \cref{lem:conc-sn} with $\Gamma_0 = \lambda I$, $\phi_i = \phi_h(x_h\^{t,i,h}, a_h\^{t,i,h})$ and $\vep_i = \lng \phi_h(x_h\^{t,i,h}, a_h\^{t,i,h}), w_h\^t \rng - r_h\^{t,i,h} - \sum_{g=h+1}^H \hat r_g\^{t,i,h}$ (so that we may take $\sigma = 2Y$), under some event $\ME\^{t,h}$ which holds with probability at least $1-\delta/(2TH)$, we have
  \begin{align}
    & \left\| \sum_{i=1}^n \phi_h(x_h\^{t,i,h}, a_h\^{t,i,h}) \cdot \left( \lng \phi_h(x_h\^{t,i,h}, a_h\^{t,i,h}), w_h\^t \rng  - r_h\^{t,i,h} - \sum_{g=h+1}^H \hat r_g\^{t,i,h}\right) \right\|_{(\Sigma_h\^t)^{-1}} \nonumber\\
    \leq & 8Y\cdot \sqrt{\log \left( \frac{TH \cdot \det(\Sigma_h\^t)^{1/2} \det(\lambda I )^{-1/2}}{\delta/2} \right)}\nonumber\\
    \leq & 8Y\sqrt{d}  \cdot \sqrt{ \log \left( TH \cdot \frac{\lambda d + n}{\lambda d \cdot \delta/2 } \right)}\label{eq:omega-bound},
  \end{align}
  where the final inequality follows since
  \begin{align}
\det(\Sigma_h\^t) \leq \left( \frac 1d \Tr \Sigma_h\^t \right)^{d} \leq \left( \frac 1d \cdot \left(\lambda d + \sum_{i=1}^n \| \phi_h(x_h\^{t,i,h}, a_h\^{t,i,h}) \|_2^2 \right) \right)^d\leq \left( \frac{\lambda d + n}{d} \right)^d\nonumber.
  \end{align}
  Let us define $\omega_h\^t := w_h\^t - \hat w_h\^t - (\Sigma_h\^t)^{-1} \cdot \lambda w_h\^t$, so that \cref{eq:wt-hatwt-diff} and \cref{eq:omega-bound} give that, under $\ME\^{t,h}$,
  \begin{align}
    \label{eq:true-omega-bound}
    \| (\Sigma_h\^t)^{1/2} \omega_h\^t \|_2 \leq 8Y \sqrt{d}  \cdot \sqrt{\log \left( TH \cdot  \frac{\lambda d + n}{\lambda d \delta/2}\right)}.
  \end{align}
  
  It follows that, under $\ME\^{t,h}$, for all $\phi \in \BR^d$, 
  \begin{align}
| \lng \hat w_h\^t - w_h\^t, \phi \rng |
  =&   | \lambda \lng (\Sigma_h\^t)^{-1} w_h\^t, \phi \rng| + | \lng \omega_h\^t, \phi \rng | \nonumber\\
    =& | \lambda \lng (\Sigma_h\^t)^{-1} w_h\^t, \phi \rng| + | \lng (\Sigma_h\^t)^{1/2} \cdot \omega_h\^t, (\Sigma_h\^t)^{-1/2} \cdot \phi \rng |\nonumber\\
    \leq &  \left(\sqrt{\lambda}\cdot  Y + 8Y \sqrt{d} \cdot \sqrt{\log \left(TH \cdot  \frac{\lambda d + n}{\lambda d \delta/2}\right)} \right)\cdot \sqrt{\phi^\t \cdot (\Sigma_h\^t)^{-1} \cdot \phi}\label{eq:qthat-qt-nearly},
  \end{align}
  where the final inequality uses \cref{eq:true-omega-bound}, Cauchy-Schwarz, and
  \begin{align}
    | \lambda \lng (\Sigma_h\^t)^{-1} w_h\^t, \phi \rng| =&  | \lng w_h\^t, \lambda (\Sigma_h\^t)^{-1} \phi \rng | \nonumber\\
    \leq & \| w_h\^t\|_2 \cdot \sqrt{ \phi^\t \cdot \lambda^2 (\Sigma_h\^t)^{-2} \cdot \phi }\nonumber\\
    \leq & \sqrt{\lambda}\cdot  Y \cdot  \sqrt{\phi^\t \cdot (\Sigma_h\^t)^{-1} \cdot \phi}\nonumber,
  \end{align}
  where the final inequality above uses the fact that $\lambda \cdot (\Sigma_h\^t)^{-1} \preceq I_d$. Finally, we note that \cref{eq:qthat-qt-nearly} gives that \cref{eq:qhat-qt-lemma} holds under $\ME\^{h,t}$ since we have %
  \begin{align}\beta = Y \cdot C_{\ref{lem:phi-what-wt-diff}} \sqrt{d\iota} \geq &  Y \cdot \left( \sqrt{C_{\ref{lem:psd-concentration}} d \log(2THn/\delta)} + 8 \sqrt{d \log \left( \frac{TH(\lambda d + n)}{\lambda d \delta/2} \right)} \right)\nonumber\\
  =&  Y \cdot \left( \sqrt{\lambda} + 8\sqrt{d} \sqrt{\log \left( \frac{TH(\lambda d + n)}{\lambda d \delta/2} \right)}\right),
\end{align}
  where the above follows by the choices of $\lambda = C_{\ref{lem:psd-concentration}} d \log(2THn/\delta)$, $\iota =\log \left( \frac{TH(\lambda d + n)}{ \delta/2}\right) \geq \log(2THn/\delta)$, and as long as  $C_{\ref{lem:phi-what-wt-diff}} $ is chosen sufficiently large. 
  The conclusion of the lemma follows by taking $\ME = \bigcap_{h \in [H], t \in [T]} \ME\^{h,t}$ and noting that by a union bound, $\ME$ occurs with probability at least $1-\delta/2$. 
\end{proof}

\subsection{Establishing optimism}
\label{sec:optimism-proof}
In this section, we establish that that $Q_h\^t(\cdot)$ is (approximately) an upper bound on $Q_h^\st(\cdot)$, in \cref{lem:approx-optimism} below.

First, we prove \cref{lem:opposite-skew} below, which shows that for states $x$ which satisfy a certain inequality (namely, \cref{eq:ratio-assumption}), $V_h^\pi(x)$ is approximately upper bounded by $V_h\^t(x)$ for any linear policy $\pi$. It is used to deal with one particular case in the proof of \cref{lem:approx-optimism}. Recall that for any $\pi \in \Pilin$, \cref{cor:qlin} guarantees the existence of a vector $w_h^\pi \in \BR^d$ so that for all $(x,a) \in \MX \times \MA$, $Q_h^\pi(x,a) = \lng w_h^\pi, \phi_h(x,a) \rng$. 
\begin{lemma}
  \label{lem:opposite-skew}
  Consider an orthogonal pair of PSD matrices $(\Sigma, \Lambda)$, and fix $x \in \MX$. For some $\xi \geq 1$, suppose that
  \begin{align}
\max_{a_1, a_2 \in \MA} \| \Lambda \cdot (\phi_h(x,a_1) - \phi_h(x,a_2)) \|_2 \leq \xi \cdot \max_{a_1, a_2 \in \MA} \| \Sigma \cdot (\phi_h(x,a_1) - \phi_h(x,a_2)) \|_2\label{eq:ratio-assumption}.
  \end{align}
  Then for any $\pi \in \Pilin$, 
  \begin{align}
V_h\^t(x) + 2\xi \| w_h^\pi \|_2 \cdot \max_{a_1, a_2 \in \MA} \| \Sigma \cdot (\phi_h(x, a_1) - \phi_h(x, a_2)) \|_2 \geq V_h^\pi(x) + \left( Q_h\^t(x, \hat \pi_h\^t(x)) - Q_h^\pi(x, \hat \pi_h\^t(x)) \right)\nonumber.
  \end{align}
\end{lemma}
\begin{proof}
  Note that, for any actions $a,a' \in \MA$, we have
  \begin{align}
    \| \phi_h(x,a) - \phi_h(x,a') \|_2 \leq &  \| \Lambda \cdot (\phi_h(x,a) - \phi_h(x,a')) \|_2 + \| \Sigma \cdot (\phi_h(x,a) - \phi_h(x,a')) \|_2\nonumber\\
    \leq &2\xi \cdot \max_{a,a' \in \MA} \| \Sigma \cdot (\phi_h(x,a) - \phi_h(x,a')) \|_2\label{eq:bound-1pxi},
  \end{align}
  where the second inequality uses \cref{eq:ratio-assumption} and the fact that $\xi \geq 1$. 
  
  We compute
  \begin{align}
    & V_h\^t(x) - V_h^\pi(x) \nonumber\\
    =& Q_h\^t(x, \hat \pi_h\^t(x)) - Q_h^\pi(x, \pi_h(x)) \nonumber\\
    =& \lng w_h^\pi, \phi_h(x, \hat \pi_h\^t(x)) - \phi_h(x, \pi_h(x))\rng + \left( Q_h\^t(x, \hat \pi_h\^t(x)) - Q_h^\pi(x, \hat \pi_h\^t(x)) \right)\nonumber\\
    \geq & - \| w_h^\pi \|_2 \cdot \| \phi_h(x, \hat \pi_h\^t(x)) - \phi_h(x, \pi_h(x)) \|_2 + \left( Q_h\^t(x, \hat \pi_h\^t(x)) - Q_h^\pi(x, \hat \pi_h\^t(x)) \right)\nonumber\\
    \geq & - 2\xi \| w_h^\pi\|_2 \cdot \max_{a,a' \in \MA} \| \Sigma \cdot (\phi_h(x,a) - \phi_h(x,a')) \|_2 + \left( Q_h\^t(x, \hat \pi_h\^t(x)) - Q_h^\pi(x, \hat \pi_h\^t(x)) \right)\nonumber,
  \end{align}
  where the final inequality uses \cref{eq:bound-1pxi}. 
\end{proof}

\cref{lem:approx-optimism} establishes that the functions $Q_h\^t$ are an approximate upper bound on $Q_h^\st$. %
\begin{lemma}[Optimism]
  \label{lem:approx-optimism}
  Under the event $\ME$ of \cref{lem:phi-what-wt-diff}, it holds that, for all $x,a,h,t$, %
  \begin{align}
    \epbell \cdot (H-h) + Q_h\^t(x,a) \geq & Q_h^\st(x,a) \label{eq:qh-nearly-optimistic}\\
    \epbell \cdot (H+1-h) + V_h\^t(x) + F_h\^t(x) \geq & V_h^\st(x) \label{eq:vh-nearly-optimistic}. 
  \end{align}
\end{lemma}
\begin{proof}
Condition on the event $\ME$, and fix $t \in [T]$.  We prove the statement using reverse induction on $h$. The base case $h=H+1$ is immediate since $Q_{H+1}\^t(x,a) = Q_{H+1}^\st(x,a)=V_{H+1}\^t(x) = V_{H+1}^\st(x) = 0$ for all $x,a,t$.

  Now suppose that the statement of the lemma holds at step $h+1$. 
  For all $x,a$, we have
  \begin{align}
    Q_h\^t(x,a) =&  r_h(x,a) + \E_{x' \sim \BP_h(x,a)}\left[  F_{h+1}\^t(x')  + V_{h+1}\^t(x') \right]\label{eq:qht-vht-2}\\
    Q_h^\st(x,a) =&  r_h(x,a) + \E_{x' \sim \BP_h(x,a)} \left[ V_{h+1}^\st(x') \right]\label{eq:qhpi-vhpi-2},
  \end{align}
  where \cref{eq:qht-vht-2} uses the definition of $Q_h\^t, V_h\^t$ in \cref{eq:define-qht,eq:define-vht}. 
  The inductive hypothesis gives that
  \begin{align}
\epbell \cdot (H+1-h-1) + V_{h+1}\^t(x') + F_{h+1}\^t(x') \geq V_{h+1}^\st(x') \label{eq:ucb-ih-2}
  \end{align}
  for all $x' \in \MX$. Combining \cref{eq:qht-vht-2,eq:qhpi-vhpi-2}, we see that %
  \begin{align}
     Q_h\^t(x,a) - Q_h^\st(x,a) 
    =& \E_{x' \sim \BP_h(x,a)} \left[  F_{h+1}\^t(x')  + V_{h+1}\^t(x')- V_{h+1}^\st(x') \right] \geq -\epbell \cdot (H-h)\label{eq:qht-qhstar-ind-nearly},
  \end{align}
  where the inequality uses \cref{eq:ucb-ih-2}. This verifies \cref{eq:qh-nearly-optimistic} at step $h$. %

  To establish \cref{eq:vh-nearly-optimistic} at step $h$, set  $(\Sigma', \Lambda') := \trunc{(\beta/\lambda_1) \cdot (\Sigma_h\^t)^{-1/2}}{\sigtr}$ (per \cref{def:mat-truncation}), so that $(\Sigma', \Lambda')$ is an orthogonal pair (note that this is the same choice as is made in the definition of $F_h\^t(\cdot)$ in \cref{eq:define-fht-bonus}). 

Let $\xi$ be defined as per \cref{def:params}, so that $\xi \geq 1$.  Now consider any $x \in \MX$ and $\pi \in \Pi$. We consider two cases based on the value of $x$.

  \paragraph{Case 1.} In the first case, we assume that
  \begin{align}
\max_{a,a' \in \MA} \| \Lambda' \cdot  (\phi_h(x,a) - \phi_h(x,a')) \|_2 \leq \xi \cdot \max_{a,a' \in \MA} \| \Sigma' \cdot (\phi_h(x,a) - \phi_h(x,a')) \|_2 \label{eq:case1}.
  \end{align}
  Then \cref{lem:opposite-skew} with $(\Sigma, \Lambda) = (\Sigma', \Lambda')$ and $\pi = \pi^\st$ gives that
  \begin{align}
    V_h\^t(x) + 2\xi \| w_h^\st \|_2 \cdot \max_{a,a' \in \MA} \| \Sigma' \cdot (\phi_h(x,a) - \phi_h(x,a')) \|_2 \geq &  V_h^\st(x) + \left( Q_h\^t(x, \hat \pi_h\^t(x)) - Q_h^\st(x, \hat \pi_h\^t(x)) \right)\nonumber\\
    \geq & V_h^\st(x) - \epbell \cdot (H-h)\nonumber,
  \end{align}
  where the second inequality uses \cref{eq:qht-qhstar-ind-nearly}. We may compute
  \begin{align}
    2\xi \| w_h^\st \|_2 \cdot \max_{a,a' \in \MA} \| \Sigma' \cdot (\phi_h(x,a) - \phi_h(x,a')) \|_2 \leq & 2\xi \lambda_1 \cdot \sqrt{2\pi} \E_{w \sim \MN(0, \Sigma')} \left[ \max_{a \in \MA} \lng w, \phi_h(x,a) \rng \right] 
    \leq  F_h\^t(x)\nonumber,
  \end{align}
  where the first inequality uses \cref{lem:quadratic-sim,cor:qlin} as well as the fact that $(\Sigma')^2 = \Sigma'$, and the second inequality uses the definition of $F_h\^t(\cdot)$ in \cref{eq:define-fht-bonus}.
It then follows that \cref{eq:vh-nearly-optimistic} holds at step $h$.

  \paragraph{Case 2.} In the second case, we assume that \cref{eq:case1} does not hold, which implies that
  \begin{align}
    & \max \left\{ \xi \cdot \max_{a,a' \in \MA} \| \Sigma' \cdot (\phi_h(x,a) - \phi_h(x,a')) \|_2, \max_{a,a' \in \MA} \| \Lambda' \cdot (\phi_h(x,a) - \phi_h(x,a')) \|_2 \right\} \nonumber\\
    \leq &\max_{a,a' \in \MA} \| \Lambda' \cdot (\phi_h(x,a) - \phi_h(x,a')) \|_2 \leq  \max_{a,a' \in \MA} \| \phi_h(x,a) - \phi_h(x,a') \|_2 \leq 2.\label{eq:xisigma-small}
  \end{align}
  We will apply \cref{cor:optimal-perimeter}  with $(\Sigma, \Lambda) = (\Sigma', \Lambda')$, the value of $\beta$ set to the present value of $\beta$ (defined in \cref{def:params}), $\vep = \epapx$, and $\Phi = \cPhi_h(x)$. We need to check that \cref{eq:skew-bound} holds; indeed, we have
  \begin{align}
    & \max \left\{ \beta \cdot \max_{a,a' \in \MA} \| \Sigma' \cdot (\phi_h(x,a) - \phi_h(x,a')) \|_2, \max_{a,a' \in \MA} \| \Lambda' \cdot (\phi_h(x,a) - \phi_h(x,a')) \|_2 \right\}\nonumber\\
    \leq & 8\sqrt{2\pi} C_{\ref{lem:phi-what-wt-diff}} \lambda_1 HB d\sqrt{\iota} \cdot \max \left\{ \xi \cdot \max_{a,a' \in \MA} \| \Sigma' \cdot (\phi_h(x,a) - \phi_h(x,a')) \|_2, \max_{a,a' \in \MA} \| \Lambda' \cdot (\phi_h(x,a) - \phi_h(x,a')) \|_2 \right\}\nonumber\nonumber\\
    \leq & 16\sqrt{2\pi} C_{\ref{lem:phi-what-wt-diff}} \lambda_1 HB d\sqrt{\iota}\nonumber,
  \end{align}
  where the first inequality uses the definition of $\xi$ in \cref{def:params} and the second inequality uses \cref{eq:xisigma-small}. Thus we may apply \cref{cor:optimal-perimeter} with $\zeta = 16\sqrt{2\pi} C_{\ref{lem:phi-what-wt-diff}} \lambda_1 HB d\sqrt{\iota}$. For $a,a' \in \MA$, let us write $\midpoint[x,h,a,a']$ as shorthand for $\midpoint[\cPhi_h(x), \phi_h(x,a), \phi_h(x,a')]$, which was defined in \cref{eq:define-midpoint}.  Fix any $a,a' \in \MA$; then applying \cref{cor:optimal-perimeter} with $\phi = \phi_h(x,a), \phi' = \phi_h(x,a')$, we have %
  \begin{align}
    \frac{1}{\lambda_1} \cdot F_h\^t(x) \geq     &  \left( \frac{C_{\ref{cor:optimal-perimeter}}}{\epapx} \right)^{2A} \cdot   \E_{u' \sim \MN(0, \Sigma')} \E_{v' \sim \MN(0, \Lambda')} \left[ \Ftp_h(x; \beta u', v') \right] \nonumber\\
        & + \sqrt{2\pi} \cdot \E_{w \sim \MN(0, \Sigma')} \left[ \max_{\bar a \in \MA} \lng w, \phi_h(x, \bar a) \rng \right]\nonumber\\
    \geq &  \| \beta \Sigma' \cdot (\phi_h(x,a) - \midpoint[x,h,a,a'] ) \|_2 + \| \Lambda' \cdot (\midpoint[x,h,a,a'] - \phi_h(x,a')) \|_2 - 64\sqrt{2\pi}C_{\ref{lem:phi-what-wt-diff}}  \lambda_1 HB d\sqrt{\iota}\cdot \epapx \nonumber\\
        & + \sqrt{2\pi} \cdot \E_{w \sim \MN(0, \Sigma')} \left[ \max_{\bar a \in \MA} \lng w, \phi_h(x, \bar a) \rng \right]\nonumber\\
    \geq &  \| \beta \Sigma' \cdot (\phi_h(x,a) - \midpoint[x,h,a,a'] ) \|_2 + \| \Lambda' \cdot (\midpoint[x,h,a,a'] - \phi_h(x,a')) \|_2 - \epbell/(2\lambda_1) \nonumber\\
        & + \sqrt{2\pi} \cdot \E_{w \sim \MN(0, \Sigma')} \left[ \max_{\bar a \in \MA} \lng w, \phi_h(x, \bar a) \rng \right]\label{eq:fht-truncate-lb},
  \end{align}
  where the first inequality follows from the definition of $F_h\^t(x)$ in \cref{eq:define-fht-bonus} and the fact that $\xi \geq 1$, the second inequality uses \cref{cor:optimal-perimeter} with the above parameter settings, and the final inequality uses the definition of $\epapx$ in \cref{def:params}. 

  Write $\mu^\st := \midpoint[x, h,\hat \pi_h\^t(x), \pi_h^\st(x)]$. We now apply \cref{lem:phi-what-wt-diff} with $\phi = \phi_h(x, \hat \pi_h\^t(x)) - \mu^\st$. %
Then \cref{lem:phi-what-wt-diff} ensures that 
  \begin{align}
    \left| \lng w_h\^t - \hat w_h\^t, \phi_h(x, \hat \pi_h\^t(x)) - \mu^\st \rng \right| \leq & \beta \cdot \left\| (\Sigma_h\^t)^{-1/2} \cdot (\phi_h(x, \hat \pi_h\^t(x) - \mu^\st) \right\|_2 \label{eq:vhstar-perturb-error}.
  \end{align}
  Moreover, since $\mu^\st \in \bar \Phi_h(x)$, we have from the fact that \cref{eq:qh-nearly-optimistic} holds at step $h$ that
  \begin{align}
\lng w_h\^t - w_h^\st, \mu^\st \rng \geq & -\epbell \cdot (H-h)\label{eq:vhstar-ucb}.
  \end{align}
  Finally, we have
  \begin{align}
\lng w_h^\st, \mu^\st - \phi_h(x, \pi_h^\st(x)) \rng \geq & - \| w_h^\st \|_2 \cdot \| \mu^\st - \phi_h(x, \pi_h^\st(x)) \|_2 \label{eq:vhstar-whstar}.
  \end{align}
  We may now compute
  \begin{align}
     V_h\^t(x) - V_h^\st(x) =& Q_h\^t(x, \hat \pi_h\^t(x)) - Q_h^\st(x, \pi_h^\st(x)) \nonumber\\
    \geq & Q_h\^t(x, \hat \pi_h\^t(x)) - Q_h^\st(x, \pi_h^\st(x)) +\lng \hat w_h\^t, \mu^\st - \phi_h(x, \hat \pi_h\^t(x)) \rng  \nonumber\\
    = & \lng w_h\^t - w_h^\st, \mu^\st \rng + \lng \hat w_h\^t - w_h\^t, \mu^\st - \phi_h(x, \hat \pi_h\^t(x)) \rng + \lng w_h^\st, \mu^\st - \phi_h(x, \pi_h^\st(x)) \rng \nonumber\\
    \geq & -\epbell \cdot (H-h) - \beta \cdot  \left\| (\Sigma_h\^t)^{-1/2} \cdot (\phi_h(x, \hat \pi_h\^t(x) - \mu^\st) \right\|_2  - \| w_h^\st \|_2 \cdot \| \mu^\st - \phi_h(x, \pi_h^\st(x)) \|_2 \label{eq:vht-vhstar-compute},
  \end{align}
  where the first inequality uses that $\hat \pi_h\^t(x) = \argmax_{a' \in \MA} \hat Q_h\^t(x, a')$, and the second inequality uses \cref{eq:vhstar-perturb-error,eq:vhstar-ucb,eq:vhstar-whstar}. Moreover, %
  \begin{align}
    & \beta \cdot  \left\| (\Sigma_h\^t)^{-1/2} \cdot (\phi_h(x, \hat \pi_h\^t(x) - \mu^\st) \right\|_2  + \| w_h^\st \|_2 \cdot \| \mu^\st - \phi_h(x, \pi_h^\st(x)) \|_2\nonumber\\
    \leq & \beta \cdot  \left\| (\Sigma_h\^t)^{-1/2} \cdot (\phi_h(x, \hat \pi_h\^t(x) - \mu^\st) \right\|_2  + \lambda_1 \cdot \| \mu^\st - \phi_h(x, \pi_h^\st(x)) \|_2\nonumber\\
    \leq & \lambda_1 \cdot \left(\| (\beta/\lambda_1)\cdot  \Sigma' \cdot (\phi_h(x, \hat \pi_h\^t(x)) - \mu^\st) \|_2 +   \| \Lambda' \cdot (\mu^\st - \phi_h(x, \pi_h^\st(x))) \|_2\right.\nonumber\\
    &\left.+ \sqrt{2\pi} \cdot \E_{w \sim \MN(0, \Sigma')} \left[ \max_{a \in \MA} \lng w, \phi_h(x,a) \rng \right] + 2 \sigtr \right)\label{eq:apply-truncation},
  \end{align}
  where the first inequality uses that $\| w_h^\st \|_2 \leq \lambda_1$ (\cref{cor:qlin}) and the second inequality uses \cref{lem:bound-truncation-error} with $\Gamma = (\beta/\lambda) \cdot  (\Sigma_h\^t)^{-1/2}$ and $\sigma = \sigtr$, as well as the fact that $\| (\Sigma_h\^t)^{-1/2} \| \leq 1$ (which follows by the choice of $\lambda = 1$ in \cref{def:params}), so that $\| \Gamma \| \leq \beta/\lambda_1$. 
  
  Combining \cref{eq:vht-vhstar-compute,eq:apply-truncation,eq:fht-truncate-lb} and the fact that $\lambda_1 \geq 1$, it follows that
  \begin{align}
    & V_h\^t(x) - V_h^\st(x) \nonumber\\
    \geq & -\epbell \cdot (H-h) -  \| \beta \Sigma' \cdot (\phi_h(x, \hat \pi_h\^t(x)) - \mu^\st) \|_2 - \lambda_1 \cdot \| \Lambda' \cdot (\mu^\st - \phi_h(x, \pi_h^\st(x))) \|_2   \nonumber\\
    & - \sqrt{2\pi} \lambda_1 \cdot \E_{w \sim \MN(0, \Sigma')} \left[ \max_{a \in \MA} \lng w, \phi_h(x,a) \rng \right]- 2\lambda_1\sigtr \nonumber\\
    \geq & -\epbell \cdot (H-h)  - F_h\^t(x) - \epbell/2 - 2\lambda_1\sigtr\label{eq:vht-vhstar-m1}\\
    \geq &  -\epbell \cdot (H+1-h)  - F_h\^t(x)\label{eq:vht-vhstar-completion},
  \end{align}
  where the final inequality follows since $\epbell/2+ 2\lambda_1\sigtr \leq \epbell$ by the choice of $\sigtr$ in \cref{def:params}.  This verifies that \cref{eq:vh-nearly-optimistic} holds at step $h$, completing the proof of the lemma. 
\end{proof}
\begin{theorem}
  \label{thm:policy-learning}
  For any $\epfinal, \delta \in (0,1)$, the policy 
  $\hat \pi := \frac 1T \sum_{t=1}^T \hat \pi\^t$ produced by \cref{alg:psdp-ucb} has suboptimality bounded by $V_1^\st(x_1) - V_1^{\hat \pi}(x_1) \leq \epfinal$ after the algorithm has gathered $$ O\left(Hd^2 \cdot\left( \frac{H^4 B^3 dA^{1/2} \log^{1/2}(HABd/(\epfinal\delta))}{\epfinal} \right)^{12A + 4}\right)$$ samples and used $(HBdA\log(1/(\epfinal \delta)))^{O(A)}$ time. 
\end{theorem}
\begin{proof}
  Using the fact that $V_1^{\hat \pi\^t}(x_1) = \E^{\hat \pi\^t} \left[ \sum_{h=1}^H r_h(x_h, a_h) \right]$ as well as the definition of $V_h\^t$ in \cref{eq:define-qht,eq:define-vht}, we see that
  \begin{align}
V_1\^t(x_1) - V_1^{\hat \pi\^t}(x_1) =& \sum_{h=2}^H \E^{\hat \pi\^t} \left[F_h\^t( x_h)\right]\label{eq:qht-qhpit-full}.
  \end{align}
Set $\gamma = \sqrt{d} \beta \lambda_1 \cdot \left( \frac{C_{\ref{lem:sigmap-bound}}}{\epapx} \right)^{2A}$. For each $h \in [H], t \in [T]$, let us write $((\Sigma_h')\^t, (\Lambda_h')\^t) := \trunc{(\beta/\lambda_1) \cdot (\Sigma_h\^t)^{-1/2}}{\sigtr}$. Then under the event $\ME$ of \cref{lem:approx-optimism} (which satisfies $\Pr(\ME) \geq 1-\delta/2$),
  \begin{align}
     V_1^\st(x_1) - V_1^{\hat \pi\^t}(x_1)
    \leq &H \cdot \epbell +  V_1\^t(x_1) + F_1\^t( x_1) - V_1^{\hat \pi\^t}(x_1)\nonumber\\
    = &H \cdot \epbell + \sum_{h=1}^H \E^{\hat \pi\^t}\left[ F_h\^t( x_h) \right]\nonumber\\
    \leq &  H \cdot \epbell + \gamma \sum_{h=1}^H \E^{\hat \pi\^t} \left[ \E_{w \sim \MN(0, (\Sigma_h')\^t)} \left[ \max_{a \in \MA} \lng w, \phi_h(x_h, a) \rng \right] \right] \nonumber\\
    \leq & H \cdot \epbell + \gamma\sum_{h=1}^H \E^{\hat \pi\^t} \left[\sqrt{d} \cdot \E_{w \sim \MN(0, (\Sigma_h')\^t)} \left[ \phi_h(x_h, \pi_{h,w}(x_h))^\t (\Sigma_h')\^t \phi_h(x_h, \pi_{h,w}(x_h)) \right]^{1/2}\right]\nonumber\\
    \leq & H \cdot \epbell + \gamma \sqrt{d} \sum_{h=1}^H \E^{\hat \pi\^t \circ_h \tilde \pi\^t} \left[  \phi_h(x_h, a_h)^\t (\Sigma_h')\^t \phi_h(x_h, a_h) \right]^{1/2}\nonumber\\
    \leq & H \cdot \epbell + \gamma \sqrt{d} \cdot \frac{\beta}{\sigtr \lambda_1} \sum_{h=1}^H \E^{\hat \pi\^t \circ_h \tilde \pi\^t} \left[  \phi_h(x_h, a_h)^\t (\Sigma_h\^t)^{-1} \phi_h(x_h, a_h) \right]^{1/2}\label{eq:regret-bound-full},
  \end{align}
  where the first inequality uses \cref{lem:approx-optimism}, the second inequality uses \cref{lem:sigmap-bound} together with the definition of $\gamma$, the third inequality uses \cref{lem:quadratic-sim}, 
  the fourth inequality uses Jensen's inequality as well as the definition of $\tilde \pi_h\^t$ in \cref{line:define-pi-tilde} of \cref{alg:psdp-ucb}, and the final inequality uses the fact that $((\Sigma_h')\^t, (\Lambda_h')\^t) = \trunc{\frac{\beta}{\lambda_1} \cdot (\Sigma_h\^t)^{-1/2}}{\sigtr}$ together with \cref{lem:orig-truncated}, which implies that $((\Sigma_h')\^t)^2 = (\Sigma_h')\^t \preceq \frac{\beta}{\sigtr \lambda_1} \cdot (\Sigma_h\^t)^{-1/2}$. 

    For each $h \in [H], t \in [T]$, let us define
  \begin{align}
\Gamma_h\^t := \E^{\hat \pi\^t \circ_h \tilde \pi\^t}\left[ \phi_h(x_h, a_h) \phi_h(x_h, a_h)^\t \right]\nonumber.
  \end{align}

  Next, for each $t \in [T]$ and $h \in [H]$, we will now apply \cref{lem:psd-concentration} with $P$ given by the distribution of $\phi_h(x_h, a_h)$ for $(x_h, a_h) \sim \BP^{\unif(\{ \hat \pi^s \circ_h \tilde \pi\^s \}_{s=1}^{t-1}\})}$. Write 
  \begin{align}
\Lambda_h\^t = \E^{\unif(\{ \hat \pi\^s \circ_h \tilde \pi\^s\})_{s=1}^{t-1}}[ \phi_h(x_h, a_h) \phi_h(x_h, a_h)^\t]\nonumber,
  \end{align}
  so that for each $i \in [n]$, $\E[\phi_h(x_h\^{t,i,h}, a_h\^{t,i,h}) \phi_h(x_h\^{t,i,h}, a_h\^{t,i,h})^\t] = \Lambda_h\^t$. 
  Then \cref{lem:psd-concentration} together with the definition of $\Sigma_h\^t$ in \cref{line:define-sigmaht} of \cref{alg:psdp-ucb} gives that, since  $\lambda \geq C_{\ref{lem:psd-concentration}} d \log(2THn/\delta)$ (per \cref{def:params}), %
  \begin{align}
\Pr\left( \frac 13 (\lambda I + n\Lambda_h\^t) \preceq \Sigma_h\^t \preceq \frac 53 (\lambda I + n\Lambda_h\^t) \right) \geq 1-\delta/(2TH)\label{eq:ht-psd-conc-full}.
  \end{align}
  Let $\ME'$ denote the event that \cref{eq:ht-psd-conc-full} holds for all $h \in [H], t \in [T]$, so that $\ME'$ occurs with probability at least $1-\delta/2$. Since $n \geq 3T$, we have that $\frac{n}{3}\Lambda_h\^t \succeq \sum_{s=1}^t \Gamma_h\^s$, and therefore, under $\ME'$, we have $\tilde \Sigma_h\^t := \frac{\lambda}{3} I_d + \sum_{s=1}^t \Gamma_h\^s \preceq \Sigma_h\^t$. Next, we may compute, for each $h \in [H]$,
  \begin{align}
    \sum_{t=1}^T \E^{\hat \pi\^t \circ_h \tilde \pi\^t}[\phi_h(x_h, a_h)^\t (\Sigma_h\^t)^{-1} \phi_h(x_h, a_h)] =& \sum_{t=1}^T \lng \Gamma_h\^t, (\Sigma_h\^t)^{-1} \rng \leq  \sum_{t=1}^T \lng \Gamma_h\^t, (\tilde \Sigma_h\^t)^{-1} \rng \leq  2d \log(2T)\label{eq:apply-epl-full},
  \end{align}
  where the final inequality uses \cref{lem:epl-gen} together with the fact that $\lambda/3 \geq 1$. Combining \cref{eq:regret-bound-full} and \cref{eq:apply-epl-full}, and using Jensen's inequality, we obtain that under $\ME \cap \ME'$ (which occurs with probability at least $1-\delta$), for sufficiently large constants $C, C'$,
  \begin{align}
    \sum_{t=1}^T (V_1^\st(x_1) - V_1^{\hat \pi\^t}(x_1)) \leq & TH\epbell +  \frac{H \gamma \beta \sqrt{d}}{\sigtr\lambda_1} \cdot \sqrt{2Td \log(2T)} \nonumber\\
    \leq&  TH\epbell + \frac{4H d\beta^2 \lambda_1}{\epbell} \cdot \left( \frac{C_{\ref{lem:sigmap-bound}}}{\epapx} \right)^{2A} \cdot \sqrt{2Td\log(2T)}\nonumber\\
    \leq & TH\epbell + \frac{C H^6 \Bbnd^5 d^3 \iota}{\epbell} \cdot \left(\frac{C_{\ref{cor:optimal-perimeter}}}{\epapx} \right)^{4A}  \cdot \left( \frac{C_{\ref{lem:sigmap-bound}}}{\epapx} \right)^{2A} \cdot \sqrt{2Td\log(2T)}\nonumber\\
    \leq & TH\epbell + \frac{C H^6 \Bbnd^5 d^3 \iota}{\epbell} \cdot \left( \frac{C' H^3 B^3 d \sqrt{\iota}}{\epbell} \right)^{6A} \nonumber\\
    \leq & T\epfinal/2 + CH^3 B^2 d^2 \sqrt{\iota} \left( \frac{2C' H^4 B^3 d \sqrt{\iota}}{\epfinal} \right)^{6A+1}\label{eq:final-regret},
  \end{align}
  where the second inequality uses the definition of $\gamma, \sigtr$, and the remaining inequalities use the definitions of $\epapx, \beta, \lambda_1,\epfinal$ in \cref{def:params}. Next, we have the following claim:
  \begin{claim}
    \label{clm:t-lb}
The choice of $T,\iota$ (in \cref{def:params}) ensures that  $T \geq d \cdot \left( \frac{C_{\ref{thm:policy-learning}}^{1/2} H^4 \Bbnd^3 d\sqrt \iota}{\epfinal} \right)^{6A+2}$.
  \end{claim}
  \cref{clm:t-lb} is proved at the end of the section. 
By \cref{clm:t-lb}, as long as the constant $C_{\ref{thm:policy-learning}}$ is chosen sufficiently large, the expression in \cref{eq:final-regret} is bounded above by $T\epfinal$. 
Hence the policy $\hat \pi := \frac 1T \sum_{t=1}^T \hat \pi\^t$ satisfies $V_1^\st(x_1) - V_1^{\hat \pi}(x_1) \leq \epfinal$. 

Altogether, \cref{alg:psdp-ucb} collects $THn \leq O(HT^2) \leq O\left(Hd^2 \cdot\left( \frac{H^4 B^3 dA^{1/2} \log^{1/2}(HABd/(\epfinal\delta))}{\epfinal} \right)^{12A + 4}\right)$ trajectories.

Finally, we analyze the computation time of \cref{alg:psdp-ucb}. It is straightforward to see that each step of the algorithm can be implemented in $\poly(T)$ time, with the exception of the computation of $F_g\^t(x_g\^{t,i,h})$ in \cref{eq:rhat-rewards}, which requires integrating a nonlinear function over a Gaussian (see \cref{eq:define-fht-bonus}). To handle this difficulty, we discuss a slight modification of our algorithm and its analysis which proceeds by approximating this integral. To efficiently approximate the result of each such integration, we use $O (\log(T/\delta)/\epapx^2)$ draws from the appropriate Gaussian to approximate each integral to within $\epapx$. The error term of $\epapx$ that arises in this approximation leads to an additional error term of $\epapx$ in \cref{eq:qhat-qt-lemma} (\cref{lem:phi-what-wt-diff}), which in turn leads to an additional term of $-\epapx$ in \cref{eq:vht-vhstar-m1}. By an appropriate choice of the constants in \cref{def:params}, we can ensure that $\epapx + \epbell/2 + 2\lambda_1\sigtr \leq \epbell$, which ensures that \cref{eq:vht-vhstar-completion} still holds. The remaining details of the proof remain unchanged. %
\end{proof}

\begin{proof}[Proof of \cref{clm:t-lb}]
  Note that, since $n = 3T$ and by the definition of $T$ in \cref{def:params}, for  sufficiently large constants $C, C'$ (independent of $C_{\ref{thm:policy-learning}}$), we have $\iota \leq C \cdot (\log(H\Bbnd d/\delta) + \log(T)) \leq C' \cdot A \log \left( \frac{C_{\ref{thm:policy-learning}}H\Bbnd d A }{\epfinal \delta} \right)$. Then 
  \begin{align}
d \cdot \left( \frac{C_{\ref{thm:policy-learning}}^{1/2} H^4 \Bbnd^3 d\sqrt \iota}{\epfinal} \right)^{6A+2} \leq d \cdot \left( \frac{\sqrt{C'} \log^{1/2} (C_{\ref{thm:policy-learning}}) \cdot C_{\ref{thm:policy-learning}}^{1/2} H^4 \Bbnd^3 d A^{1/2} \log^{1/2} ( H\Bbnd d A/(\epfinal \delta))}{\epfinal} \right)^{6A+2}\nonumber.
  \end{align}
  As long as $C_{\ref{thm:policy-learning}}$ is chosen sufficiently large so that $\sqrt{C'} \log^{1/2}(C_{\ref{thm:policy-learning}}) C_{\ref{thm:policy-learning}}^{1/2} \leq C_{\ref{thm:policy-learning}}$, the right-hand side of the above expression is bounded above by $T$, as desired.
\end{proof}

\section{Useful lemmas}

\subsection{Concentration}
\begin{lemma}[Concentration for self-normalized process; e.g., Theorem D.3 of \cite{jin2020provably}]
  \label{lem:conc-sn-martingale}
  Fix $n \in \BN$ and let $\vep_1, \ldots, \vep_n$ be random variables which are adapted to a filtration $(\MF_i)_{0 \leq i \leq n}$. Suppose that for each $i \in [n]$, $\E[\vep_i | \MF_{i-1}] = 0$ and $\E[e^{\lambda \vep_i} | \MF_{i-1}] \leq e^{\lambda^2 \sigma^2/2}$. Suppose that $\phi_1, \ldots, \phi_n$ is a sequence which is predictable with respect to $(\MF_i)_{0 \leq i \leq n}$, i.e., $\phi_i$ is measurable with respect to $\MF_{i-1}$ for all $i \in [n]$. Suppose that $\Gamma_0 \in \BR^{d \times d}$ is positive definite, and let $\Gamma_i = \Gamma_0 + \sum_{j=1}^i \phi_j \phi_j^\t$. Then for any $\delta > 0$, with probability at least $1-\delta$,
    \begin{align}
\left\| \sum_{i=1}^n \phi_i \vep_i \right\|_{\Gamma_i^{-1}}^2 \le 2 \sigma^2 \log \left( \frac{\det(\Gamma_t)^{1/2} \det (\Gamma_0)^{-1/2} }{\delta} \right)\nonumber. 
  \end{align}
\end{lemma}
In the special case that the random variables $\vep_i$ are i.i.d., we obtain the following:
\begin{corollary}[Concentration for self-normalized process; i.i.d.~data]
  \label{lem:conc-sn}
  Fix $n \in \BN$, and let $\vep_1, \ldots, \vep_n$ be independent real-valued random variables, so that for each $i \in [n]$, $\E[\vep_i] = 0$ and $\E[e^{\lambda \vep_i}] \leq e^{\lambda^2 \sigma^2 / 2}$ for all $\lambda \in \BR$.  Let $\phi_1, \ldots, \phi_n \in \BR^d$ be given. Suppose that $\Gamma_0 \in \BR^{d \times d}$ is positive definite, and let $\Gamma_i = \Gamma_0 + \sum_{j=1}^i \phi_j \phi_j^\t$. Then for any $\delta > 0$, with probability at least $1-\delta$,
  \begin{align}
\left\| \sum_{i=1}^n \phi_i \vep_i \right\|_{\Gamma_i^{-1}}^2 \le 2 \sigma^2 \log \left( \frac{\det(\Gamma_t)^{1/2} \det (\Gamma_0)^{-1/2} }{\delta} \right)\nonumber. 
  \end{align}
\end{corollary}

\begin{lemma}[Lemma 39 of \cite{zanette2021cautiously}]
  \label{lem:psd-concentration}
  There is a constant $C_{\ref{lem:psd-concentration}} > 0$ so that the following holds. Suppose $P$ is a distribution supported on the unit Euclidean ball in $\BR^d$. Write $\Sigma = \E_{\phi \sim P}[\phi\phi^\t]$. Suppose $n \in \BN, \delta > 0, \lambda$ are given so that $\lambda \geq C_{\ref{lem:psd-concentration}} d \log(n/\delta)$. Suppose $\phi_1, \ldots, \phi_n \sim P$ are i.i.d. Then
  \begin{align}
\Pr \left( \frac 13 (\lambda I + n\Sigma ) \preceq \lambda I + \sum_{i=1}^n \phi_i\phi_i^\t \preceq \frac 53 (\lambda I + n\Sigma) \right) \geq 1-\delta\nonumber.
  \end{align}
\end{lemma}

\subsection{Elliptical potential}
The below lemma generalizes the elliptical potential lemma.
\begin{lemma}
  \label{lem:epl-gen}
  Consider any sequence $\Gamma_1, \ldots, \Gamma_T \in \BR^{d \times d}$ of PSD matrices, and suppose that $\Tr \Gamma_t \leq 1$ for all $t \in [T]$. Define $\Sigma_t = \lambda I + \sum_{s=1}^t \Gamma_s$, for some $\lambda \geq 1$. Then
  \begin{align}
\sum_{t=1}^T \lng \Gamma_t, \Sigma_{t-1}^{-1} \rng \leq 2d\log(2T).\nonumber
  \end{align}
\end{lemma}
\begin{proof}
  Write $\Sigma_0 = \lambda I$. Since $\lambda \geq 1$, we have $\Sigma_{t-1}^{-1} \preceq I_d$ and therefore $\lng \Gamma_t, \Sigma_{t-1}^{-1} \rng \leq \Tr \Gamma_t \leq 1$ for all $t \in [T]$. Thus $\lng \Gamma_t, \Sigma_{t-1}^{-1} \rng \leq 2 \log (1 + \lng \Gamma_t, \Sigma_{t-1}^{-1} \rng)$. Moreover, for each $t \in [T]$, 
  \begin{align}
    \det(\Sigma_t) = & \det(\Sigma_{t-1} + \Gamma_t) \nonumber\\
    =& \det(\Sigma_{t-1}) \cdot \det\left(I + \Sigma_{t-1}^{-1/2} \Gamma_t \Sigma_{t-1}^{-1/2}\right)\nonumber\\
    \geq & \det(\Sigma_{t-1}) \cdot (1 + \Tr(\Sigma_{t-1}^{-1/2} \Gamma_t \Sigma_{t-1}^{-1/2})) = \det(\Sigma_{t-1}) \cdot (1 + \lng \Gamma_t, \Sigma_{t-1}^{-1}\rng)\nonumber,
  \end{align}
  where the inequality uses that for any real numbers $\sigma_1, \ldots, \sigma_d \geq 0$, $\prod_{i=1}^d (1+\sigma_i) \geq 1 + \sum_{i=1}^d \sigma_i$, which implies that $\det(I+A) \geq 1 + \Tr(A)$ for any PSD matrix $A$ (by taking $\sigma_1, \ldots, \sigma_d$ to be the eigenvalues of $A$). Telescoping the above display, we obtain that
 {\small \begin{align}
\frac 12 \sum_{t=1}^T \lng \Gamma_t, \Sigma_{t-1}^{-1} \rng \leq \sum_{t=1}^T \log(1 + \lng \Gamma_t, \Sigma_{t-1}^{-1} \rng) \leq \log \left( \frac{\det(\Sigma_T)}{\det(\lambda I)}\right) \leq \log \left( \frac{\left( \frac{1}{d} \Tr(\Sigma_T)\right)^d}{\det(\lambda I)} \right) \leq d \log \left( \frac{\lambda d + T}{\lambda d} \right)\leq d \log(2T)\nonumber.
  \end{align}}
\end{proof}

\subsection{Performance difference lemma}
\begin{lemma}[Performance difference lemma \cite{kakade2002approximately}]\label{lem:perf-diff}
For any MDP $M$, policies $\pi,\pi'\in\Pi$, it holds that
\[\E\sups{M,\pi}\left[\sum_{h=1}^H r_h(x_h,a_h)\right] - \E\sups{M,\pi'}\left[\sum_{h=1}^H r_h(x_h,a_h)\right] = \sum_{h=1}^H \E\sups{M,\pi'}\left[V\sups{M,\pi}_h(x_h) -  Q\sups{M,\pi}_h(x_h,a_h)\right].\]
\end{lemma}

\subsection{Linear algebraic lemmas}
\begin{lemma}
  \label{lem:sym-commute}
Two symmetric matrices $\Sigma, \Gamma \in \BR^{d \times d}$ commute if and only if their eigenspaces coincide.
\end{lemma}

\arxiv{
\section*{Acknowledgements}
We thank Dylan Foster and Sham Kakade for bringing this problem to our attention and for helpful discussions. 
}

\arxiv{\appendix}

\section{Proofs for \cref{sec:bellman-linear}}
\label{sec:bellman-linear-proofs}

\subsection{Functions which are not Bellman-linear}
\begin{proposition}
  \label{prop:lsvi-counterexample}
  There is an MDP $M = (H, \MX, \MA, (P_h)_h, (r_h)_h, d_1)$ with horizon $H=2$ satisfying linear Bellman completeness with respect to some feature mappings $\phi_h$ in dimension $d=1$, together with a vector $w^\st \in \BR^d$ so that the function
  \begin{align}
V_h(x) := \max_{a \in \MA} \min \left\{ \lng w^\st, \phi_h(x,a)\rng , H \right\}\nonumber
  \end{align}
  is not Bellman-linear at step $2$ (per \cref{def:bl}).
\end{proposition}
\begin{proof}
  We set $H = 2, d = 1$, $\MX = \{ \mfx_1, \mfx_{2,0}, \mfx_{2,1}\}$, and $\MA = \{0,1\}$, and define:
  \begin{align}
    & \phi_1(\mfx_1, 0) = 1, \phi_1(\mfx_1, 1) = 2,\nonumber\\
    & \phi_2(\mfx_{2,0}, 0) = H, \phi_2(\mfx_{2,0}, 1) = -H, \qquad \phi_2(\mfx_{2,1}, 0) = 2H, \phi_2(\mfx_{2,1}, 1) = -2H\nonumber.
  \end{align}
  The transitions are defined as follows: $(\mfx_1, 0)$ transitions to $\mfx_{2,0}$, and $(\mfx_1, 1)$ transition to $\mfx_{2,1}$. All rewards are 0. By defining $\MT_1 w := H \cdot |w|$, it is readily seen that $M$ is linear Bellman complete with respect to the above feature mappings.

  Let us choose $w^\st = 1$. Then Bellman-linearity of $V_2(x)$ would require that there is some $\bar w \in \BR$ so that
  \begin{align}
\phi_1(\mfx_1, 0) \cdot \bar w = H, \qquad \phi_1(\mfx_1, 1) \cdot \bar w = H\nonumber.
  \end{align}
  But since $\phi_1(\mfx_1,1) = 2$ and $\phi_1(\mfx_1, 0) = 1$, there is no such $\bar w$. 
\end{proof}
\cref{prop:lsvi-counterexample} provides an example for which a function $\hat Q_h\^t$ as defined in \cref{eq:lsvi-q}, with $B_h\^t \equiv 0$, is so that $\max_{a \in \MA} \hat Q_h\^t(\cdot, a)$ is not Bellman-linear. We have considered that case that $B_h\^t$ is 0 for simplicity. Note that in an instantiation of \texttt{LSVI-UCB}, $B_h\^t$ would not be identically 0, as it is defined by \cref{eq:quad-bonus}. Nevertheless, the example in \cref{prop:lsvi-counterexample} can be readily modified to account for nonzero quadratic bonus functions $B_h\^t$. In a similar manner, by rescaling the features we may modify the example to ensure that $\| \phi_h(x,a) \|_2 \leq 1$ for all $h,x,a$ (so as to satisfy \cref{asm:boundedness}). 

\begin{proposition}
  \label{prop:quadratic-ctex}
  There is an MDP $M = (H, \MX, \MA, (P_h)_h, (r_h)_h, d_1)$ with horizon $H=2$ satisfying linear Bellman completeness with respect to some feature mappings $\phi_h$ in dimension $d=1$, so that the function
  \begin{align}
F_h(x) := \max_{a \in \MA}  \| \phi_h(x,a) \|_2\nonumber
  \end{align}
  is not Bellman-linear at step 2.
\end{proposition}
\begin{proof}
  We set $H=2, d = 1, \MX = \{ \mfx_1, \mfx_{2,0}, \mfx_{2,1}, \mfx_{2,2} \}$, $\MA = \{0,1\}$, and define:
  \begin{align}
    & \phi_1(\mfx_1, 0) = 1, \phi_1(\mfx_1, 1) = 1, \nonumber\\
    & \phi_2(\mfx_{2,0}, 0) = 1, \phi_2(\mfx_{2,0}, 1) = -1, \quad \phi_2(\mfx_{2,1}, 0) = 2, \phi_2(\mfx_{2,1}, 1) = 0, \quad \phi_2(\mfx_{2,2}, 0) = -2, \phi_2(\mfx_{2,2}, 1) = 0\nonumber.
  \end{align}
  The transition are defined as follows: $(\mfx_1, 0)$ transitions to $\mfx_{2,0}$, and $(\mfx_1, 1)$ transitions to each of $\mfx_{2,1}, \mfx_{2,2}$ with probability $1/2$. All rewards are 0. By defining $\MT_1 w := |w|$, it is readily seen that $M$ is linear Bellman complete with respect to the above feature mappings.

  But Bellman-linearity of $F_2(x)$ would require that there is some $\bar w \in \BR$ so that
  \begin{align}
\phi_1(\mfx_1, 0) \cdot \bar w = 1, \qquad \phi_1(\mfx_1, 1) \cdot \bar w = 2\nonumber.
  \end{align}
  No such $\bar w$ exists since $\phi_1(\mfx_1, 0) = \phi_1(\mfx_1,1).$
\end{proof}
\cref{prop:quadratic-ctex} establishes that, for $\Bquad_h(x,a; \Sigma)$ defined as in \cref{eq:quad-bonus}, the mapping $\max_{a \in \MA} \Bquad_h(\cdot, a; I_d)$ is, in general, not Bellman-linear. 

\subsection{Perturbed linear policies}
We begin by defining a perturbed version of linear policies. 
\begin{definition}[Perturbed linear policies]
  \label{def:plinear}
  For $\sigma > 0$, $h \in [H]$ and $w \in \BR^d$, define $\pi_{h,w,\sigma} : \MX \ra \Delta(\MA)$ by
  \begin{align}
\pi_{h,w,\sigma}(x)(a) = \Pr_{\theta \sim \MN(w, \sigma^2 \cdot I_d)} \left( a \in \MA_{h,\theta}(x)\right)\nonumber.
  \end{align}
  In words, to draw an action $a \sim \pi_{h,w,\sigma}(x)$, we draw $\theta \sim \MN(w, \sigma^2 \cdot I_d)$ and then play $\argmax_{a' \in \MA} \lng \phi_h(x,a'), \theta \rng$. Given $\sigma > 0$, we denote the set of all $\pi_{h,w,\sigma'}$, where $w \in \BR^d, \sigma' > 0$ satisfy $\sigma' / \| w \|_2 \geq\sigma$, by $\Pilinp_h$, and $\Pilinp := \prod_{h=1}^H \Pilinp_h$. %
  As a matter of convention we write $\Pilinp[0]_h :=\Pilin_h$ and $\Pilinp[0] = \Pilin$. Moreover, we write $\Pilinpp_h := \bigcup_{\sigma \geq 0} \Pilinp_h$ and $\Pilinpp =\bigcup_{\sigma \geq 0} \Pilinp$. 
\end{definition}Note that, for any $c > 0$, $\pi_{h,cw, \sigma} = \pi_{h, w, \sigma/{c}}$. Moreover, note that tie-breaking is not an issue for perturbed linear policies, since the  measure of all $\theta$ so that $|\MA_{h,\theta}(x)| > 1$ is 0. 
The following lemma provides another interpretation of linear policies, as a limit of perturbed linear policies:
\begin{lemma}
  \label{lem:linpol-limit}
  For any $w \in \BR^d, h \in [H], x \in \MX$, we have
  \begin{align}
\pi_{h,w}(a|x) = \lim_{\sigma \ra 0^+} \pi_{h,w,\sigma}(a|x)= \lim_{\sigma \ra 0^+}\Pr_{\theta \sim \MN(w, \sigma^2 \cdot I_d)}\left( a \in \MA_{h,\theta}(x) \right)\label{eq:pihw-limit}. %
  \end{align}
\end{lemma}
\begin{proof}
  It is straightforward to see that for any $a \not \in \MA_{h,w}(x)$, $\lim_{\sigma \ra 0^+} \Pr_{\theta \sim \MN(w, \sigma^2 I_d)}(a \in \MA_{h,\theta}(x)) = 0$. This verifies \cref{eq:pihw-limit} for $a \not \in \MA_{h,w}(x)$. Moreover, it also implies that for each $a \in \MA_{h,w}(x)$, we have
  \begin{align}
\lim_{\sigma \ra 0^+} \left| \Pr_{\theta \sim \MN(w, \sigma^2 I_d)}(a \in \MA_{h,\theta}(x)) -  \Pr_{\theta \sim \MN(w, \sigma^2 I_d)} \left(\lng \theta, \phi_h(x,a) \rng > \max_{a' \in \MA_{h,w}(x) \backslash \{ a \}} \lng \theta, \phi_h(x,a') \rng \right) \right| = 0 \nonumber.
  \end{align}
  Since $\lng w, \phi_h(x,a) \rng = \lng w, \phi_h(x,a') \rng$ for all $a,a' \in \MA_{h,w}(x)$, we have that
  \begin{align}
    & \Pr_{\theta \sim \MN(w, \sigma^2 I_d)} \left(\lng \theta, \phi_h(x,a) \rng > \max_{a' \in \MA_{h,w}(x) \backslash \{ a \}} \lng \theta, \phi_h(x,a') \rng \right)\nonumber\\
    =&  \Pr_{\theta \sim \MN(0, \sigma^2 I_d)} \left(\lng \theta, \phi_h(x,a) \rng > \max_{a' \in \MA_{h,w}(x) \backslash \{ a \}} \lng \theta, \phi_h(x,a') \rng \right)\nonumber\\
    =&  \Pr_{\theta \sim \nu_d} \left(\lng \theta, \phi_h(x,a) \rng > \max_{a' \in \MA_{h,w}(x) \backslash \{ a \}} \lng \theta, \phi_h(x,a') \rng \right) = \nu_d(\MG_{h,w}(x,a))\nonumber,
  \end{align}
  where the second equality follows by rescaling. Combining the two displays above gives the result. 
\end{proof}
For a randomized policy $\pi$ and $h \in [H]$, we use the shorthand
\begin{align}
\phi_h(x,\pi_h(x)) := \E_{a \sim \pi_h(x)}[\phi_h(x,a)]\nonumber.
\end{align}
 Then \cref{lem:linpol-limit} yields that, for any $h \in [H], w\in \BR^d, f : \MX \times \MA \ra \BR$, we have
\begin{align}
\phi_h(x, \pi_{h,w}(x)) = \lim_{\sigma \ra 0^+} \phi_h(x, \pi_{h,w,\sigma}(x)), \qquad f_h(x, \pi_{h,w}(x)) = \lim_{\sigma \ra 0^+} f(x, \pi_{h,w,\sigma}(x))\label{eq:feature-limit}.
\end{align}

\arxiv{\bibliographystyle{alpha}}
\arxiv{\bibliography{lbc}}

\end{document}